\documentclass{article}

\usepackage{layouts}

\usepackage{authblk}
\usepackage{graphicx}
\usepackage{amssymb,amsmath, amsthm,mathtools,amsfonts}

\usepackage[marginratio=1:1,height=600pt,width=460pt,tmargin=100pt]{geometry}

\usepackage{layout, color}
\usepackage{enumerate}
\usepackage{algorithm}
\usepackage{algpseudocode}
\usepackage{setspace}
\usepackage{enumitem}
\usepackage{array}
\usepackage{multirow}
\usepackage{hyperref}
\usepackage{url}
\usepackage{mdwlist}
\usepackage{multirow}
\usepackage{amsthm}
\usepackage{color}
\usepackage[font=small,labelfont=bf]{caption}

\usepackage[labelfont=scriptsize,labelfont=scriptsize]{subcaption}

\newcommand{\norm}[1]{\left\lVert#1\right\rVert}
\newcommand{\abs}[1]{\left |#1\right |}
\newcommand*\diff{\mathop{}\!\mathrm{d}}
\newcommand{\trace}[1]{\mathrm{tr}\left(#1\right)}
\newcommand{\calB}{\mathcal{B}}
\newcommand{\calN}{\mathcal{N}}

\newcommand{\bbR}{\mathbb{R}}

\newcommand{\ie}{i.e.}
\newcommand{\eg}{e.g.}

\newtheorem{theorem}{Theorem}[section]

\newtheorem{corollary}[theorem]{Corollary}
\newtheorem{lemma}[theorem]{Lemma}

\theoremstyle{remark}
\newtheorem{remark}{Remark}

\newtheorem{conjecture*}{Conjecture}

\theoremstyle{plain}

\usepackage{xcolor}

\title{SpecNet2:
Orthogonalization-free spectral embedding\\
by neural networks}

\begin{document}

\author[1]{Ziyu Chen\thanks{Email: ziyu@math.duke.edu}}
\author[2]{Yingzhou Li\thanks{Email: yingzhouli@fudan.edu.cn}
}

\author[1]{Xiuyuan Cheng\thanks{Email: xiuyuan.cheng@duke.edu}}
\affil[1]{Department of Mathematics, Duke University}
\affil[2]{School of Mathematical Sciences, Fudan University}

\date{}

\maketitle

\begin{abstract}%

Spectral methods which represent data points by eigenvectors of kernel
matrices or graph Laplacian matrices have been a primary tool in
unsupervised data analysis. In many application scenarios, parametrizing
the spectral embedding by a neural network that can be trained over
batches of data samples gives a promising way to achieve automatic
out-of-sample extension as well as computational scalability. Such an
approach was taken in the original paper of SpectralNet (Shaham et al.
2018), which we call SpecNet1. The current paper introduces a new
neural network approach, named SpecNet2, to compute spectral embedding which optimizes an equivalent objective of the eigen-problem and removes the orthogonalization layer in SpecNet1. SpecNet2 also allows separating
the sampling of rows and columns of the graph affinity matrix by tracking
the neighbors of each data point through the gradient formula.
Theoretically, we show that any local minimizer of the new
orthogonalization-free objective reveals the leading eigenvectors.
Furthermore, global convergence for this new orthogonalization-free objective using a batch-based gradient descent method is proved. Numerical experiments demonstrate the improved
performance and computational efficiency of SpecNet2 on simulated data and
image datasets. 

\end{abstract}

\section{Introduction}

Spectral embedding, namely representing data points in lower dimensional
space using eigenvectors of a kernel matrix or graph Laplacian matrix,
plays a crucial role in unsupervised data analysis. It can be used, for
example, for dimension reduction, spectral clustering, and revealing
the underlying topological structure of a dataset. A known challenge in
the use of spectral embedding is the out-of-sample extension. Another
shortcoming in practice is the possible high computational cost due to the
involvement of constructing a kernel matrix and solving an eigen-problem.
To overcome these challenges, previously, the original
SpectralNet~\cite{shaham2018spectralnet}, which we call SpecNet1 in this
paper, adopted a neural network to embed data into the eigenspace of its
associated graph Laplacian matrix. To be able to enforce the orthogonality
among eigenvectors, an additional orthogonalization layer is appended to
the neural network and updated after each optimization step. 
Accurate computation of the orthogonalization layer requires evaluation
of the neural network model on the whole dataset, which would be computationally too expensive for large datasets. To reduce the computational cost in practice, 
in SpecNet1 the computation of the orthogonalization layer is approximated by only using mini-batches  of data samples, see more in Section \ref{subsec:specnet1}.
However, when a small batch is being used, the approximation error therein leads to unsatisfactory
convergence behavior in practice. 
This is further elaborated in Remark~\ref{rmk:2}. 
This work develops SpecNet2,
which removes the orthogonalization layer and will compute the neural network spectral embedding more efficiently.

From the perspective of linear algebra eigen-problems, 
SpecNet1 adopts the projected gradient descent method to optimize an {\it orthogonally-constrained} quadratic objective of the eigenvalue problem, where the orthogonalization layer is updated to conduct the orthogonalization projection step through a QR decomposition or a Cholesky decomposition. 
Meanwhile, the past decade has witnessed a trend in employing
{\it unconstrained} optimization to address the eigenvalue problem without the need for the orthogonalization step~\cite{Liu2015c, Lei2016, Li2019c,
Gao2020, Gao2021}, especially in the field of computational
chemistry~\cite{Mauri1993, Ordejon1993, Wang2019}. These unconstrained optimization techniques are also known as ``orthogonalization-free optimization'' for solving eigen-problems. All of these methods adopt various forms
of quadratic polynomials as their objective functions. Some works ~\cite{Liu2015c,
Lei2016, Li2019c, Gao2020, Wang2019} are equivalent to applying the
penalty method to the orthogonally constrained optimization problem. There
are two major reasons behind moving from constrained optimization to
unconstrained optimization: 1) explicit orthogonalization requires
solving a reduced size eigenvalue problem, which is not of high parallel
efficiency; 2) explicit orthogonalization requires accessing the entire
vectors, which is incompatible with batch updating scheme. As a result, in dealing with large-scale eigenvalue problems where parallelization and/or batch updating scheme are needed, an unconstrained optimization approach is preferred. This naturally suggests the use of unconstrained optimization
for the eigen-problem in neural network spectral embedding methods.

In the current paper, we modify the orthogonalization-free objective in
\cite{Li2019c} for the graph Laplacian matrix and use it under the spectral
network framework~\cite{shaham2018spectralnet} so as to compute neural network
parametrized spectral embedding of data. The contribution of the work includes 
\begin{itemize}

\item
The proposed spectral network, SpecNet2, is trained with an
orthogonalization-free training objective, which can be optimized more
efficiently than SpecNet1. In particular, the new optimization
objective we proposed allows updating the graph neighbors of the samples in a
mini-batch at each iteration, which memory-wise only requires loading part
of the affinity matrix restricted to that graph neighborhood. Thus the
method better scales to large graphs.

\item 
Theoretically, it is proved that the global minimum of the
orthogonalization-free objective function (unique up to a rotation)
reveals the leading eigenvectors of the graph Laplacian matrix and the
iterative update scheme is guaranteed to converge to the global minimizer
for any initial point up to a measure-zero set.

\item
The efficiency and advantage of SpecNet2 over SpecNet1 with neighbor
evaluation scheme are demonstrated empirically on several numerical examples.
The network embedding also shows better stability and accuracy in some
cases. 

\end{itemize}

The rest of the paper is organized as follows. In
Section~\ref{sec:notation}, we introduce notations used throughout the
paper, as well as the background of the spectral embedding problem we aim
to solve. In Section~\ref{sec:lamethod}, we propose an
orthogonalization-free iterative eigen-problem solver from a numerical linear
algebra point of view with three updating schemes. In
Section~\ref{sec:nnmethod}, we introduce the neural network
parametrization as well as the updating scheme implementations in neural
network training. Theoretical results are analyzed in
Section~\ref{sec:theorem}. Numerical results are shown in
Section~\ref{sec:numres}. Finally, we conclude our paper with discussions
in Section~\ref{sec:conclusion}.

\section{Background}
\label{sec:notation}
    
\subsection{Graph Laplacian and spectral embedding}

Given a dataset of $n$ samples $X = \{x_i\}_{i=1}^n$ in $\bbR^m$, an
\emph{affinity matrix} $W$ is constructed such that $W_{i,j}$ measures the
similarity between $x_i$ and $x_j$. By construction, $W$ is a real
symmetric matrix of size $n \times n$, and $W_{i,j} \ge 0$. One could also
view $W$ as the weights on edges of an undirected graph $G=(V,E)$, where
$V = [n] = \{1, 2, \dots, n\}$, $E = \{(i,j), W_{i,j} > 0\}$. In many
scenarios, $W$ is constructed to be a sparse matrix. For example, in
Laplacian eigenmap~\cite{belkin2003laplacian} and Diffusion
map~\cite{coifman2006diffusion}, the affinity matrix can be constructed as
\begin{equation} \label{eq:kernel-W}
    W_{i,j} = h \left(  \frac{ \|x_i - x_j\|^2}{ \sigma^2} \right),
\end{equation}
where $h$ is a non-negative function on $[0,\infty)$, is compactly
supported or decays exponentially. As a result, when the kernel bandwidth
$\sigma$ is chosen to be the scale of the size of a local neighborhood,
then $W_{i,j}$ is significantly non-zero only when $x_i$ is a nearby
neighbor of $x_j$. Typical non-negative function $h$ includes the
indicator function of $[0,1)$, Gaussian function $e^{-r^2}$, truncated
Gaussian function, et al. Other examples of $W$
which differs from the form of \eqref{eq:kernel-W} include kNN graphs and
kernels with self-tuned bandwidth~\cite{cheng2020convergence}. These
constructions also produce a sparse real symmetric matrix $W$.

Given an affinity matrix $W$, the \emph{degree matrix} $D$ of $W$ is a
diagonal matrix with diagonal entries defined by $D_{i,i} = \sum_{j=1}^n
W_{i,j}$. Note that $D_{i,i} > 0$ whenever the graph has no isolated node. The
matrix $P := D^{-1} W$ is row-stochastic and can be viewed as the
transition matrix of a random walk on the graph. The matrix $L_{rw} = I-
P$ is called the ``random-walk graph Laplacian'' and $L_{rw}$ has real
eigenvalues and eigenvectors $L_{rw} \psi_k = \lambda_k \psi_k$, starting
from $\lambda_1 = 0$ and $\psi_1$ is a constant eigenvector. Throughout this paper, we call the zero eigenvalue the ``trivial'' eigenvalue and eigenvectors associated with zero eigenvalue the ``trivial'' eigenvectors of $L_{rw}$; ``nontrivial'' refers to non-zero eigenvalues and eigenvectors associated with non-zero eigenvalues of $L_{rw}$. When
the graph is connected, the trivial eigenvalue zero is of multiplicity one. The
first $K-1$ nontrivial eigenvectors $\psi_2, \dots, \psi_{K}$ associated with
the smallest eigenvalues $0 < \lambda_2 \le \cdots \le \lambda_K$ of
$L_{rw}$, can provide a low-dimensional embedding of the dataset $X$,
known as the Laplacian Eigenmap~\cite{belkin2003laplacian}, where each
sample is mapped to
\begin{equation} \label{eq:eigen-map}
    x_i \mapsto \Psi(x_i) =  [\psi_2(i), \dots, \psi_K(i) ] \in \bbR^{K-1}.
\end{equation}
Diffusion maps~\cite{coifman2006diffusion} is closely related that maps
\begin{equation}\label{eq:diffusion-map}
    x_i \mapsto \Psi_t(x_i) =  [\lambda_2^t\psi_2(i), \dots,
    \lambda_K^t\psi_K(i) ] \in \bbR^{K-1},
\end{equation}
where $t > 0$ is the diffusion time. These embeddings using the
eigenvectors of graph Laplacians are called \emph{spectral embedding}. The
eigenvectors of unnormalized graph Laplacian $D-W$ have also been used for
embedding and spectral clustering.

\subsection{Out-of-sample extension and limiting eigenfunctions}\label{sec:limiting}

Note that in \eqref{eq:eigen-map}, the mapping $\Psi$ is defined on the
discrete points $x_i\in X$ but not on the whole space yet, since it is
provided by the eigenvectors of a discrete graph Laplacian matrix. The
problem of \emph{out-of-sample extension} for kernel methods and spectral
methods refers to the problem of efficiently generalizing the spectral
embedding to new samples not in $X$. Recomputing the eigenvalue
decomposition on the extended dataset is computationally too expensive to
be practical. Ideally, we would like to generalize the spectral embedding
without such recomputation. Classical methods include Nystr{\"o}m
extension~\cite{nystrom1930praktische, belabbas2009landmark,
williams2000using} and its variants~\cite{bermanis2013multiscale}. More
recently, a neural network-based approach has been proposed in
\cite{shaham2018spectralnet} to parametrize the eigenvectors of the
Laplacian that automatically gives an out-of-sample extension.

Theoretically, it is thus natural to ask when the mapping $\Psi (x_i)$ is
the restriction of an underlying eigenfunction in the continuous space on
the dataset $X$. An answer has been provided by the theory of
\emph{spectral convergence} in a manifold data setting: when data are
sampled on a sub-manifold $\mathcal{M}$ which can be of lower
dimensionality than the ambient space, the eigenvectors and eigenvalues of
the graph Laplacian $L_{rw}$ built from $n$ samples with kernel bandwidth
parameter $\sigma$ converge to the eigenfunctions and eigenvalues of a
limiting differential operator $\mathcal{L}$ when $n \to \infty$ and
$\sigma \to 0$~\cite{coifman2006diffusion}. The expression of
$\mathcal{L}$ depends on the affinity construction and the kernel matrix
normalization, e.g., when data points are uniformly sampled on the
manifold with respect to the Riemannian volume then $\mathcal{L} =
-\Delta_{\mathcal{M}}$ (the Laplace-Beltrami operator up to a sign); and
when density is non-uniform, $\mathcal{L}$ is a certain infinitesimal
generator of the manifold diffusion process. The spectral convergence on
finite samples requires $\sigma$ to scale with $n$ in a proper way, and in
practice, the low-lying eigenvectors, namely those with smaller
eigenvalues of $\mathcal{L}$ near zero, converge faster than the
high-frequency (high-lying) ones.

As a result, in applications where the data samples can be viewed as lying
on or near to a low-dimensional submanifold, it is natural to parametrize the
first $K-1$ nontrivial eigenvectors $\psi_k$ of the large kernel matrix,
evaluated at sample $x_i$, by a neural network, that gives us $\psi_{k,
\theta}(x_i)$, $k=2,\dots, K$, where $\theta$ stands for network
parameters.

\subsection{Summary of SpecNet1
}\label{subsec:specnet1}

SpecNet1~\cite{shaham2018spectralnet} adopts neural network parametrizations of eigenvectors of a normalized graph Laplacian, and the network is trained by minimizing an objective which is the variational form of the eigen-problem with an orthogonality constraint. Here we briefly review the three ingredients of the method of SpecNet1: the linear algebra optimization objective, the batch-based gradient evaluation scheme, and the neural network parametrization (including the orthogonalization layer). 

{\it Optimization objective.}
From a linear algebra point of view,
SpecNet1 aims to find the first $K$ eigenvectors of the symmetrically
normalized Laplacian $L_{sym}:=I-D^{-\frac{1}{2}}WD^{-\frac{1}{2}}$ via
solving the following orthogonally constrained optimization problem
\begin{equation} \label{eq:specnet1}
    \min_{\substack{Y^\top Y = n I \\ Y \in \bbR^{n \times K}}} f_1(Y)
    = \frac{1}{n}\trace{Y^\top (I-D^{-\frac{1}{2}}WD^{-\frac{1}{2}}) Y}.
\end{equation}
Note that in \eqref{eq:specnet1},  $Y$ is a real array as in the classical variational form of eigen-problem. It will be parametrized by a neural network below. 

{\it Mini-batch gradient evaluation.}
When the graph is large, memory constraint in practice usually prevents loading the full graph affinity matrix $W$ into the memory or solving the full matrix $Y$ over iterations. Thus, mini-batch is used in the training of SpecNet1. Given a batch of data points $\calB \subset X$, SpecNet1 performs a single projected
gradient descent step of a surrogate constrained optimization problem
\begin{equation} \label{eq:specnet1-batch}
    \min_{\substack{Y_\calB^\top Y_\calB = b I \\
    Y_\calB \in \bbR^{b \times K}}} \tilde{f}_1(Y_\calB)
    = \frac{1}{b}\trace{Y_\calB^\top (I-\tilde{D}_\calB^{-\frac{1}{2}}
    W_{\calB, \calB}\tilde{D}_\calB^{-\frac{1}{2}}) Y_\calB},
\end{equation}
where $W_{\calB,\calB}$ is a submatrix of $W$ with row and column index
associated to data points in $\calB$, $\tilde{D}_{\calB}$ is the diagonal
degree matrix of $W_{\calB, \calB}$, and $b$ is the number of data points
in $\calB$. 
We call~\eqref{eq:specnet1-batch} the ``local evaluation scheme'' of SpecNet1, as it only uses $ W_{\calB, \calB}$ retrieved from the matrix $W$ when updating $Y_\calB$. 
In this paper, we will propose and study three different  mini-batch evaluation schemes in the training SpecNet2, and local scheme is one of the three. Corresponding to the other two mini-batch evaluation schemes of SpecNet2, which are called ``neighbor'' and ``full'' schemes respectively, we also study the counterpart schemes for SpecNet1. The details are explained in Section \ref{subsec:net2-training} (for neural network training) and Section  \ref{sec:minibatch} (for linear algebra optimization problem). 
Figure \ref{fig:onemoon-net-all} gives a comparison of the different mini-batch schemes used to train SpecNet1 and SpecNet2.
It can be seen that the performance of the local scheme is inferior to the other mini-match evaluation schemes.
Actually, in the linear algebra iterative solver (without neural network parametrization) of the variational eigen-problem, the relatively worse performance of the local scheme already presents, c.f. Figure \ref{fig:onemoon-la-evecs}. 
This is because only using the submatrix $W_{\calB, \calB}$ may drastically lose the information of $W$ when the batch size is small, especially when the graph is sparse. In contrast, the neighbor and full schemes use more information of $W$. See more in later sections.

{\it Neural network parametrization.}
The neural network architecture in SpecNet1~\cite{shaham2018spectralnet} contains two parts: 
First, a network mapping $\Phi_\theta: \bbR^m \to \bbR^K$, parametrized by $\theta$,
which maps an input data point $x_i \in \bbR^m $ to the $K$-dimensional space of spectral embedding coordinates;
Second, 
an additional linear layer $\Xi \in \bbR^{K \times K}$, mapping from $\bbR^K$ to $\bbR^K$ and parametrized by the matrix $\Xi$, such that the composed mapping $\Phi_\theta(x_i) \Xi$ approximates the spectral embedding (eigenvectors), i.e., 
\begin{equation}
    \Psi(x_i) \approx \Phi_\theta(x_i) \Xi.
\end{equation}
The linear layer parametrized by $\Xi$ is called the ``orthogonalization layer''. The neural network embedding of the entire dataset $X$ is then represented as
\begin{equation} \label{eq:nnY}
    Y(X) =
    \begin{pmatrix}
        (\Phi_\theta(x_1) \Xi)^\top
        & (\Phi_\theta(x_2) \Xi)^\top
        & \cdots
        & (\Phi_\theta(x_n) \Xi)^\top
    \end{pmatrix}^\top\in\mathbb{R}^{n\times K},
\end{equation}
where $(\Phi_\theta(x_i) \Xi)^\top$ is a column vector in $\mathbb{R}^K$ for each $i=1,\dots,n$.

{\it  Influence on the orthogonality constraint.}
We now explain a crucial difference when parametrizing $Y$ by a neural network on the maintenance of the orthogonality constraint when using mini-batch. 
Note that the network representation \eqref{eq:nnY} differs from a real array $Y$ in that all rows of $Y$ in \eqref{eq:nnY} are related via network parametrization $\theta$ and $\Xi$.
Using mini-batch, in a linear-algebra update of $Y$ in \eqref{eq:specnet1-batch}, an update on $Y_\calB$ would only change $Y_\calB$ and leave the rest entries $Y_{\calB^c}$ unchanged, where $\calB^c = X \setminus \calB$. In contrast, using the back-propagated gradient to update network parameters in  \eqref{eq:nnY}, any update on $\theta$ and $\Xi$ would change the embedding of all data points in $\calB$ and $\calB^c$.

In the training of SpecNet1, the neural network parameters $\theta$ and
$\Xi$ are updated separately in a mini-batch iteration. Specifically, at
each mini-batch iteration, SpecNet1 first computes an overlapping matrix $\begin{pmatrix}
(\Phi_\theta(x_i))^\top \Phi_\theta(x_j) \end{pmatrix}_{x_i,x_j \in
\calB}$ and its Cholesky factor $L$. Then the orthogonalization layer
parameter $\Xi$ is updated as $\Xi = \sqrt{b}(L^{-1})^\top$ to enforce the
orthogonality constraint in \eqref{eq:specnet1-batch}. In the
second step, it takes a gradient descent step or
an equivalent optimization step of $\tilde{f}_1(Y_\calB)$ with respect to
$\theta$ to update weights $\theta$ and keep the orthogonalization layer unchanged. Due to the
dependence among rows of $Y$ as in \eqref{eq:nnY}, we emphasize that such
a mini-batch iteration also changes $Y_{\calB^c}$ and the orthogonality
constraint as in \eqref{eq:specnet1} cannot be exactly maintained.

We see in Figure \ref{fig:onemoon-la-evecs}
that in the linear algebra setting, SpecNet1 achieves good convergence with both the full and neighbor evaluation schemes; however, in the neural network setting, SpecNet1 with the neighbor scheme performs significantly worse than the full scheme, as shown in Figure \ref{fig:onemoon-net-all}. This is because in the neural network, at each iteration, the orthogonalization is computed based on the update only on the neighborhood of $\calB$ for the neighbor scheme, while for the full scheme, the orthogonalization is computed on the updated output on the whole dataset $X$. On the other hand, due to memory constraints, we do not want to perform orthogonalization over all data samples at each iteration, we thus want to find a way such that we can still obtain good convergence with light memory budget. This motivates our development of SpecNet2 in this paper.

\subsection{Other related works}

The convergence of graph Laplacian eigenvectors to the limiting
eigenfunctions of the manifold Laplacian operator has been proved in a
series of works~\cite{belkin2007convergence, von2008consistency,
burago2014graph, singer2016spectral} and recently in
\cite{trillos2020error, calder2019improved, dunson2021spectral,
calder2020lipschitz, cheng2021eigen}. The result shows that in the
i.i.d. manifold data setting, the empirical graph Laplacian eigenvectors
approximate the eigenfunctions evaluated on the data points in the large
sample limit, where the kernel bandwidth is properly chosen to decrease to
zero. The robustness of spectral embedding with input data noise has been
shown in \cite{shen2020scalability}, among others. Based on these
theories, the current work utilizes the neural network to approximate
eigenfunctions so as to generalize to test data samples, due to that the
eigenfunctions are the consistent limit of the eigenvectors of properly
constructed graph Laplacian. 

For neural network methods to obtain dimension-reduced embedding, neural
network embedding guided by pairwise relation was explored earlier in
SiameseNet~\cite{hadsell2006dimensionality}, where the training objective
is heuristic. Using kernel affinity and spectral embedding to overcome the
topological constraint in neural network embedding has been explored in
\cite{mishne2019diffusion}, and under the Variational Auto-encoder
framework in \cite{li2020variational}. The current paper differs from
these auto-encoder methods in that SpecNet2, same as in SpecNet1, outputs
a dimension-reduced representation of data in a low-dimensional space,
from which the training objective is computed via the graph Laplacian
matrix.

\section{Orthogonalization-free Iterative Eigensolver}
\label{sec:lamethod}

We first investigate an orthogonalization-free iterative eigensolver,
which serves as the loss function of SpecNet2 from a linear algebra point
of view. Then three updating schemes incorporated with the coordinate
descent method are proposed and compared, which later will be turned into
the mini-batch technique in the neural network in
Section~\ref{sec:nnmethod}. Finally, the computational costs of three
updating schemes are analyzed.

\subsection{Unconstrained optimization}\label{sec:unconstrained}

Recall that the spectral embedding is by computing the leading
eigenvectors of the graph Laplacian $L_{rw} = I - D^{-1} W$. Equivalently,
it aims to find $K$ eigenvectors corresponding to the $K$ largest
eigenvalues of a generalized eigenvalue problem (GEVP) with the matrix
pencil $(W, D)$, where $K-1$ is the dimension of embedded space and $D$ is
the diagonal degree matrix associated with $W$. More explicitly, the
generalized eigenvalue problem is of the form,
\begin{equation} \label{eq:evp}
    \begin{split}
        & W U = D U \Lambda, \\
        & U^\top D U = n^2 I,
    \end{split}
\end{equation}
where $\Lambda \in \bbR^{K \times K}$ is a diagonal matrix with its
diagonal entries being the largest $K$ eigenvalues of $(W,D)$, $U \in
\bbR^{n \times K}$ is the corresponding eigenvector matrix, and $I$
denotes the identity matrix of size $K$. Throughout this paper, we assume
the eigenvalue problem \eqref{eq:evp} has a nonzero eigengap between the
$K$-th and $(K+1)$-th eigenvalues. Such a GEVP has been extensively
studied and many efficient algorithms can be found in~\cite{Golub2013} and
references therein.

In contrast to the constrained optimization problem as in SpecNet1, we
propose to solve an unconstrained optimization problem to find the
eigenpairs of \eqref{eq:evp}. Many previous works~\cite{Liu2015c,
Lei2016, Li2019c, Wang2019} adopt an unconstrained optimization problem
for solving the standard eigenvalue problem, \ie, with $D = I$ in
\eqref{eq:evp}. The optimization problem therein minimizes $\norm{W - Y
Y^\top}_{\text{F}}^2$ without any constraint on $Y$.

Extending the optimization problem to GEVP, we propose the following
unconstrained optimization problem,
\begin{equation} \label{eq:uopt}
    \min_{Y \in \bbR^{n \times K}} f_2(Y) = \frac{1}{n^2}\trace{-2 Y^\top W Y +
    \frac{1}{n^2}Y^\top D Y Y^\top D Y}.
\end{equation}
The gradient of $f_2(Y)$ with respect to $Y$ is
\begin{equation} \label{eq:gradY}
    \nabla_Y f_2(Y) = -4 \frac{W}{n}Y + 4 \frac{D}{n^3}Y Y^\top DY.
\end{equation}
Note that $\nabla_Y f_2(Y)$ in \eqref{eq:gradY} is $n$ times the actual
gradient of $f_2(Y)$ in \eqref{eq:uopt}. 
The reason of normalizing $f_2(Y)$ and $\nabla_Y f_2(Y)$ in the  way above is due to that we want to ensure an $O(1)$ limit, corresponding to the continuous limit of the eigen-problem, as $n\to\infty$. Details are explained in Appendix~\ref{appendix:scaling}.

Once we obtain the solution $\hat{Y}$ to \eqref{eq:uopt}, we can retrieve
the approximations to eigenvectors of $D^{-1}W$, denoted as $\hat{U}$, by
a single step of Rayleigh-Ritz method. More specifically, $\hat{U}$ is
calculated as $\hat{U} = \hat{Y}O$, where $O \in \bbR^{K\times K}$
satisfies
\begin{equation}\label{eq:decouple}
    \hat{Y}^{\top}W\hat{Y}O = \hat{Y}^{\top}D\hat{Y}O\hat{\Lambda},
\end{equation}
for diagonal matrix $\hat{\Lambda}$ as a refined approximation of the
eigenvalues of $(W,D)$.

Since the first trivial constant eigenvector of $D^{-1}W$ is typically not
useful, one can skip solving for that in \eqref{eq:uopt} by deflation,
\ie, replacing $W$ by $W-\eta\eta^\top$, where
$\eta=\frac{d}{\|\sqrt{d}\|_2}$, and $d\in\mathbb{R}^n$ is a column vector
with $d_i = D_{i,i}$. Since $D$ is positive-definite,
Theorem~\ref{thm:minimizers} and other analysis results still hold except
that we will skip the first trivial eigenvector in $Y^\star$, where $Y^\star$ is the minimizer of \eqref{eq:uopt}. Hence, for
the rest of the paper beside Section~\ref{sec:theorem}, we will use
\begin{equation} \label{eq:uopt2}
    \min_{Y \in \bbR^{n \times K}} f_2(Y) = \frac{1}{n^2}\trace{-2 Y^\top
    (W-\eta\eta^\top) Y + \frac{1}{n^2}Y^\top D Y Y^\top D Y}.
\end{equation}
The gradient of $f_2(Y)$ is then
\begin{equation} \label{eq:gradY2}
    \nabla f_2(Y) = -4 \frac{W-\eta\eta^\top}{n}Y
    + 4 \frac{D}{n^3}Y Y^\top DY.
\end{equation}

\subsection{Different gradient evaluation schemes}
\label{sec:minibatch}

In this subsection, we introduce efficient optimization methods of
loss \eqref{eq:uopt2} by mini-match. Mini-batch is a mandatory technique
in dealing with big datasets. Traditional mini-batch techniques randomly
sample a mini-batch of data points $\calB \subset X$, and solve the
reduced problem on $\calB$.
Due to the fact that the computational cost to
evaluate the term $Y^\top D Y$ in \eqref{eq:gradY2} is very expensive for
large $n$, we study different approximations to the gradient $\nabla
f_2(Y)$, which yields three different gradient evaluation schemes. The
visualization of these schemes in terms of the corresponding entries of
$W-\eta\eta^\top$ is shown in Figure \ref{fig:matrix}.

\begin{figure}[t]
    \centering
    \includegraphics[width=0.32\textwidth]{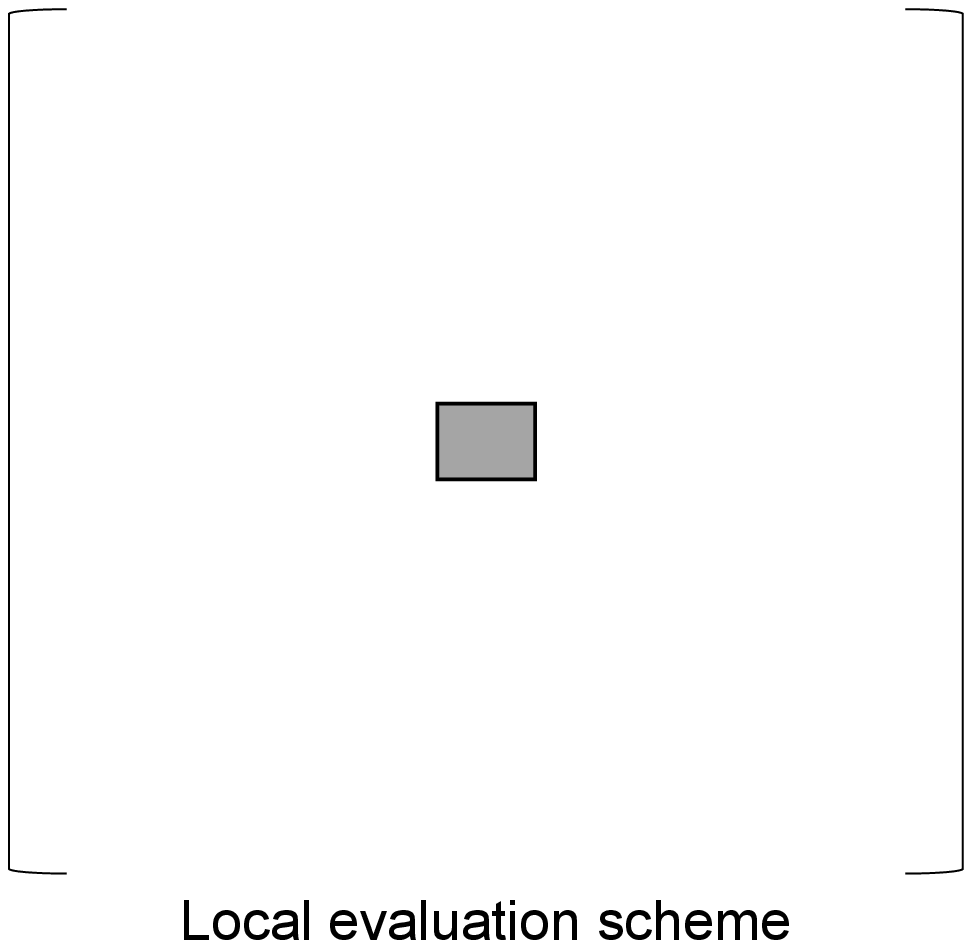}
    \includegraphics[width=0.32\textwidth]{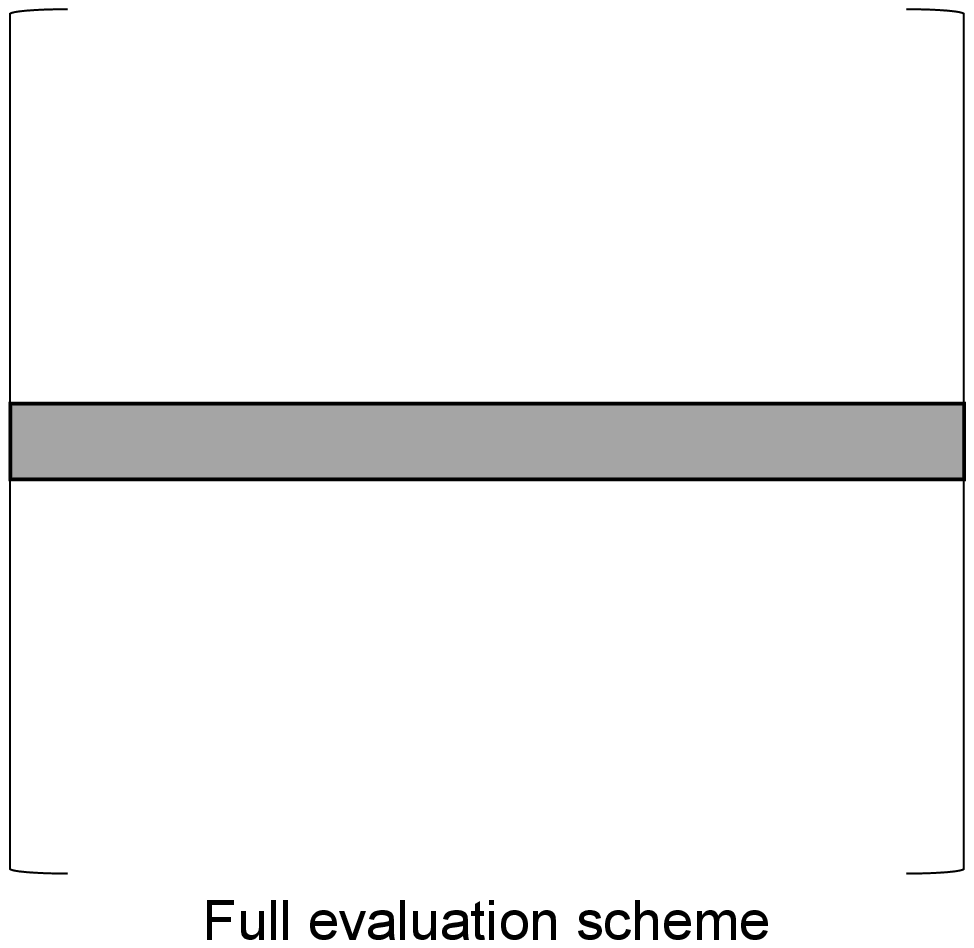}
    \includegraphics[width=0.32\textwidth]{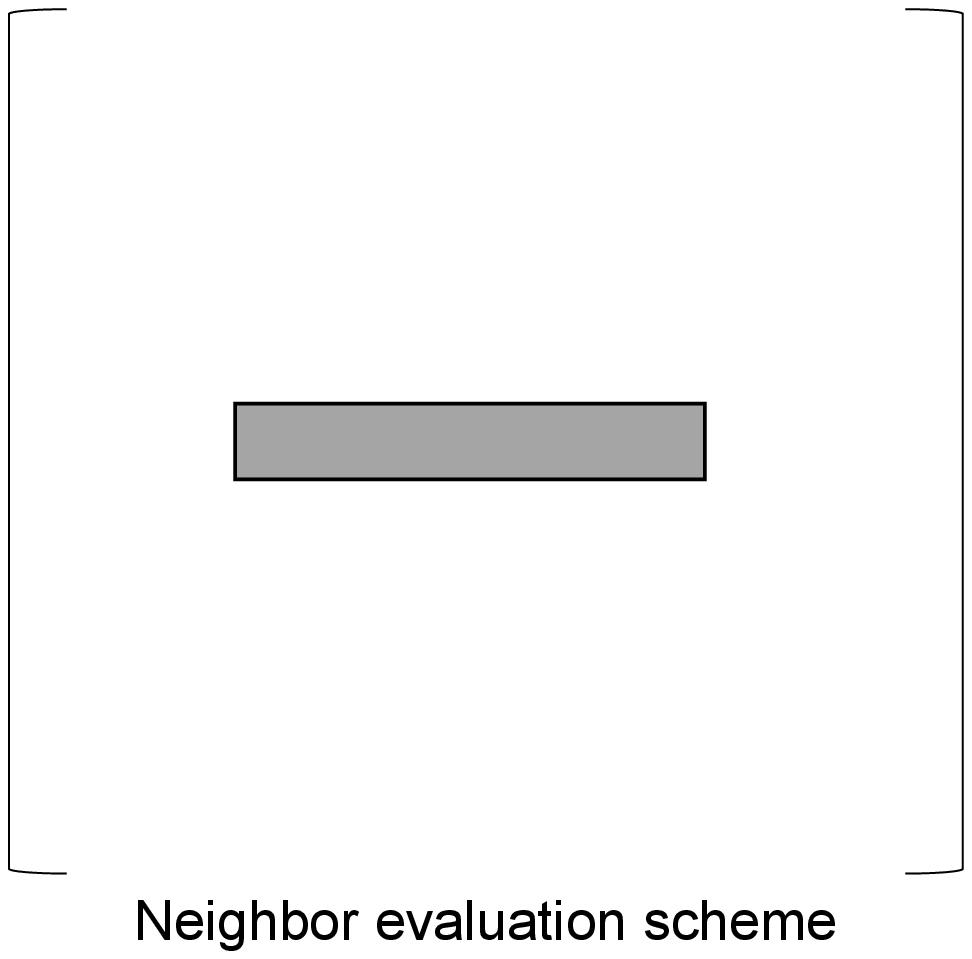}
    \caption{Entries of $W-\eta\eta^\top$ used (in gray block) in
    different gradient evaluation schemes at each batch step. See section
    \ref{sec:minibatch} for detail.}
    \label{fig:matrix}
\end{figure}

\begin{itemize}
\item 
\textbf{\emph{Local evaluation scheme:}} One can evaluate the gradient on each
mini-batch as
\begin{equation} \label{eq:gradYlocal}
    \nabla_\calB \tilde{f}_{2}(Y_{\calB}) = -\frac{4}{b}
    (W_{\calB,\calB}-\tilde{\eta}_{\calB}\tilde{\eta}_{\calB}^\top) Y_{\calB} +
    \frac{4}{b^3} \tilde{D}_{\calB} Y_{\calB} Y_{\calB}^\top \tilde{D}_{\calB}
    Y_{\calB},
\end{equation}
where $\tilde{f}_{2}(Y) = \frac{1}{b^2}\trace{-2 Y^\top
(W_{\calB,\calB}-\tilde{\eta}_{\calB}\tilde{\eta}_{\calB}^\top) Y +
\frac{1}{b^2}Y_{\calB}^\top \tilde{D}_{\calB} Y_{\calB} Y_{\calB}^\top
\tilde{D}_{\calB} Y_{\calB}}$ is the objective function on $\calB$, $b =
\abs{\calB}$ is the cardinality of $\calB$ and
$\tilde{\eta}=\frac{\tilde{d}}{\|\sqrt{\tilde{d}}\|_2}$, and
$\tilde{d}\in\mathbb{R}^{\abs{\calB}}$ is a column vector with
$\tilde{d}_i = \tilde{D}_{\calB,i,i}$, i.e., the $i$-th diagonal entry of
$\tilde{D}_\calB$. The iterative algorithm then conducts the update as,
\begin{equation}\label{eq:iterYlocal}
    Y_{\calB} = Y_{\calB} - \alpha \nabla_\calB \tilde{f}_{2}(Y_{\calB}),
\end{equation}
where $\alpha>0$ is the stepsize.

Consider an example, where data points are relatively well-separated and
the affinity matrix is very sparse. Such a mini-batch sampling is
difficult to capture the neighbor points effectively and $W_{\calB,\calB}$
for most $\calB$ is nearly diagonal. Comparing \eqref{eq:gradY2} and
\eqref{eq:gradYlocal}, \eqref{eq:gradYlocal} is not a good approximation
of \eqref{eq:gradY} unless $\calB$ is sufficiently large to capture the
asymptotic behavior of the continuous limit. Therefore, optimizing the
loss function using such a mini-batch technique requires either a big
batch size or many iterations to achieve reasonable results.

\item \textbf{\emph{Full evaluation scheme:}} We evaluate the gradient on batch
$\calB$ as
\begin{equation} \label{eq:gradYfull}
    \nabla_{\calB} f_{2}(Y) = -\frac{4}{n}
    (W_{\calB,X}-\eta_{\calB}\eta^\top) Y + \frac{4}{n^3} D_{\calB} Y_{\calB}
    Y^\top D Y,
\end{equation}
where $D_{\calB}$ is the principle submatrix of $D$ restricting to rows
and columns in $\calB$. And the update is then conducted as
\begin{equation} \label{eq:iterYfull}
    Y_{\calB} = Y_{\calB} - \alpha \nabla_{\calB} f_{2}(Y),
\end{equation}
where $\alpha>0$ is the stepsize.

This update is the block coordinate descent method applied to the proposed
optimization problem. The computational burden lies in evaluating
$\eta^\top Y$ and $Y^\top DY$ every iteration.

\item \textbf{\emph{Neighbor evaluation scheme:}} We introduce another way to
conduct mini-batch on the gradient directly, which is block coordinate
gradient descent with dynamic updating and plays an important role in the
later neural network part. Given a sampled mini-batch $\calB$, we define
the neighborhood of $\calB$ as,
\begin{equation} \label{eq:def-B-nbh}
    \calN(\calB) = \{ x_j \mid W_{i,j} \neq 0, x_i \in \calB \},
\end{equation}
and we abbreviate it as $\calN$. The gradient of batch $\calB$ is
evaluated as
\begin{equation} \label{eq:gradYneighbor}
    \nabla_{\calB} \bar{f}_{2}(Y) = -\frac{4}{n}
    W_{\calB,\calN}Y_{\calN}+\frac{4}{n}\eta_{\calB}\eta^\top Y
    + \frac{4}{n^3} D_{\calB} Y_{\calB} Y^\top D Y.
\end{equation}
Note that $\eta^\top Y = \eta_{\calB}^\top Y_\calB + \eta_{\calB^c}^\top
Y_{\calB^c}$ and $Y^\top DY= Y_{\calB}^\top D_{\calB}Y_\calB +
Y_{\calB^c}^\top D_{\calB^c}Y_{\calB^c}$, where $\calB^c = [n]\backslash
\{i: x_i \in \calB\}$. At each iteration, we only update $\eta^\top Y$ and
$Y^\top DY$ on batch $\calB$ in \eqref{eq:gradYneighbor}; that is, we
update $\eta_{\calB}^\top Y_\calB$ for $\eta^\top Y$ and $Y_{\calB}^\top
D_{\calB}Y_\calB$ for $Y^\top DY$ using $Y_\calB$ without touching the
$\calB^c$ part. The iterative algorithm then conducts the update as,
\begin{equation} \label{eq:iterYneighbor}
    Y_{\calB} = Y_{\calB} - \alpha \nabla_{\calB} \bar{f}_{2}(Y)
\end{equation}
for $\alpha$ being the stepsize.
\end{itemize}

Similarly, we can evaluate the gradient of $f_1(Y)$ using three different evaluation schemes, whose detail can be found in Appendix~\ref{subsec:f1scheme}.

\begin{figure}[t]
    \includegraphics[width=0.245\textwidth]{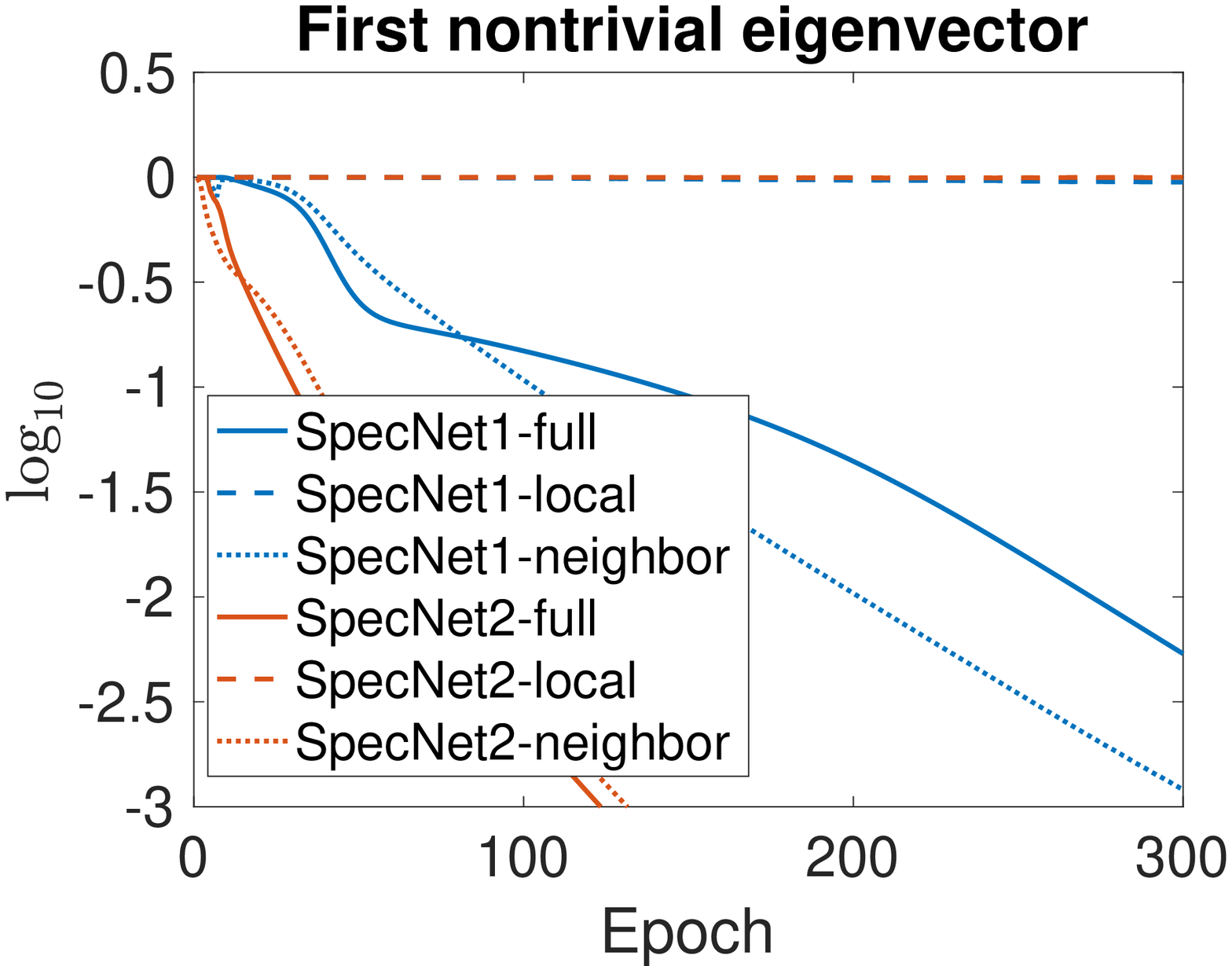}
    \includegraphics[width=0.245\textwidth]{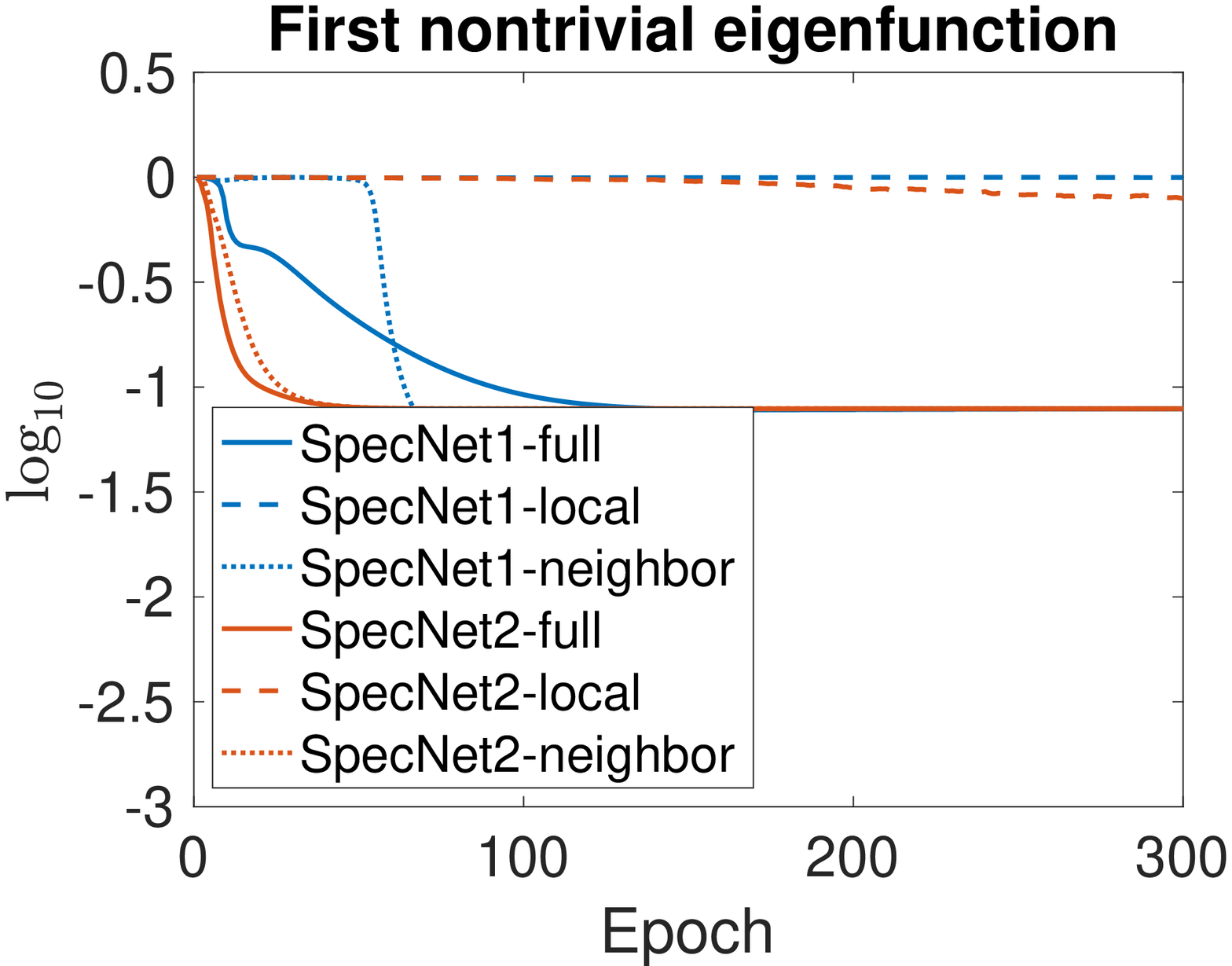}
    \includegraphics[width=0.245\textwidth]{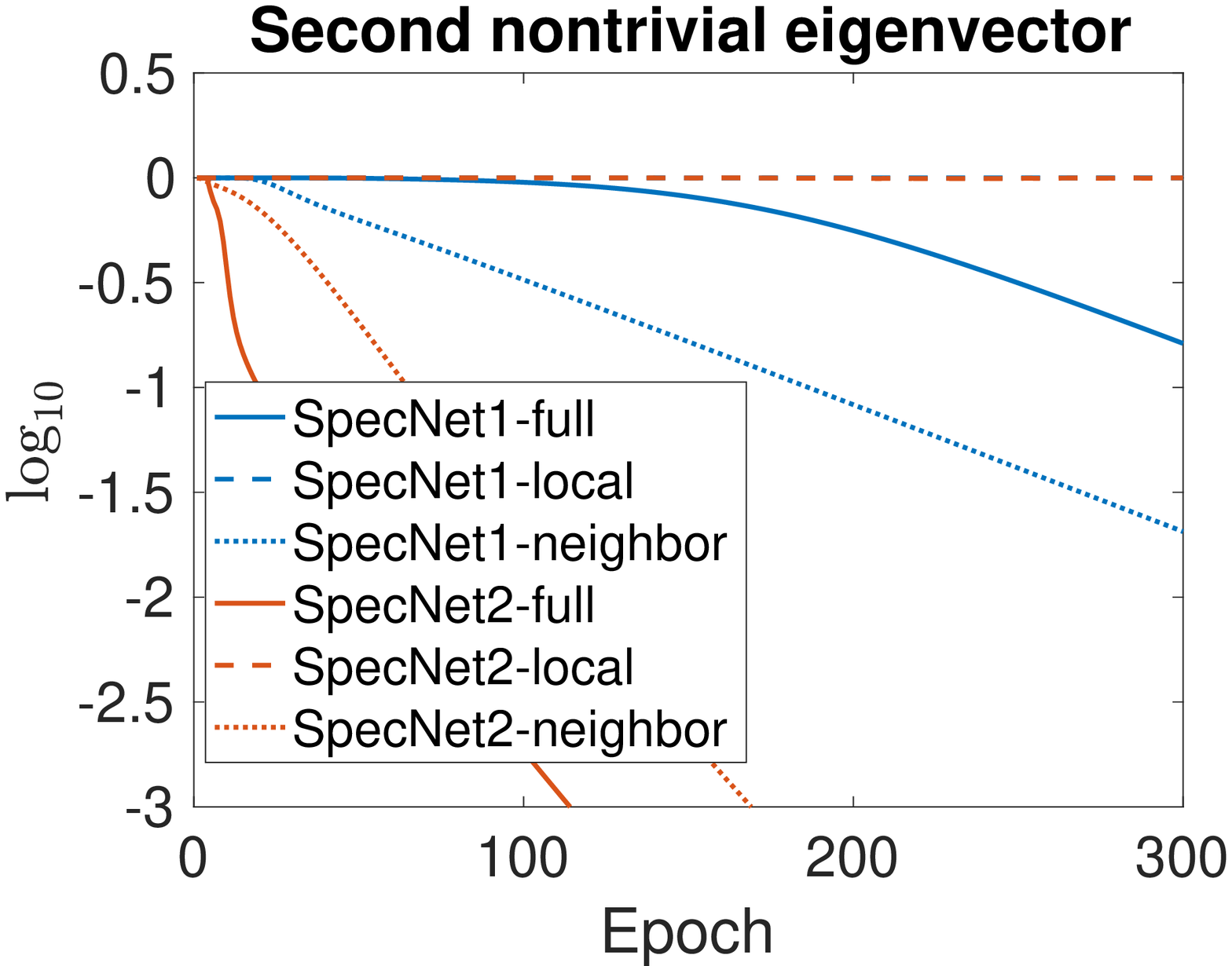}
    \includegraphics[width=0.245\textwidth]{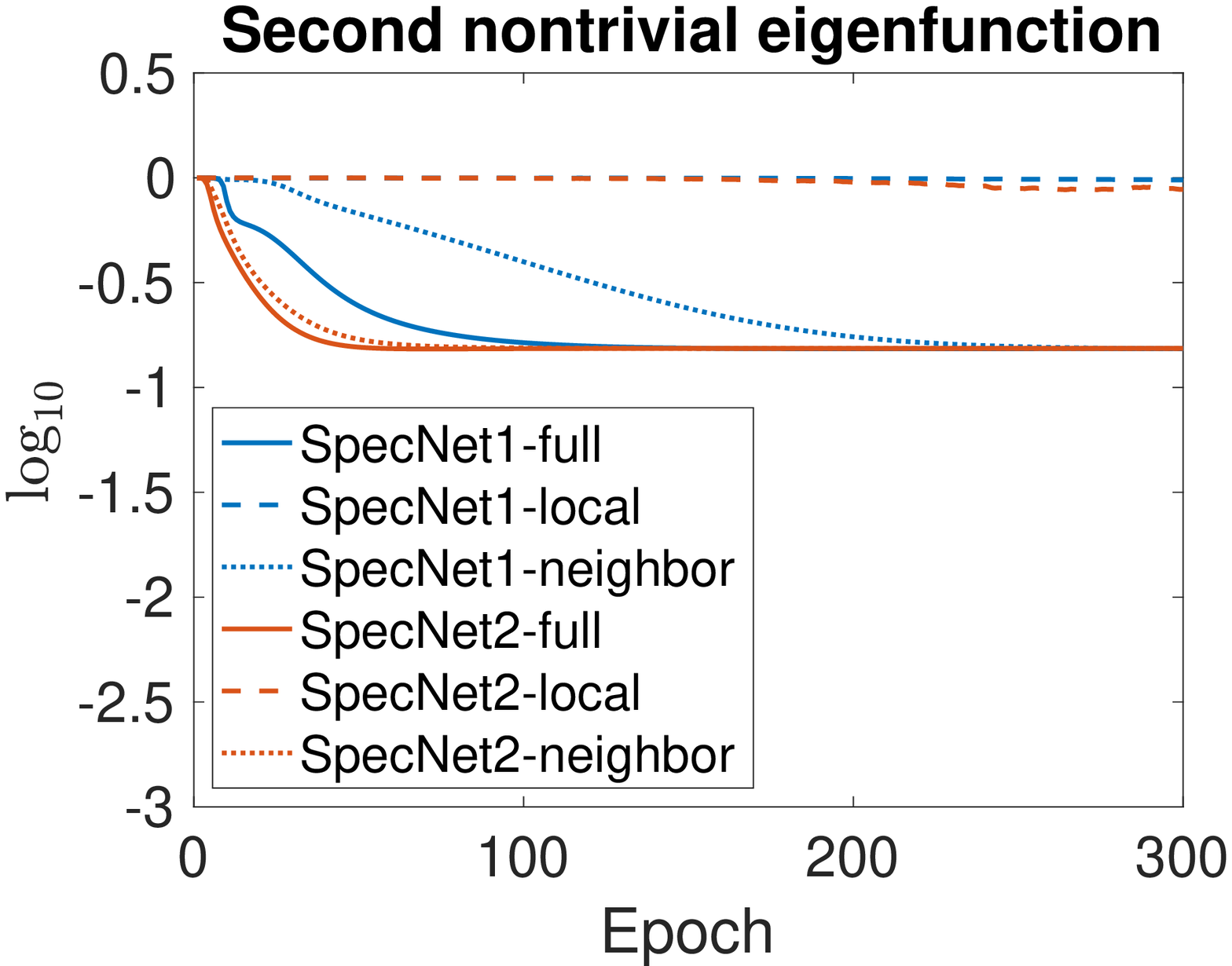}
    \caption{
    In the linear algebra problems of SpecNet1 and SpecNet2, relative errors of eigenvector (or eigenfunction, see titles of subfigures) approximations by
    different evaluation schemes on the one moon training dataset. The affinity matrix is $2000\times 2000$ and detail about how the data is generated can be found in Appendix \ref{detail:numres}. Legend refers to the update: SpecNet1-full~\eqref{eq:iterYfull1}, SpecNet1-local~\eqref{eq:iterYlocal1}, SpecNet1-neighbor~\eqref{eq:iterYneighbor1}, SpecNet2-full~\eqref{eq:iterYfull}, SpecNet2-local~\eqref{eq:iterYlocal}, SpecNet2-neighbor~\eqref{eq:iterYneighbor}. The relative error for eigenfunction approximations is defined in~\eqref{eq:relativeerror}, and the relative error for eigenvector approximation is defined as $\frac{\norm{\psi-\tilde{\psi}}_2}
    {\norm{\psi}_2}$,
    where $\psi$ is the true eigenvector of $(D,W)$ and $\tilde{\psi}$ is the corresponding column in $\hat{U}$ obtained through \eqref{eq:decouple} that approximates $\psi$ at each iteration.}
    \label{fig:onemoon-la-evecs}
\end{figure}

\begin{remark}
    In the linear algebra sense, both the gradient in
    \eqref{eq:gradYfull} and the gradient in
    \eqref{eq:gradYneighbor} are the same as the exact gradient of
    $f_2(Y)$ restricted to batch $\calB$.
    When the neural network gets
    involved, the full and neighbor gradient evaluation schemes become
    different, which will be discussed in section \ref{subsec:net2-training}. The
    gradient in \eqref{eq:gradYlocal}, however, is the
    gradient of $\tilde{f}_2$, 
    which is not
    $\nabla_\calB f_2$ unless $\calB = X$.
\end{remark}

We illustrate the convergence of three gradient evaluation schemes of $f_1$ and $f_2$ on a one moon
dataset, the visualization of which is shown in Figure~\ref{fig:moon}, and the results are shown in
Figure~\ref{fig:onemoon-la-evecs}. Here we choose constant stepsize for each method.
The formulation for computing relative errors can be found in
Appendix~\ref{detail:onemoon}.

\begin{remark}
    \label{rmk:2}
    As shown in Figure~\ref{fig:onemoon-la-evecs}, the full and neighbor
    gradient evaluation schemes of $f_1$ and $f_2$ can achieve good convergence, but the local scheme of either $f_1$ or $f_2$ does not converge. See Figure
    \ref{fig:onemoon-net-all} for the illustration in the neural network
    setting.
\end{remark}

\subsection{Computational cost of different schemes}

We study the computational cost of different gradient evaluation schemes, taking $f_2$ as an example. Consider a sparse affinity matrix with on average $s$ nonzeros on rows and
columns.

For the local evaluation scheme, the computational cost for the first part in
\eqref{eq:gradYlocal} is $O(\abs{\calB}^2 K)$ for $\abs{\calB}$ being the
cardinality of $\calB$ and the cost for the second part is
$O(\abs{\calB}K^2)$. The overall computational cost per batch step is then
$O(\abs{\calB}^2 K)$, assuming $\abs{\calB}\geq K$.

For the full evaluation scheme, the computational cost for the first part in
\eqref{eq:gradYfull} is $O(n \abs{\calB} K)$ and the cost for the second
part is $O(nK^2)$. The overall computational cost per batch step is $O( n
\abs{\calB} K)$, again assuming $\abs{\calB}\geq K$.

For the neighbor evaluation scheme, the computational cost for the first part
in \eqref{eq:gradYneighbor} is $O(s \abs{\calB} K)$, where $s$ is the number of neighbors of $\calB$. While the na{\"i}ve
computation of the third part in \eqref{eq:gradYneighbor} costs $O(nK^2)$
operations, same as the full update scheme. When dynamic updating is taken
into consideration at each step, only $Y$ restricted to $\calB$ is
updated, and the matrix $Y^\top D Y$ can be efficiently updated in
$O(\abs{\calB}K^2)$ operations. Hence we can dynamically update the matrix
$Y^\top D Y$ throughout iterations and the computation of the third part
in \eqref{eq:gradYneighbor} is reduced to $O(\abs{\calB}K^2)$. Similarly,
we can dynamically update the vector $\eta^{\top}Y$, and only those
restricted to $\calB$ is updated, and the second term can be updated in
$O(\abs{\calB}K)$ operations. The overall computational cost per batch
step is then $O(s\abs{\calB}K)$, assuming $s\geq K$. 

\section{Neural network parametrization and training}
\label{sec:nnmethod}
Inspired by the convergence results by two gradient evaluation schemes \eqref{eq:gradYfull} and \eqref{eq:gradYneighbor} as shown in Figure \ref{fig:onemoon-la-evecs}  as well as the theoretical guarantee for their convergence that we will prove later in Section \ref{sec:theorem}, we propose a neural network that can incorporate the linear algebra formulations in Section \ref{sec:minibatch}.

\subsection{Network parametrization of eigenfunctions}

In section \ref{sec:limiting} we mention that eigenvectors of the
graph Laplacian matrix can be viewed as the restriction of underlying
eigenfunctions of a limiting operator on the dataset $X$. \cite{shaham2018spectralnet} suggests we approximate those eigenfunctions by a neural
network. In this paper, we use a feedforward fully-connected neural network, and it can be extended to other types of neural networks, for example, convolutional neural network. Suppose the neural network computes a map $G_\theta:
\mathbb{R}^m\to\mathbb{R}^K$, where $\theta$ denotes the network weights.
Let $Y = G_\theta(X)$, so that each coordinate of $G_\theta$, $(G_\theta)_i$, $i=1,\dots,K$ is an approximation to an eigenfunction, and each column of $G_{\theta}(X)$ approximates an eigenvector of the graph Laplacian matrix. Our goal is to find a good approximation by training the neural network, SpecNet2, with the orthogonalization-free objective function $L(\theta) = f_2(Y)$.

\subsection{Network Training for SpecNet2}\label{subsec:net2-training}

In this subsection, we introduce the training of SpecNet2; that is, how to update $\theta$ to minimize $L(\theta) = f_2(Y)$. We have proposed three different gradient evaluation schemes in section \ref{sec:minibatch} to calculate the gradients in the block coordinate descent method to minimize $f_2(Y)$ in the linear algebra setup. In the neural network setting, note that $\frac{\partial L(\theta)}{\partial\theta} = \nabla_Y f_2(Y)\cdot\frac{\partial G_{\theta}(X)}{\partial\theta}$, we can also incorporate these gradient evaluation schemes to evaluate $\nabla_Y f_2(Y)$ in the training of a neural network. Let $\calB\subset X$ be the randomly sampled mini-batch, and
$\calN$ be the neighborhood of $\calB$. Note that unlike in the linear algebra setup where we can only update $Y$ on $\calB$, we are updating $\theta$ for the neural network, such that once $\theta$ is updated, not only $G_{\theta}(\calB)$ is different but also $G_{\theta}(\calB^c)$. We follow the notations as in
section \ref{sec:minibatch}, and we have different gradient evaluation
schemes for SpecNet2 as follows:

\textbf{\emph{Local evaluation scheme:}} At each batch step, we can
compute the neural network mapping of batch $\calB$ as $Y_\calB =
G_{\theta}(\calB)$, so that we can obtain $\nabla_\calB
\tilde{f}_{2}(Y_\calB)$ by plugging $Y_\calB$ into
\eqref{eq:gradYlocal}. Then we want to minimize
$\trace{Y_\calB(\theta)^\top \nabla_\calB \tilde{f}_{2}(Y_\calB)}$ and
update $\theta$ using the gradient of $\trace{Y_\calB(\theta)^\top
\nabla_\calB \tilde{f}_{2}(Y_\calB)}$ with respect to $\theta$ through
the chain rule, where inside the trace we write the first term $Y_\calB$
as $Y_\calB(\theta)$ to emphasize it is a function of $\theta$; and the
second term $\nabla_\calB \tilde{f}_{2}(Y_\calB)$ is detached and viewed
as constant. 

\textbf{\emph{Full evaluation scheme:}} At each batch step, we can compute
$\nabla_\calB f_{2}(Y_\calB)$ by plugging  $Y_\calB$ and $Y$ into
\eqref{eq:gradYfull}. Then we want to minimize
$\trace{Y_\calB(\theta)^\top \nabla_\calB f_{2}(Y_\calB)}$ and update
$\theta$ using the gradient of $\trace{Y_\calB(\theta)^\top \nabla_\calB
f_{2}(Y_\calB)}$ through the chain rule. And similarly, inside the trace
we only view $Y_\calB(\theta)$ as a function of $\theta$ but $\nabla_\calB
f_{2}(Y_\calB)$ as constant when computing the gradient.

\textbf{\emph{Neighbor evaluation scheme:}} We keep a record of two
matrices $(YDY)_{\star}$ and $Y_{0}$ throughout the training, where they
are initialized at the first iteration: $(YDY)_{\star} = Y^\top DY$ and
$Y_{0} = Y$, and detach both of them. At each batch step, we compute
$Y_\calN = G_{\theta}(\calN)$. Then we update $(YDY)_{\star} =
(YDY)_{\star} - Y_{0}(\calN)^\top D_{\calN}Y_{0}(\calN) + Y_\calN^\top
D_{\calN}Y_\calN$ followed by an update of $Y_{0}$ on $\calN$ as
$Y_{0}(\calN) = Y_\calN$. Both matrices are again detached. The gradient of
$f_2(Y)$ on $\calB$ is then evaluated as
\begin{equation*}
    \nabla_\calB \bar{f}_{2}(Y_\calB) = -\frac{4}{n}
    W_{\calB,\calN}Y_{\calN}+\frac{4}{n}\eta_{\calB}\eta^\top
    Y_{0} + \frac{4}{n^3}
    D_{\calB} Y_{\calB} (YDY)_{\star}.
\end{equation*}
Then we minimize $\trace{Y_\calB(\theta)^\top \nabla_\calB
\bar{f}_{2}(Y_\calB)}$ and update $\theta$ by computing the gradient of $\trace{Y_\calB(\theta)^\top \nabla_\calB \bar{f}_{2}(Y_\calB))}$ by the
chain rule. Similarly, inside the trace we only view $Y_\calB(\theta)$ as
a function of $\theta$ but $\nabla_\calB \bar{f}_{2}(Y_\calB)$ as constant
when computing the gradient.

Details about the network training for SpecNet1 can be found in Appendix~\ref{subsec:net1-training}. With those different learning objective functions from different evaluation schemes, we can choose an optimizer, for example, SGD or Adam, with some user-selected learning rate to update the network weights $\theta$. Detail for the choice in our experiments is introduced in Appendix~\ref{detail:numres}.

\begin{remark}\label{rmk:3} Note that while $\nabla_\calB
\bar{f}_{2}(Y_\calB)$ in \eqref{eq:gradYneighbor} is the exact gradient of
$f_2(Y)$ on $\calB$, the gradient $\nabla_\calB \bar{f}_{2}(Y_\calB)$ we
evaluate here in the neural network setting is no longer the exact
gradient of $f_2(Y)$ of $Y$ on $\calB$, but only an approximation.
\end{remark}

\section{Theoretical Analysis}
\label{sec:theorem}

In this section, we provide a theoretical guarantee for the performance of
SpecNet2 by analyzing the optimization iterations to minimize
\eqref{eq:uopt}. In Section~\ref{sec:energylandscape}, we discuss the
energy landscape of \eqref{eq:uopt}. Through our analysis, we show that
\eqref{eq:uopt} is a nonconvex function whose local minima are global
minima. In addition, we also give the explicit expression of the global
minima of \eqref{eq:uopt}, which span the same space as that of the
leading eigenvectors of matrix pencil $(W, D)$, assuming $D$ is positive definite. All analysis in this
section holds for general symmetric matrix $W$ and diagonal positive
definite matrix $D$ such that $(W, D)$ has at least $K$ positive
eigenvalues. Hence our results apply to deflated matrix pencil $(W - \eta
\eta^\top, D)$ as well. In Section~\ref{sec:globalconvergence}, based on
the energy landscape, we prove the global convergence of the gradient
descent method with full and neighbor evaluation schemes for all initial
points in a giant ball except a measure-zero set.

\subsection{Analysis of energy landscape}
\label{sec:energylandscape}

The explicit form of the local minimizers of \eqref{eq:uopt} are
explicitly given in Theorem~\ref{thm:minimizers}.
\begin{theorem}
    \label{thm:minimizers}
    The local minimizers of \eqref{eq:uopt} are of the form,
    \begin{equation} \label{eq:minimizerY}
        Y^\star = U \Lambda^{\frac{1}{2}} Q,
    \end{equation}
    where $U$ and $\Lambda$ are defined as in \eqref{eq:evp}, and $Q \in
    \bbR^{K \times K}$ denotes an arbitrary orthogonal matrix.
\end{theorem}
The proof of Theorem~\ref{thm:minimizers} can be found in
Appendix~\ref{app:thmproof}. Through the analysis, we find that $f_2(Y)$
is nonconvex and all local minimizers span the same space as the
eigenvectors of $(W, D)$ associated with the $K$ largest eigenvalues.

\begin{corollary}
    \label{cor:minimizers}
    All local minimizers of \eqref{eq:uopt} are global minimizers.
\end{corollary}
The proof of Corollary~\ref{cor:minimizers} can be found in Appendix~\ref{proof:cor5}. According to Theorem~\ref{thm:minimizers} and
Corollary~\ref{cor:minimizers}, the unconstrained optimization problem
\eqref{eq:uopt} does not have any spurious local minima and all local
minimizers are global minimizers. Furthermore, the target of our problem,
leading $K$ eigenpairs of $(W,D)$, can be extracted from the global
minimizers through a single step Rayleigh-Ritz method, as mentioned in
\eqref{eq:decouple}.

\subsection{Global convergence}
\label{sec:globalconvergence}

In this section, we prove the global convergence for the iterative
schemes, \eqref{eq:iterYfull} and \eqref{eq:iterYneighbor}, with full and
neighbor gradient evaluation schemes, respectively. The energy landscape
analysis in the previous section already hints at the global convergence
from the gradient flow perspective. Here, we give a rigorous statement and
its proof for the global convergence of our iterative scheme
\eqref{eq:iterYfull}, which can be applied to \eqref{eq:iterYneighbor}
directly.

The difficulties of the convergence analysis come from two aspects. First,
our objective function $f_2(Y)$ is a fourth-order polynomial of $Y$, and
its Hessian is unbounded from above for $Y \in \bbR^{n \times K}$ and so
is the Lipschitz constant. Second, the iterative scheme updates $Y$ on
different batches for different iterations. Hence the iterative mapping is
not fixed across iterations.

We first prove a few lemmas to overcome these difficulties and then
conclude the global convergence in Theorem~\ref{thm:convergence}. In
Lemma~\ref{lem:bound}, we define a giant ball with radius $R$ and prove
that our iterative scheme never leaves the ball. Given the bounded ball,
we then have a bounded Lipschitz constant being defined in
Lemma~\ref{lem:Lip} and a nonempty set for stepsize $\alpha$.
Lemma~\ref{lem:grad} shows that our iterative scheme converges to
first-order points of $f_2(Y)$. Combining these lemmas together with
results in \cite{lee2019first}, we prove the global convergence.

We define a set of notations to simplify the statements of lemmas and
theorem. The mini-batch technique partitions the dataset $X$ into disjoint
$b$ batches. We denote the index set of mini-batch partitions as $\{S_1,
S_2,\dots, S_b\}$ such that $S_p \cap S_q = \emptyset$ for $p \neq q$ and
$\cup_p S_p = [n]$. For an index $i$, $i^c$ denotes the complement
indices, \ie, $i^c = [n]\backslash \{i\}$. $D_i$ denotes the $i$-th
diagonal entry of $D$ and $Y_i$ denotes the $i$-th row of $Y$.
$Y^{(\ell)}$ denotes the iteration variable at $\ell$-th iteration.
Further, we define two constants and a function depending on entries of
$W$ and $D$,
\begin{equation*}
   M_1 := \max_i \frac{W_{i,i} + \sqrt{W_{i,i}^2 
    + D_i \norm{W_{i,i^c} D_{i^c}^{-\frac{1}{2}}}_2^2
    + \frac{D_i}{2}}}{2D_i},\,
    M_2 := \max_i \frac{W_{i,i}^2}{4D_i}
    + \frac{\norm{W_{i,i^c}D_{i^c}^{-\frac{1}{2}}}_2^2}{4},
\end{equation*}
and
$M(R) := 3\left(\max_i W_{i,i}^2 R^2 + \max_i D_i^2\cdot n^2K^2R^6
    + \max_i D_i \norm{W_{i,i^c} D_{i^c}^{-\frac{1}{2}}}_2^2
    \cdot nR^2\right)$,
where the $R$ will be the radius of the giant ball.

\begin{lemma}
    \label{lem:bound}
    Let $R$ be a constant such that $R \geq 2\sqrt{M_1}$ and $\alpha$ be
    the stepsize such that 
    \[
    \alpha < \min\{\frac{-2M_2 +
    \sqrt{4M_2^2+3M(R)R^2}}{8M(R)}, \frac{1}{16M(R)}\}.
    \]
    Then for any $Y^{(\ell)} \in W_0 = \{ Y \in \bbR^{n\times K} : \max_i
    \norm{D_i^{\frac{1}{2}}Y_i}_2 < R \}$, we have
     $Y^{(\ell+1)} \in W_0$.
\end{lemma}

\begin{lemma}
    \label{lem:Lip}
    For any $1\leq i_1,i_2\leq n$, $1\leq k_1,k_2\leq K$ and $Y\in W_0$
    with $R \geq 2 \sqrt{M_1}$, we have 
    \[
        \abs{\frac{\partial^2 f_2}{\partial Y_{i_1,k_1}
        \partial Y_{i_2,k_2}}} \leq 4\max_{i,j}W_{i,j} + 4(n+K)R^2\max_{i}D_i.
    \]
\end{lemma}

We define the upper bound in Lemma~\ref{lem:Lip} as 
\begin{equation}
    L:=4\max_{i,j}W_{i,j} + 4(n+K)R^2\max_{i}D_i,
\end{equation}
which is a Lipschitz constant of $\nabla f_2$ in the coordinate sense.

We denote the iterative mapping as
$Y^{(\ell+1)} = g_p(Y^{(\ell)})$, which is the block coordinate update at
the $\ell$-th iteration on batch $S_p$. Our iterative scheme then applies
$g_{1}, \dots, g_{b}$ in a cyclic way. When contiguous $b$ iterations of
our iterative scheme are applied, we could view it as a composed iterative
mapping as,
\begin{equation}
    g = g_{b} \circ g_{b-1}\circ\cdots \circ g_1,
\end{equation}
and the corresponding iteration is
\begin{equation}
    Y^{((i+1)b)} = g(Y^{(ib)}),\quad i=0,1,2,\dots.
\end{equation}
Though mapping $g$ is not explicitly shown in the statements of
Lemma~\ref{lem:grad} and Theorem~\ref{thm:convergence}, their proofs
rely on the detailed analysis of $g$.

\begin{lemma}
    \label{lem:grad}
    Suppose $\alpha$ is sufficiently small such that $\alpha <
    \frac{1}{L}$. Then the iteration converges to first-order points, \ie,
    \[
        \lim_{\ell\to\infty}\norm{\nabla f_2(Y^{(\ell)})}=0.
    \]
\end{lemma}

With all these lemmas available, we then show the global convergence of
our iterative scheme with full gradient evaluation scheme
\eqref{eq:iterYfull}, in Theorem~\ref{thm:convergence}. The proof is based
upon the stable manifold theorem~\cite{lee2019first}.

\begin{theorem}[Global Convergence]
    \label{thm:convergence}
    Let $R \geq 2\sqrt{M_1}$ be a constant and suppose the stepsize
    satisfies that
    \[
    \alpha < \min \left\{\frac{-2M_2+\sqrt{4M_2^2+3M(R)R^2}}{8M(R)},
    \frac{1}{16M(R)},\frac{1}{KL\max_{i\in [b]}\abs{S_i}} \right\}.
    \] 
    Then the  iteration \eqref{eq:iterYfull} converges to global minimizers of
    \eqref{eq:uopt} for all $Y^{(0)} \in W_0$ up to an initial point set of
    measure zero.
\end{theorem}

Proofs of Lemma~\ref{lem:bound}, Lemma~\ref{lem:Lip}, Lemma~\ref{lem:grad}
and Theorem~\ref{thm:convergence} are provided in
Appendix~\ref{app:proofthm3}. We emphasize that the iterative scheme with
full gradient evaluation scheme and neighbor gradient evaluation scheme
are identical in the linear algebra sense. Hence the iterative scheme with
the neighbor gradient evaluation scheme, \eqref{eq:iterYneighbor}, also
admits the same global convergence property.

\section{Numerical Experiments}
\label{sec:numres}
We compare the performance of SpecNet2 with SpecNet1 through an ablation study: That is, all the setup of SpecNet1 is the same as SpecNet2 except that SpecNet1 has one additional orthogonalization layer appended to the output layer of SpecNet2. Details about the data generation, network architecture and parameters can be found in Appendix~\ref{detail:numres}. The code is available at
\url{https://github.com/ziyuchen7/SpecNet2}. 

\subsection{One moon data}\label{sec:onemoon}

The visualization of one moon data can be found in Figure~\ref{fig:moon}. Training data and testing data both consists of 2000 samples. Figure~\ref{fig:onemoon-net-all} demonstrates the performance of
both methods with all three gradient evaluation schemes on the one moon data. We also compare the computational efficiency of full and neighbor schemes in
Figure~\ref{fig:onemoon-net-cost}, and its detail can be found in Appendix~\ref{detail:onemoon}. We observe in Figure~\ref{fig:onemoon-net-all} that SpecNet1-full,
SpecNet2-full and SpecNet2-neighbor can provide good approximations to
the first two nontrivial eigenfunctions; SpecNet1-local, SpecNet2-local,
and SpecNet1-neighbor give poor approximations as their relative errors
are significantly larger. In Figure~\ref{fig:onemoon-net-cost}, we see
that the relative error for the first nontrivial eigenfunction by
SpecNet2-neighbor reaches the plateau earlier than
SpecNet2-full in terms of the computational cost, while they can achieve similar accuracy. We also show the embedding results provided by different methods in Figure \ref{fig:onemoon-embed} in the Appendix.

\begin{figure}[t]
    \includegraphics[width=0.245\textwidth]{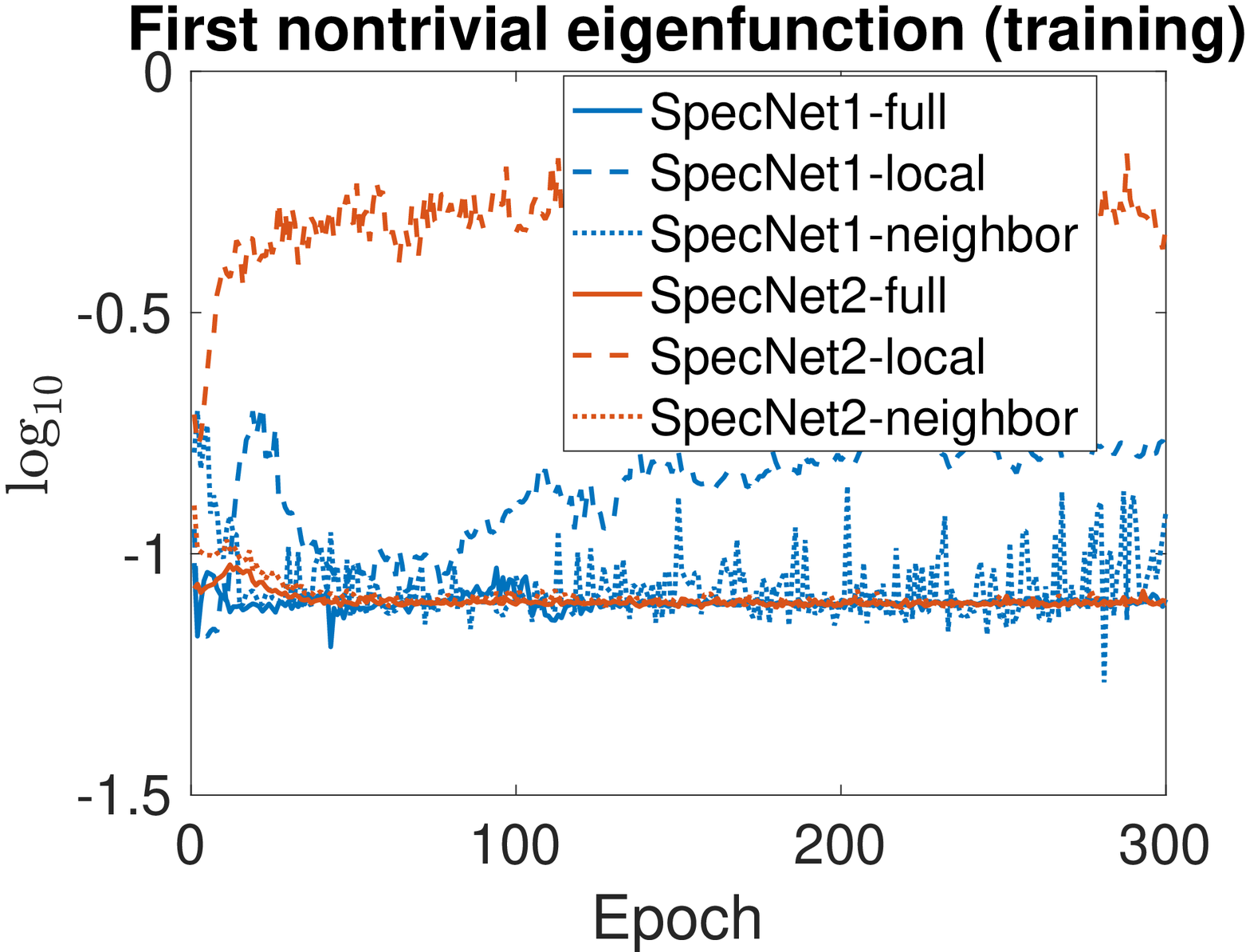}
    \includegraphics[width=0.245\textwidth]{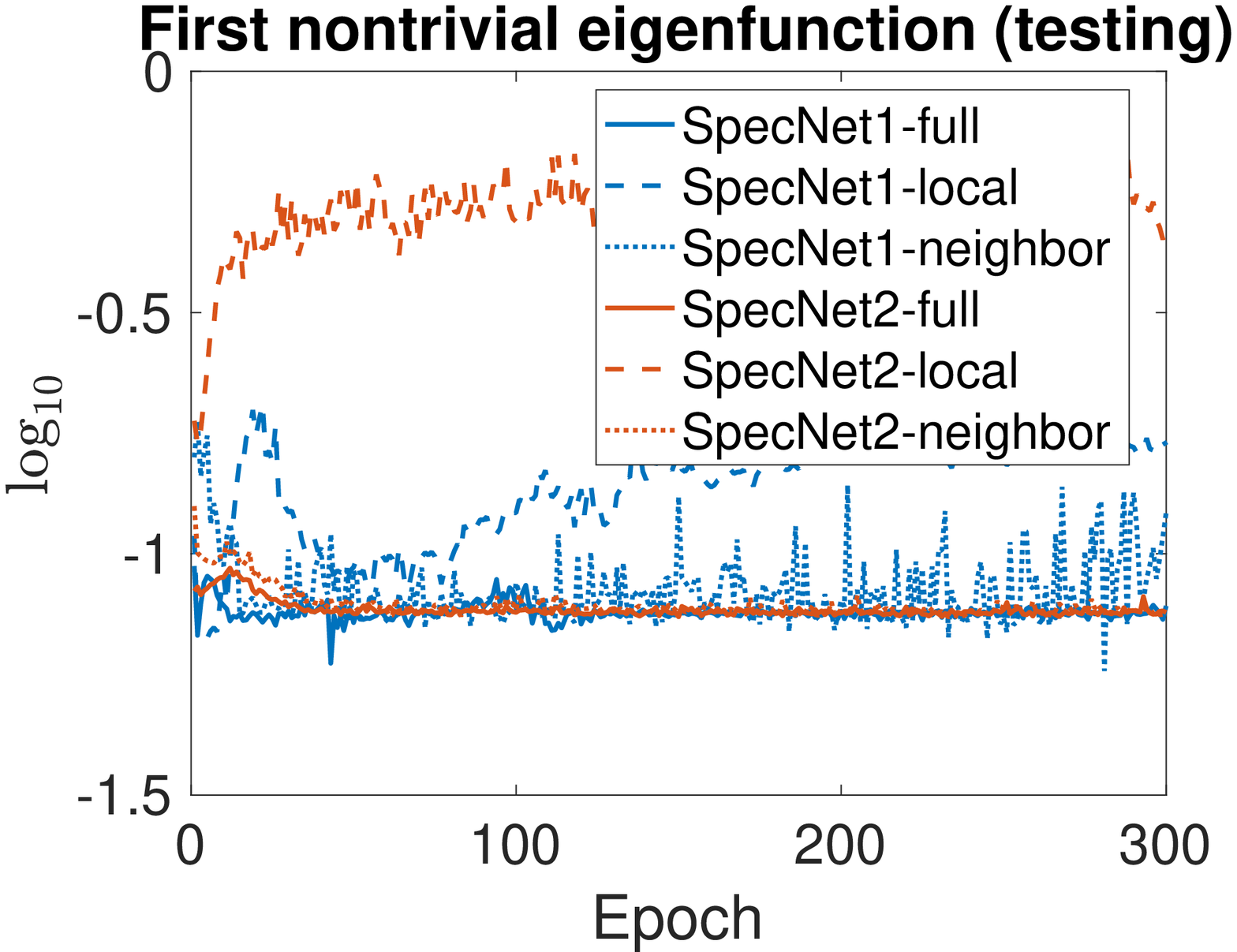}
    \includegraphics[width=0.245\textwidth]{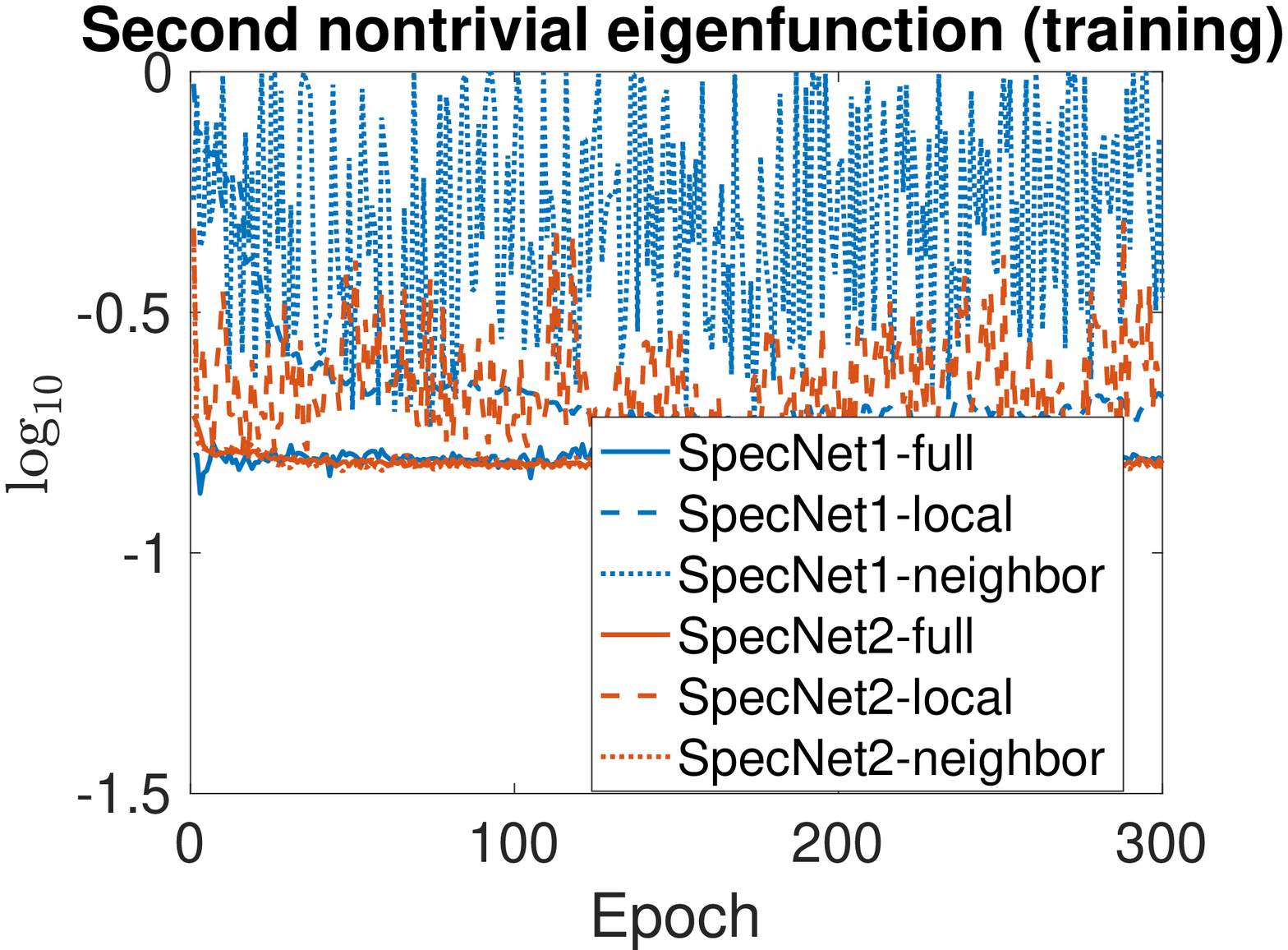}
    \includegraphics[width=0.245\textwidth]{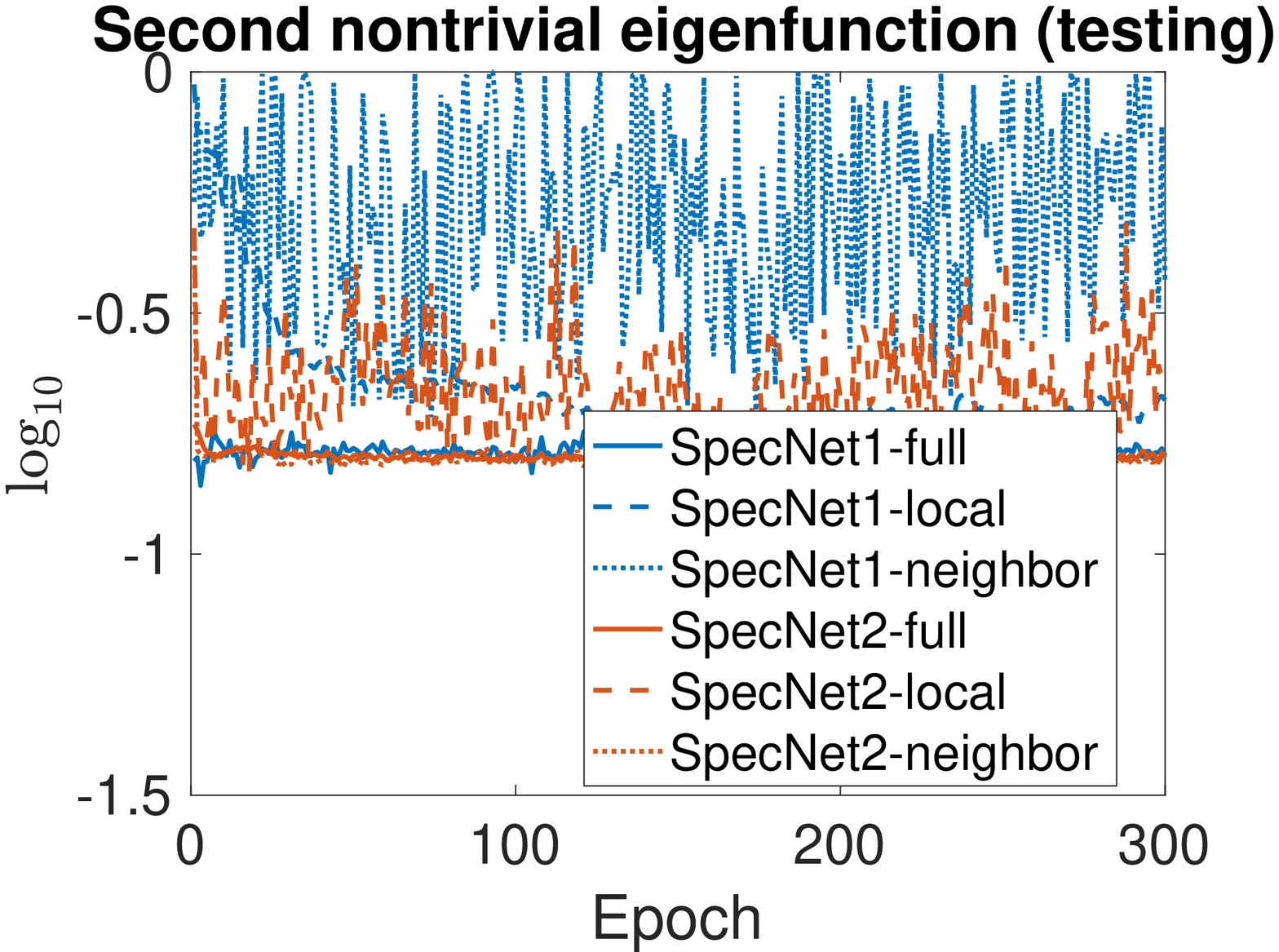}
    \caption{In the neural network setting of SpecNet1 and SpecNet2, relative errors of eigenfunction approximations by
    different evaluation schemes. SpecNet1-full, SpecNet1-local and  SpecNet1-neighbor are introduced in Section \ref{subsec:net1-training}; SpecNet2-full, SpecNet2-local and SpecNet2-neighbor are introduced in Section \ref{subsec:net2-training}. The relative error for training and testing is defined below~\eqref{eq:relativeerror}.
    }
    \label{fig:onemoon-net-all}
\end{figure}

\begin{figure}[t]
    \centering
    \includegraphics[width=0.25\textwidth]{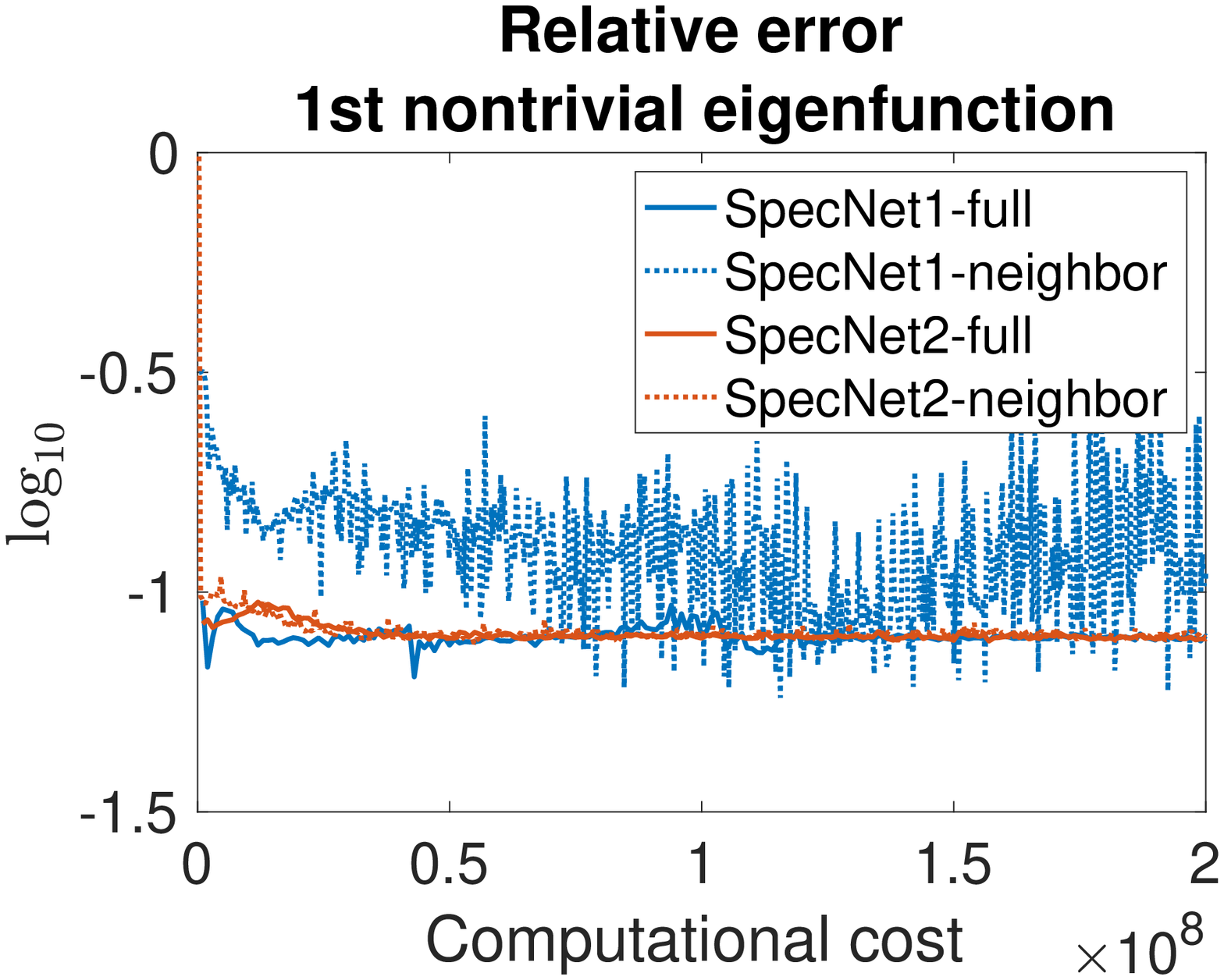}
    \includegraphics[width=0.25\textwidth]{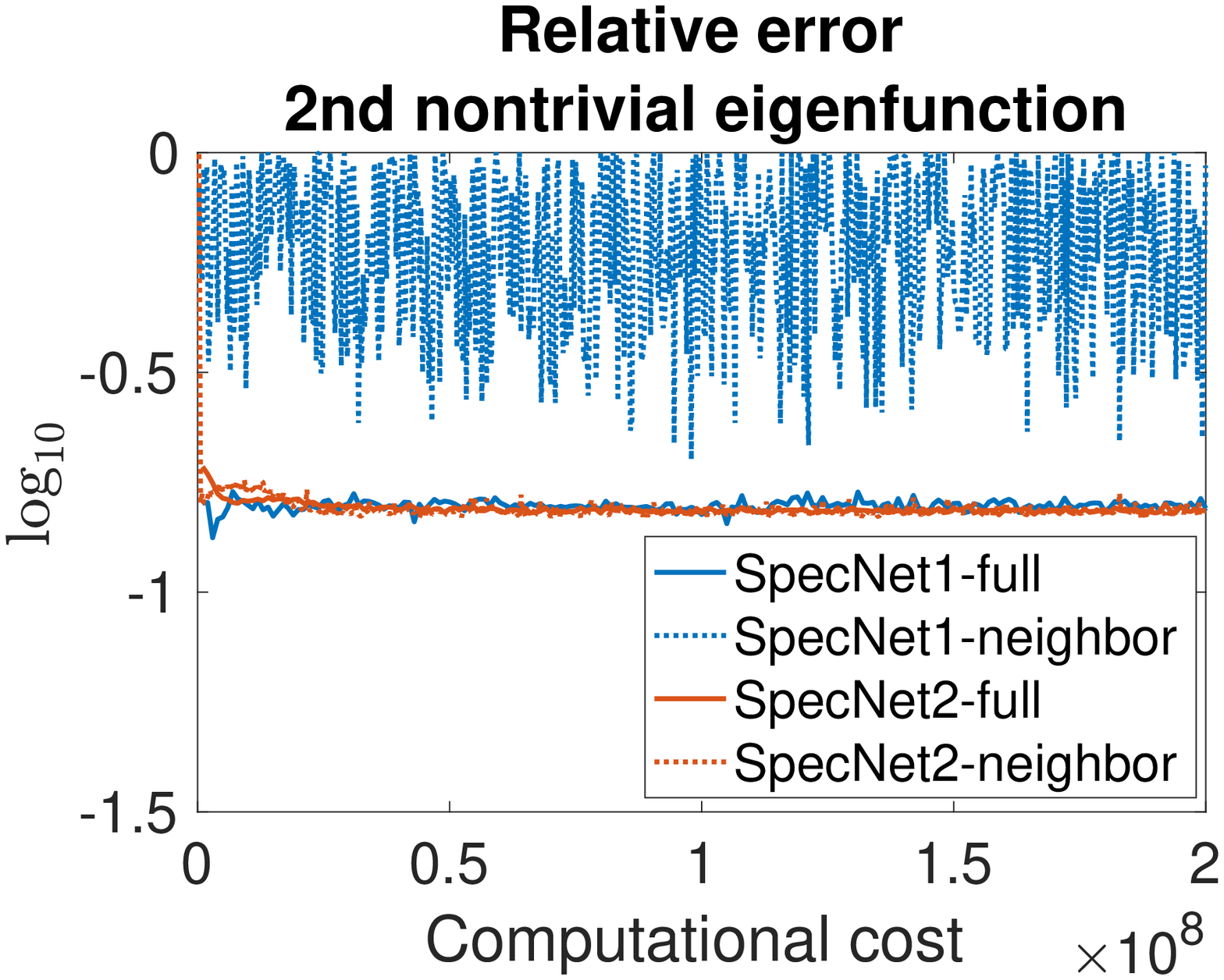}
    \caption{Results of full and neighbor schemes of SpecNet1 and SpecNet2 on the training data, where we rescale the $x$-axis to the computational cost.}
    \label{fig:onemoon-net-cost}
\end{figure}

\subsection{Two moons data}

We compare the performance and stability of SpecNet2 with SpecNet1 through an unsupervised
clustering task on a two moons dataset (visualized in Figure~\ref{fig:moon}) that contains 2000 training samples and 2000 testing samples. Due to the savings in memory and computational cost, we only compare SpecNet2-neighbor with
SpecNet1-neighbor in this example. Figure~\ref{fig:twomoons-net} shows the classification performance of
SpecNet1-neighbor and SpecNet2-neighbor over 10 different realizations of
the neural network. We observe that though the average curves provided by SpecNet1-neighbor
and SpecNet2-neighbor are close to each other, the variance of SpecNet1-neighbor is much
larger than that of SpecNet2-neighbor. Hence, we conclude that
SpecNet2-neighbor is able to achieve similar average classification accuracy as SpecNet1-neighbor
but with much higher reliability.

\begin{figure}[t]
    \centering
    \includegraphics[width=0.25\textwidth]{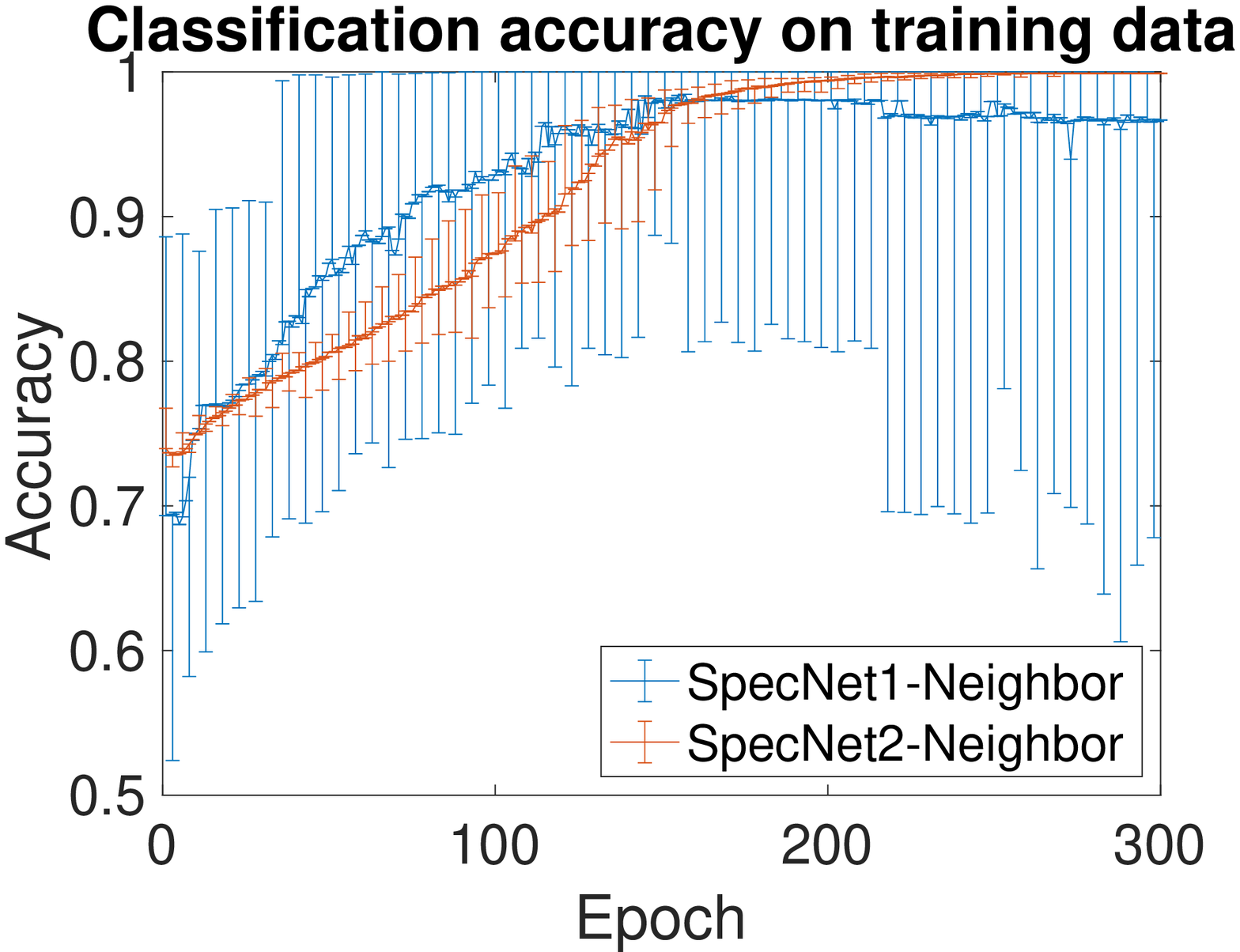}
    \includegraphics[width=0.25\textwidth]{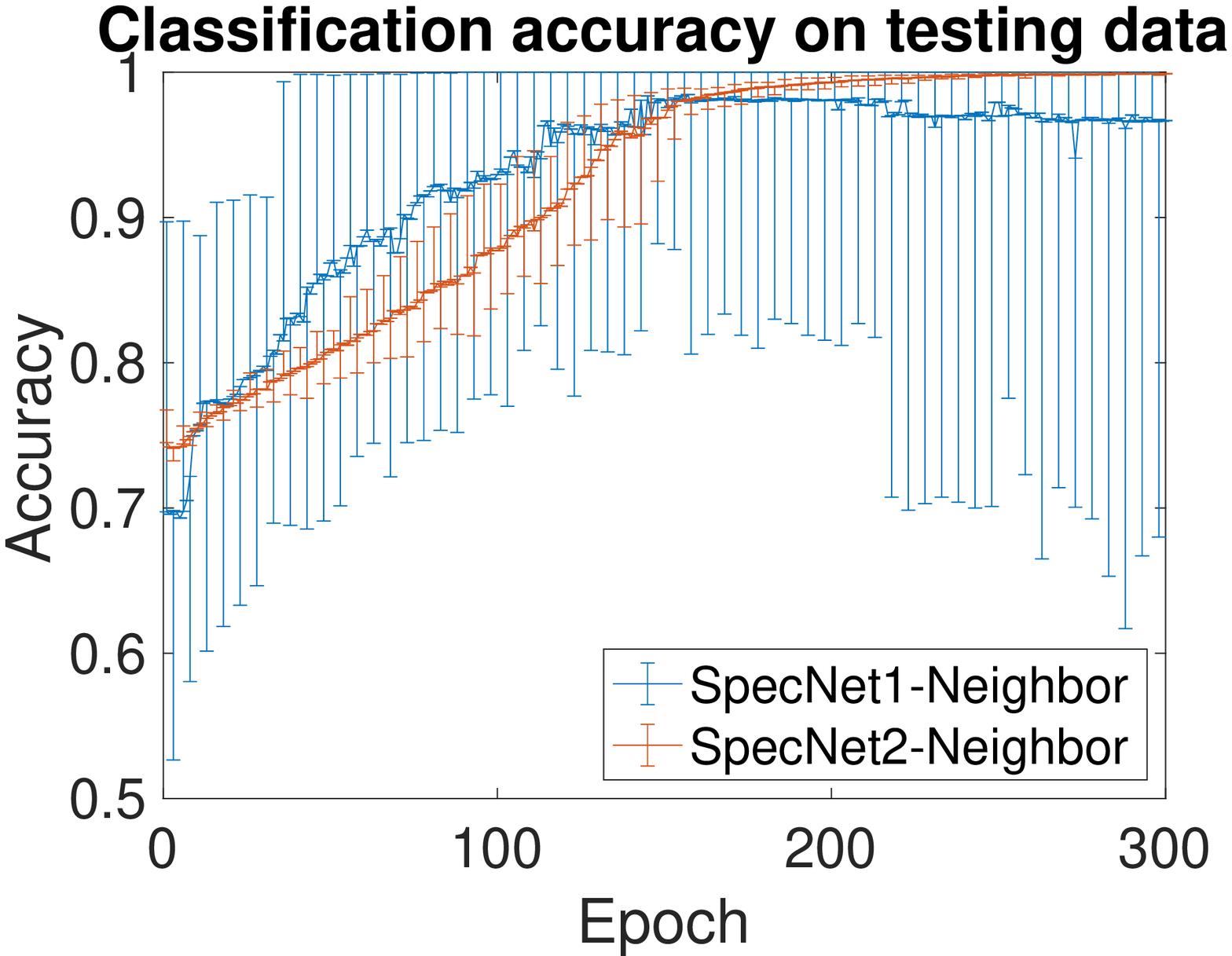}
    \caption{Classification results on two moons dataset. Average
    classification accuracy and errorbars are plotted over 10 different
    network initializations. At each epoch, the classification and its
    accuracy is computed in an unsupervised way.}
    \label{fig:twomoons-net}
\end{figure}

\subsection{MNIST data}

In this experiment, we use 20,000 samples of MNIST data (gray-scale images of hand-written digits which are of size $28 \times 28$) as the training set and 10,000 samples for testing. 
We construct the adjacency matrix $A$ of an kNN graph on the training set by setting $A_{i,j} = 1$ if the $j$-th training sample is within 
$k$ nearest neighbors of the $i$-th training sample and $A_{i,j} = 0$ otherwise, and we use $k = 16 $. The affinity matrix $W $ is obtained by setting $W = \frac{1}{2}(A + A^\top)$. We compare the performance of SpecNet1-local with SpecNet2-neighbor with different batch sizes. Specifically, the batch sizes for SpecNet2-neighbor are 2, 4 and 8 and those for SpecNet1-local are 45, 90, 180 (the average numbers of neighbors of a batch of size 2, 4 and 8 are about 45, 90 and 180 respectively).

Figure~\ref{fig:MNIST-loss} shows the losses $f_1$ and $f_2$ (defined in \eqref{eq:specnet1} and \eqref{eq:uopt2} respectively) over the training epochs. Since the minimum of $f_1$ and $f_2$ are not necessarily zero, we plot the quantities $\log_{10}(f_1(Y) - f_1^\star)$ and $\log_{10}(f_2(Y) - f_2^\star)$, where $f_1^\star = K - \sum_{i=1}^K\lambda_i$ and $f_2^\star = \sum_{i=2}^K\lambda_i^2$ are the global minimums (over matrix $Y$) of $f_1$ and $f_2$ respectively.
(In the definition of $f_1^\star$ and $f_2^\star$, $\lambda_1\geq\lambda_2\dots\geq\lambda_K$ are the $K$ largest eigenvalues of $D^{-1}W$, $D$ being the degree matrix of $W$.) 
The values of $(f_i(Y) - f_i^\star)$, $i=1,2$, in the plots are computed over 10 replicas of random initialization of the neural network. The solid curve shows the average over the replicas, and the shaded area around each curve reveals the standard deviation. 
We observe that though SpecNet2-neighbor has larger variance compared to SpecNet1-local, SpecNet2-neighbor achieves better performance in average when the batch size is small, e.g., comparing SpecNet2-neighbor with batch size 2 with SpecNet1-local with batch size 45.

\begin{figure*}[htbp]
    \centering
    \includegraphics[width=0.245\textwidth]{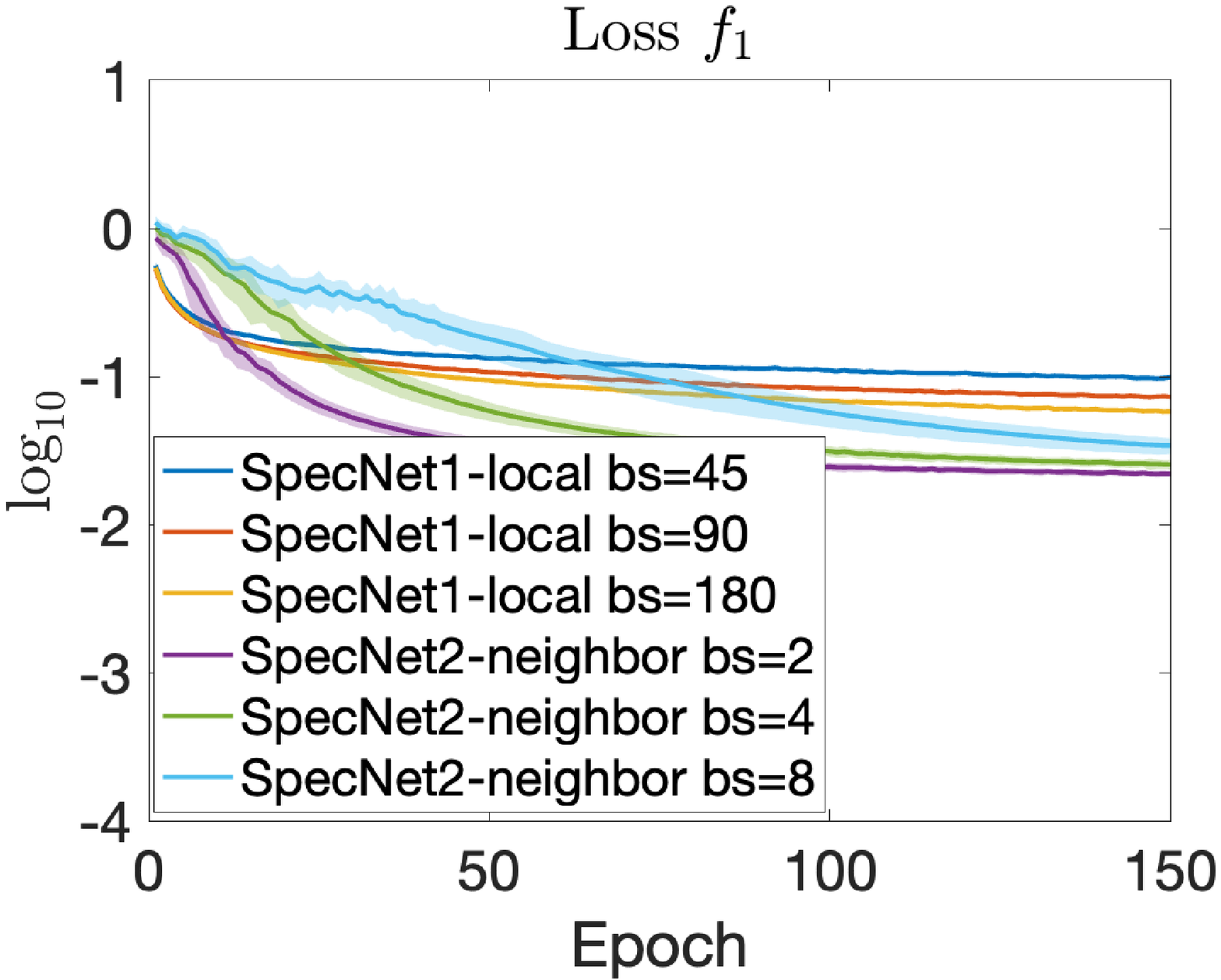}
    \includegraphics[width=0.245\textwidth]{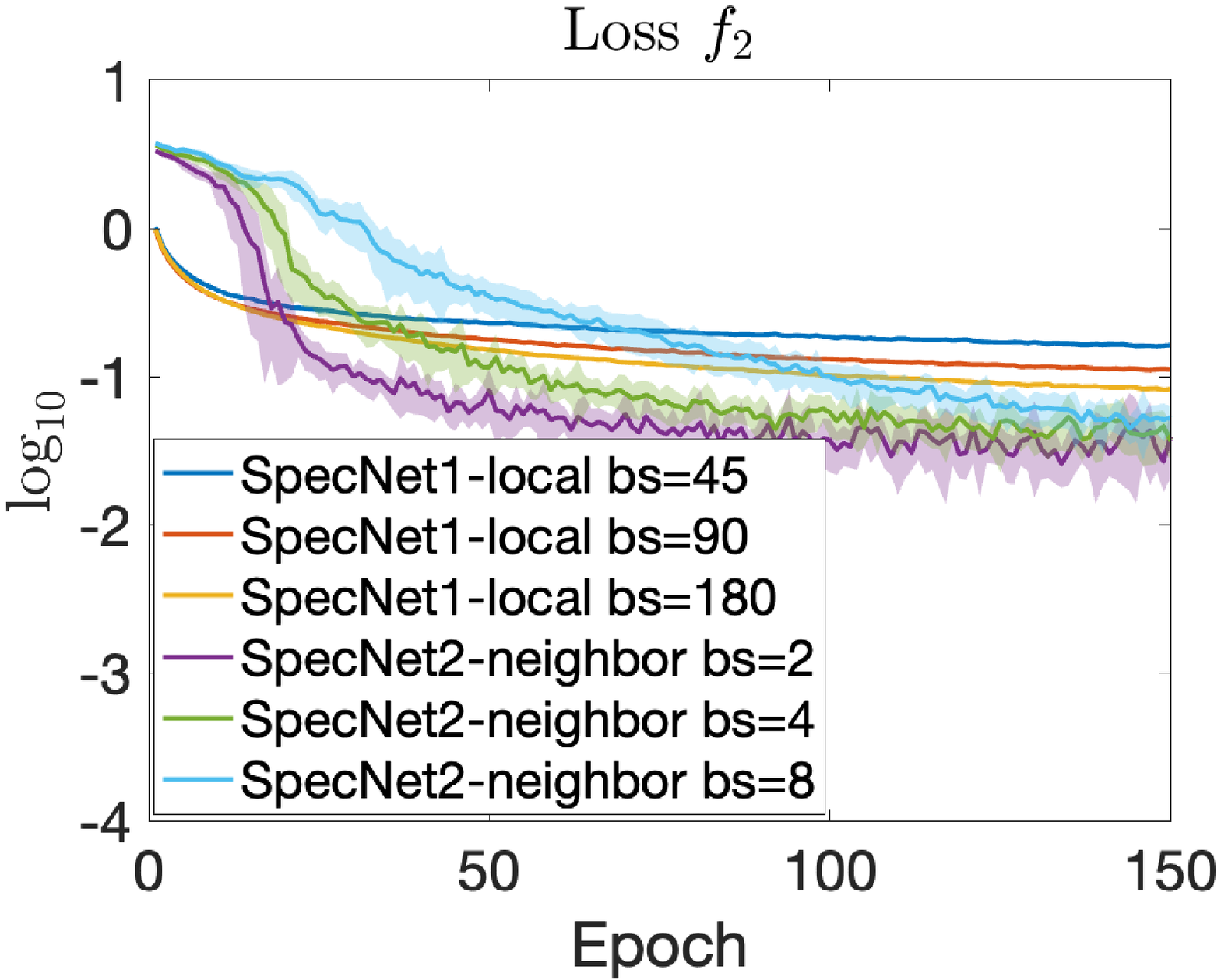}
    \includegraphics[width=0.245\textwidth]{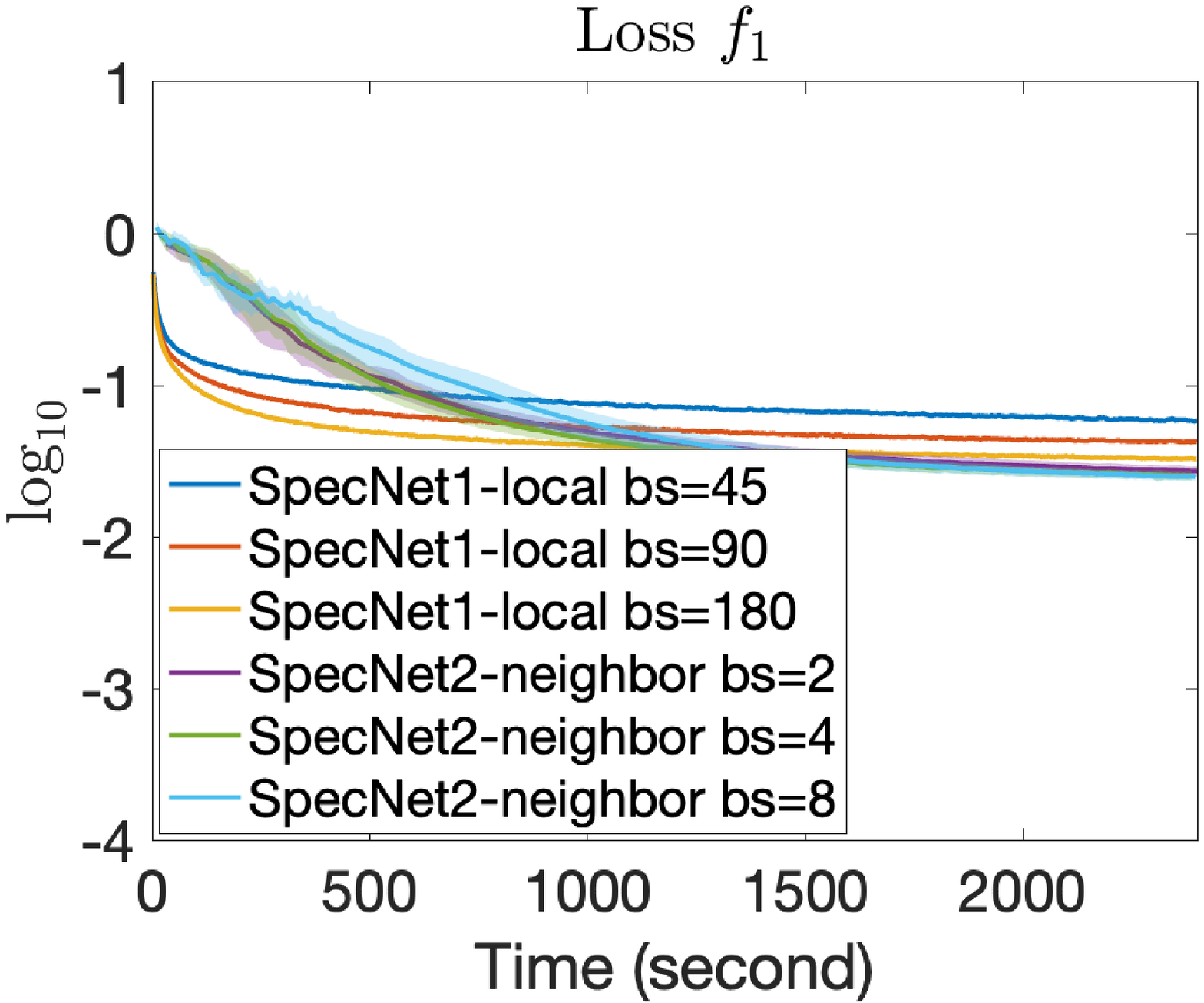}
    \includegraphics[width=0.245\textwidth]{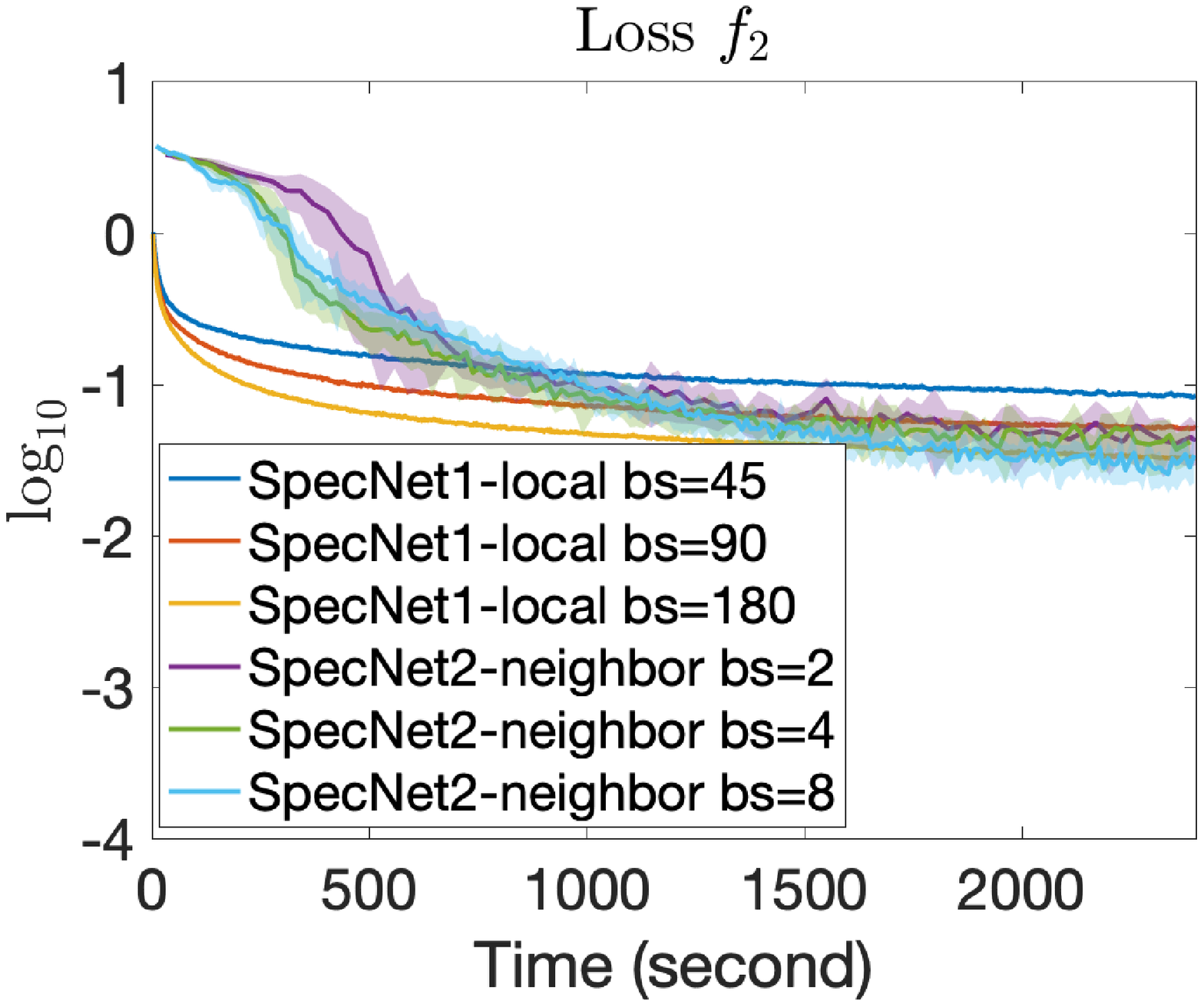}
    \caption{
    MNIST datset: plot of two different losses $\log_{10}(f_1(Y) - f_1^\star)$ and $\log_{10}(f_2(Y) - f_2^\star)$ over epochs (left two subfigures) and over time (right two subfigures), with $f_1$ and $f_2$ defined in \eqref{eq:specnet1} and \eqref{eq:uopt2}. Networks are trained on 20000 MNIST images on a 2021 14-inch Macbook Pro with an 8-core CPU. 
    }
    \label{fig:MNIST-loss}
\end{figure*}

Figure~\ref{fig:MNIST-embed} shows the embeddings (on both training and testing sets) at the 50-th training epoch, computed by SpecNet2-neighbor (with batch size 2) and SpecNet1-local (with batch size 45) respectively. By comparing to the true spectral embeddings (by linear algebra eigenvectors) plotted in the top panel, we can see that SpecNet2-neighbor gives a better result, and this is consistent with the lower value of losses of SpecNet2-neighbor in Figure \ref{fig:MNIST-loss}. As shown in Figure \ref{fig:MNIST-embed}, the embedding on test set is close to that on the training set, and this demonstrates the out-of-sample extension ability of SpecNet2.

\begin{figure*}[t]
    \begin{center}
    \includegraphics[width=0.245\textwidth]{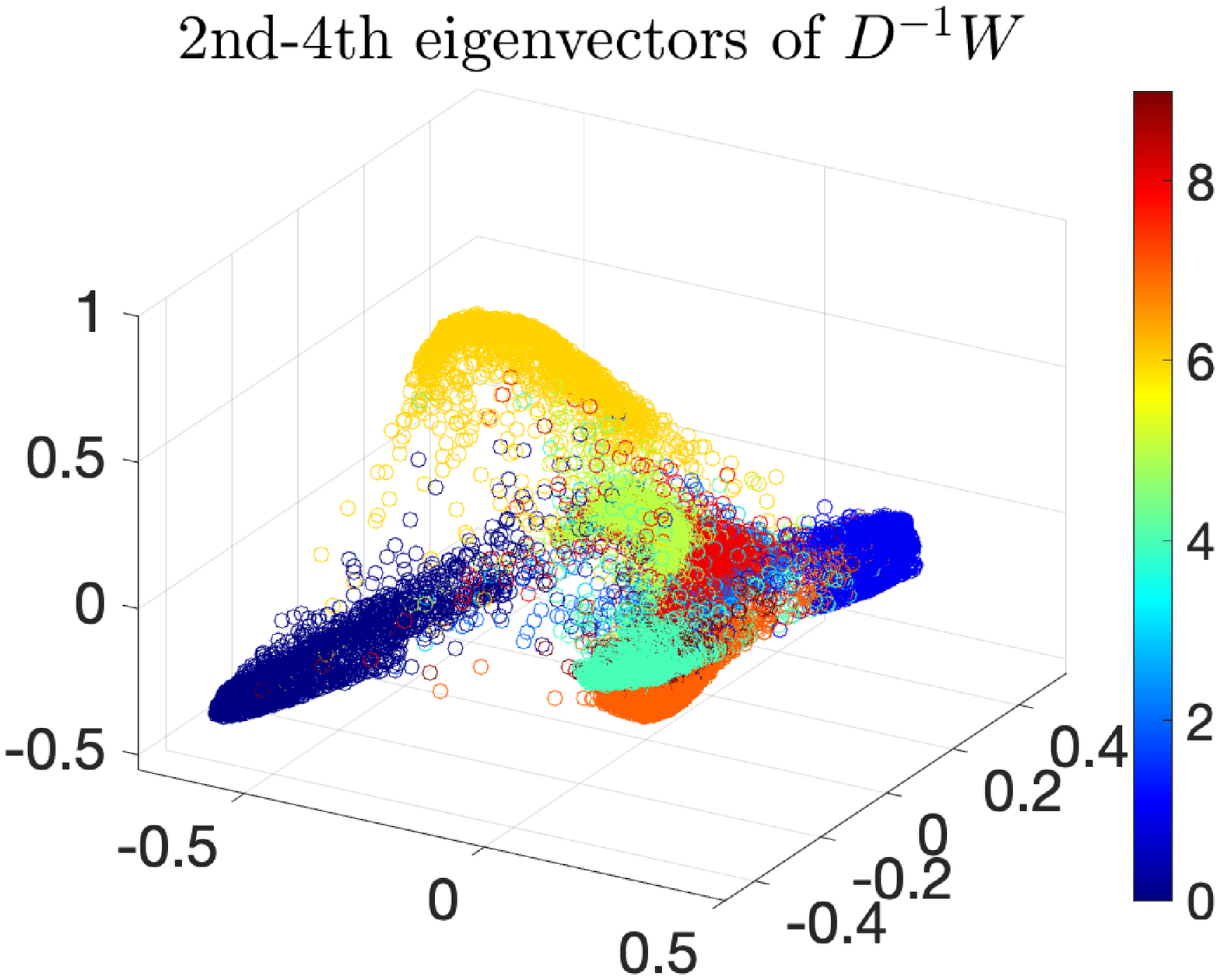}
    \includegraphics[width=0.245\textwidth]{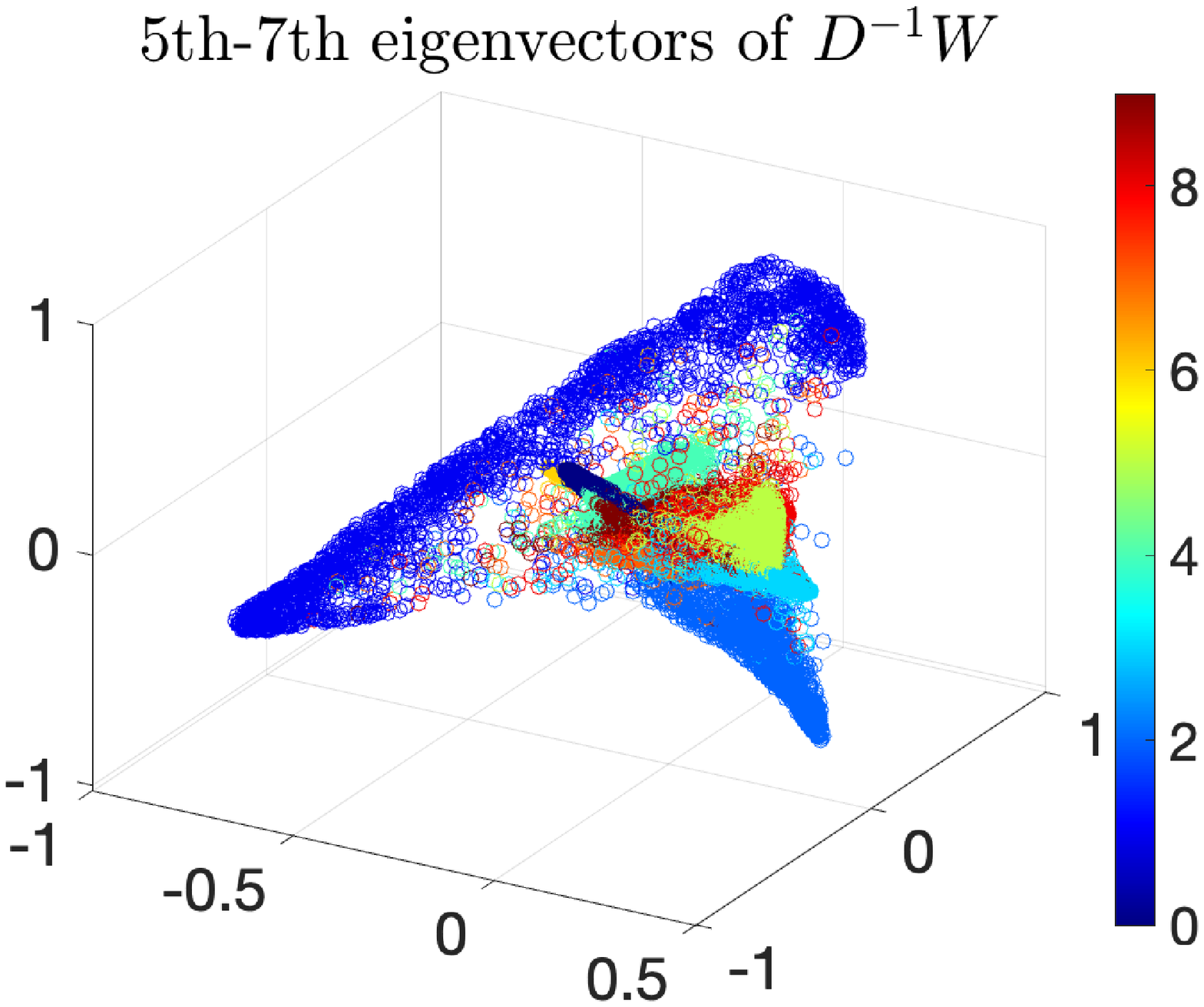}
    \end{center}
    \includegraphics[width=0.245\textwidth]{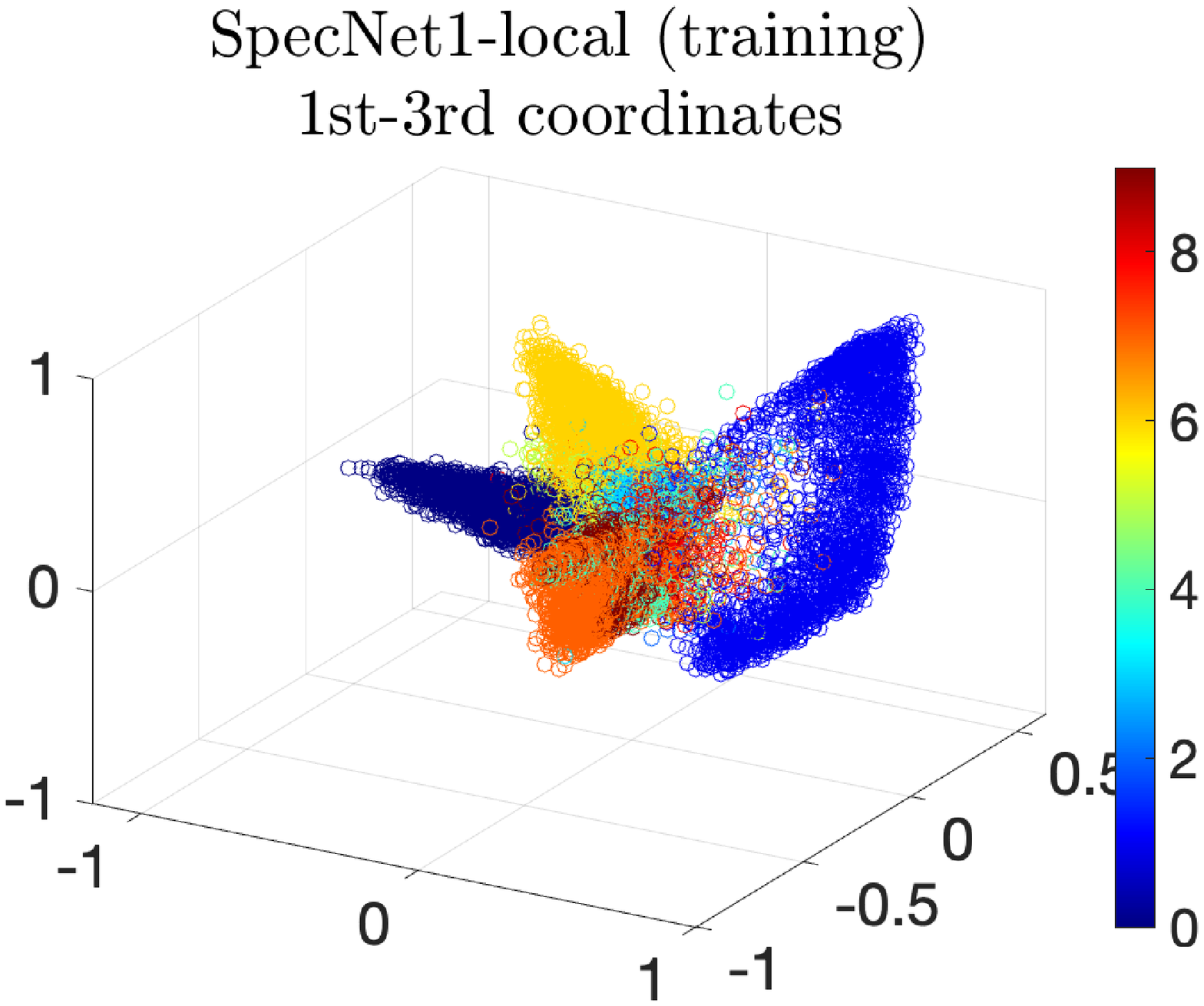}
    \includegraphics[width=0.245\textwidth]{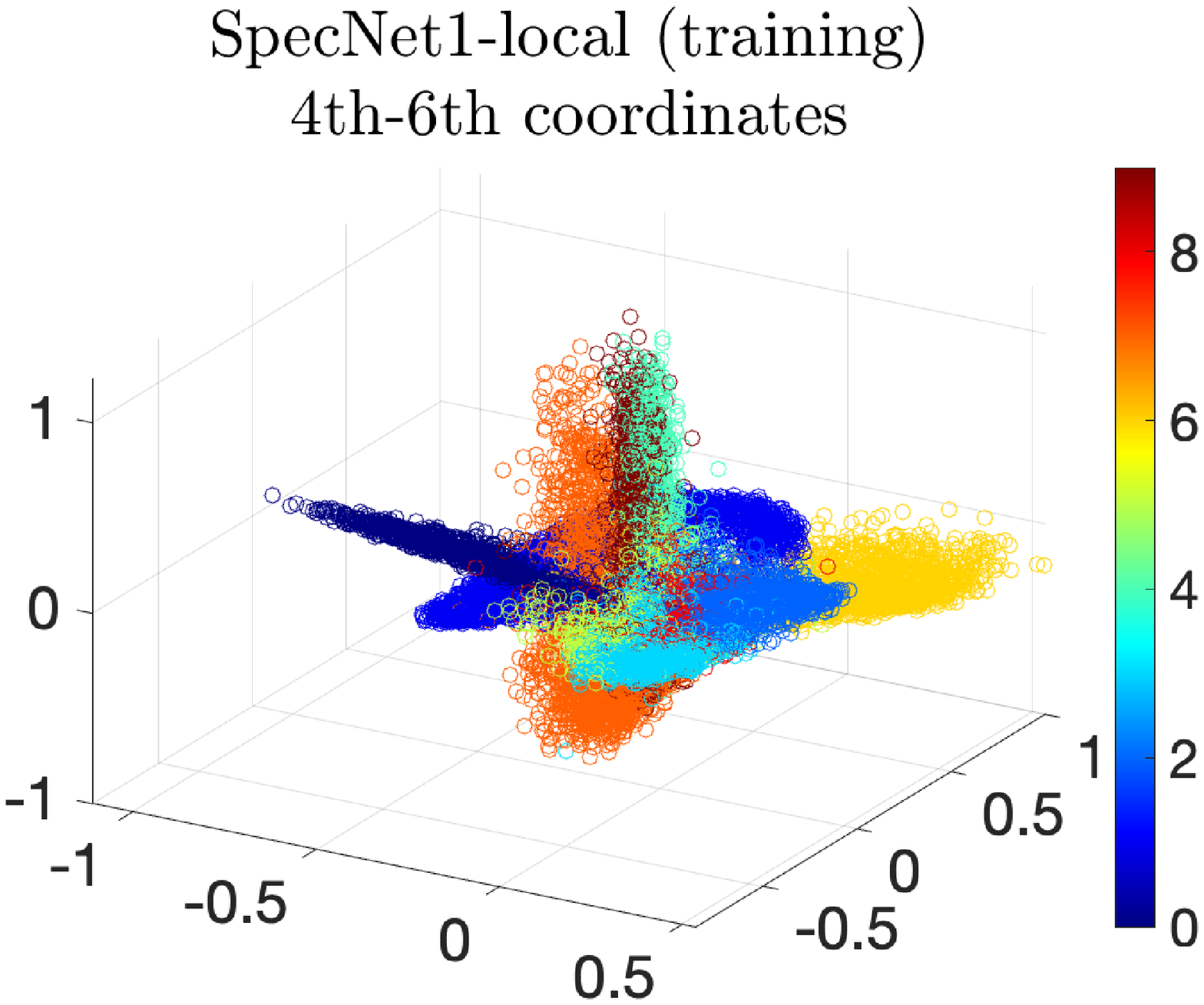}
    \includegraphics[width=0.245\textwidth]{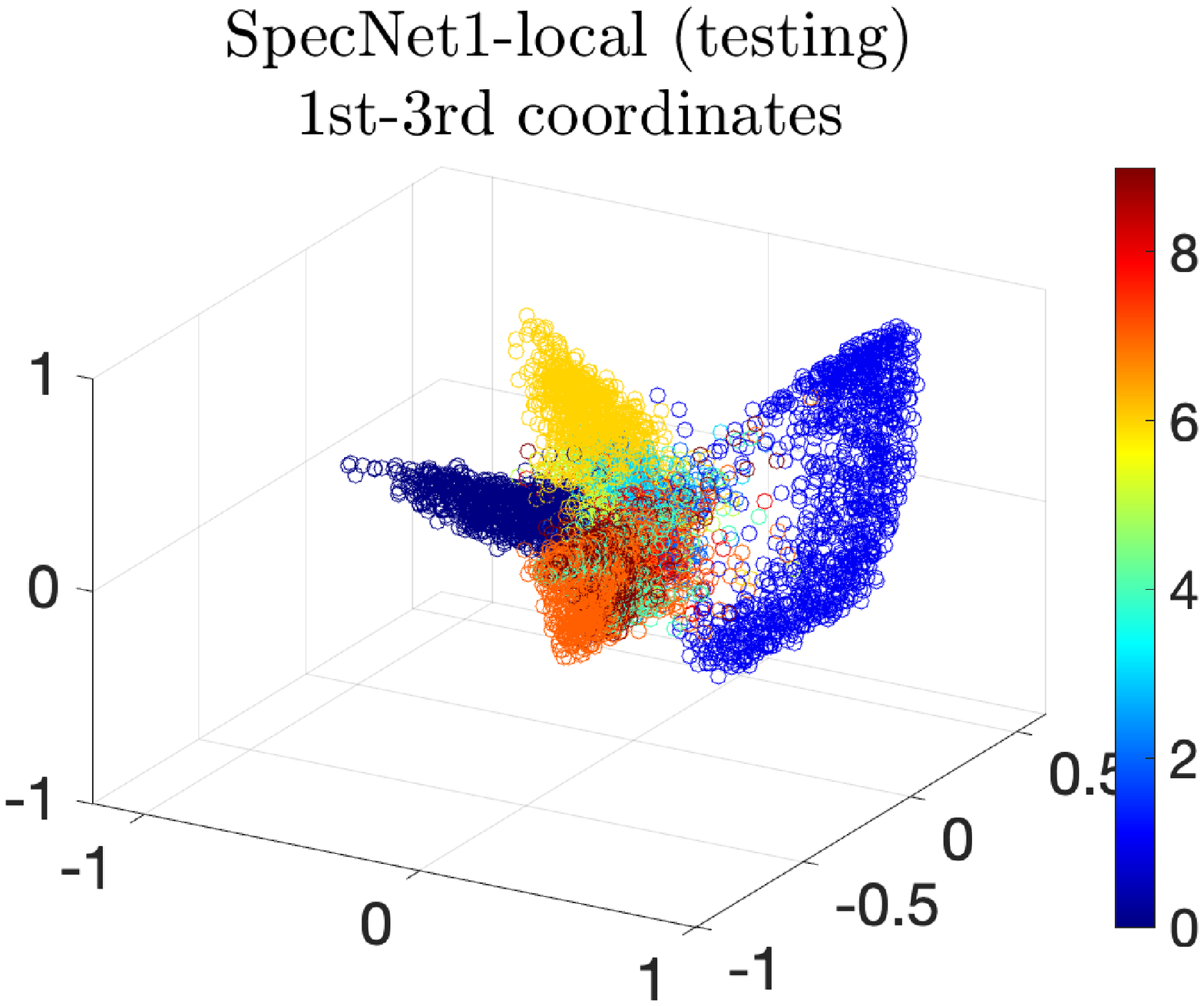}
    \includegraphics[width=0.245\textwidth]{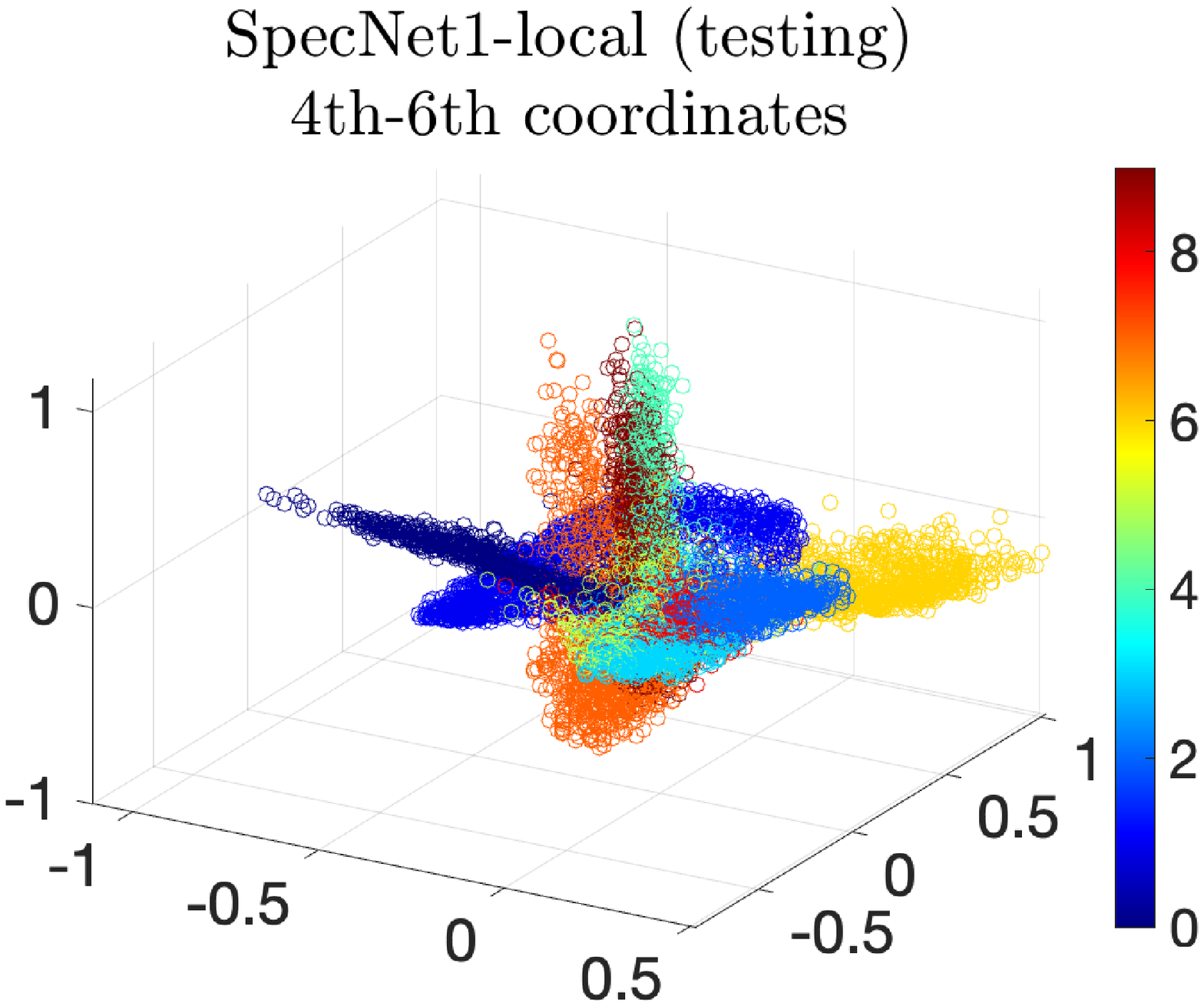}\\
    \includegraphics[width=0.245\textwidth]{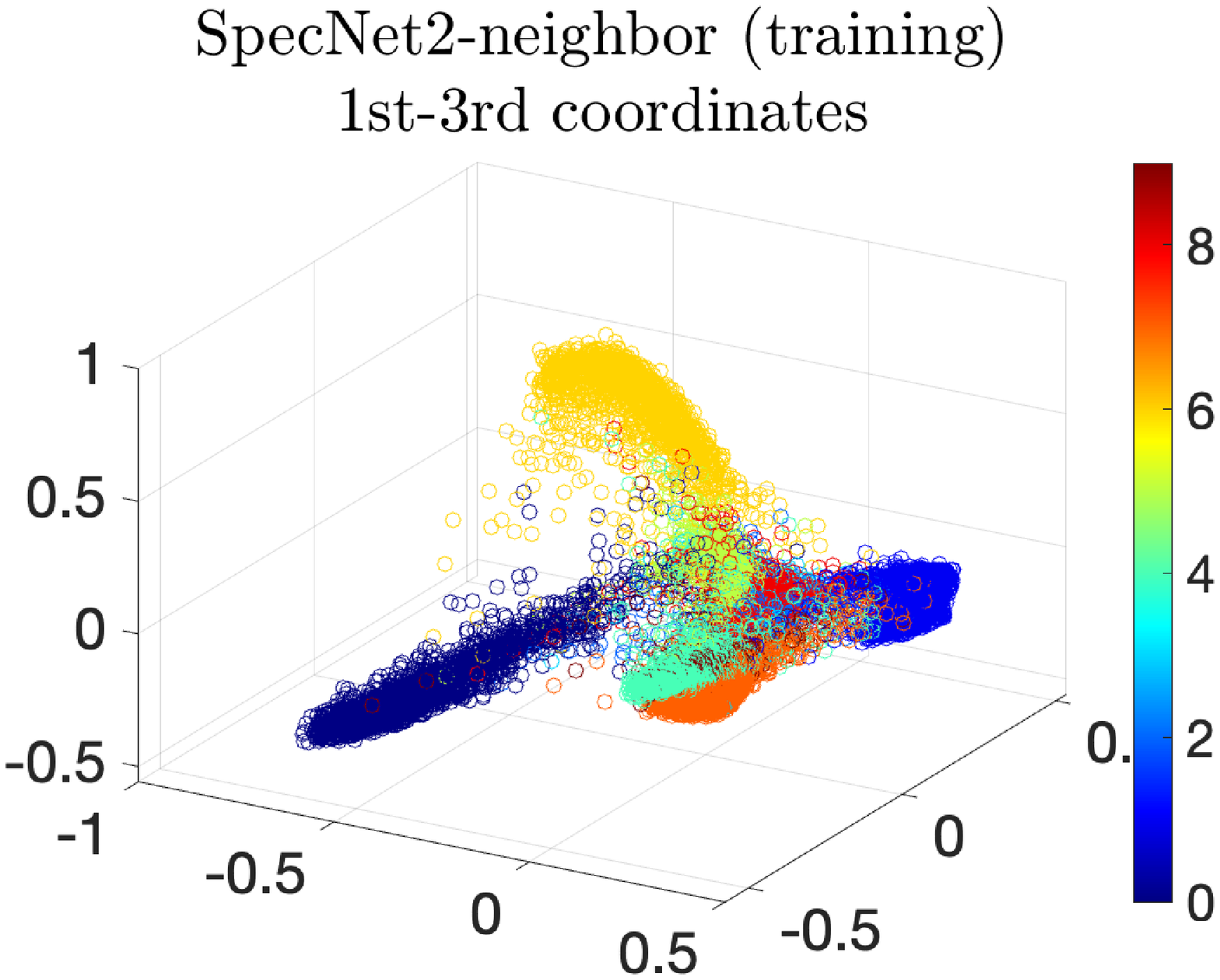}
    \includegraphics[width=0.245\textwidth]{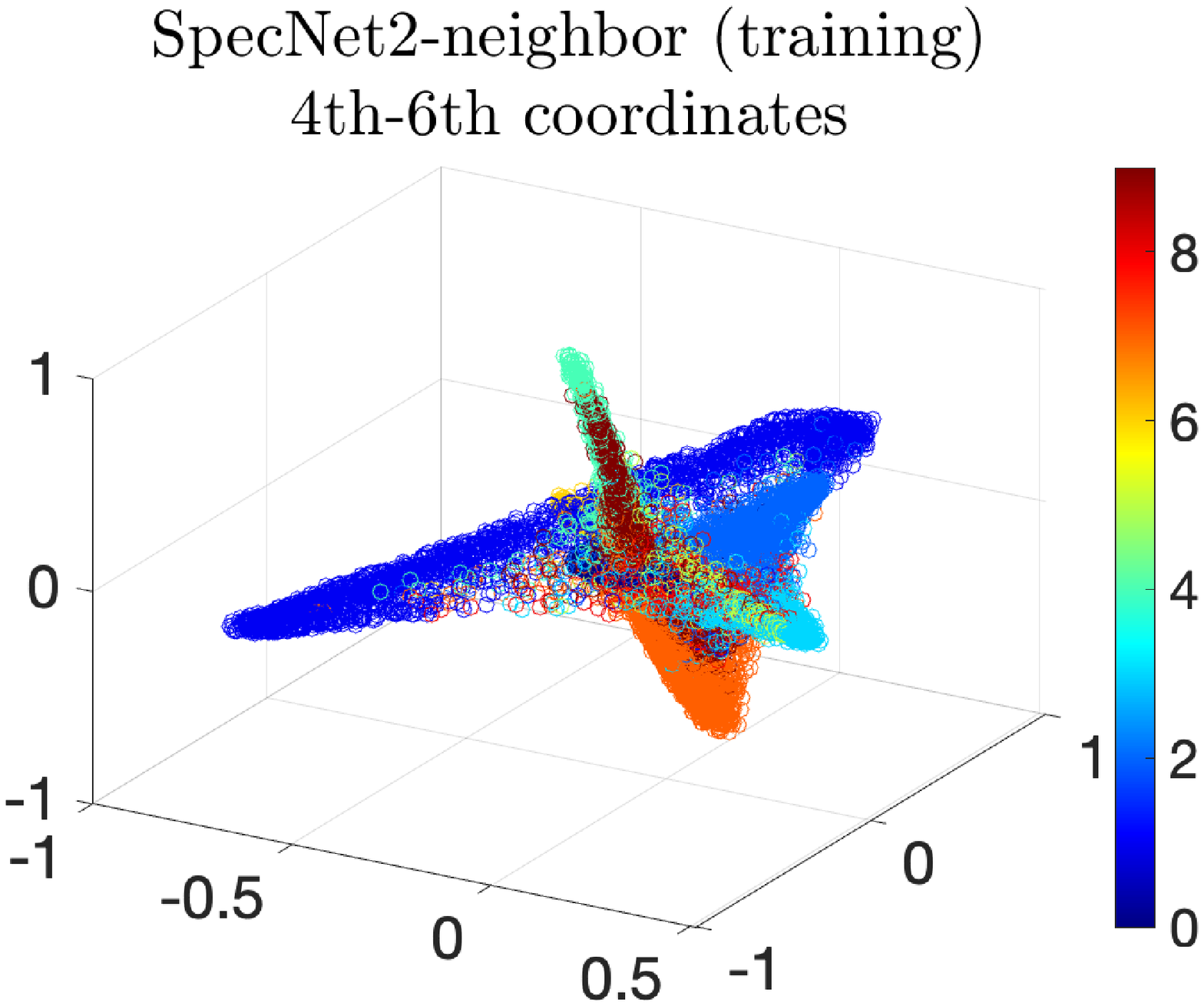}
    \includegraphics[width=0.245\textwidth]{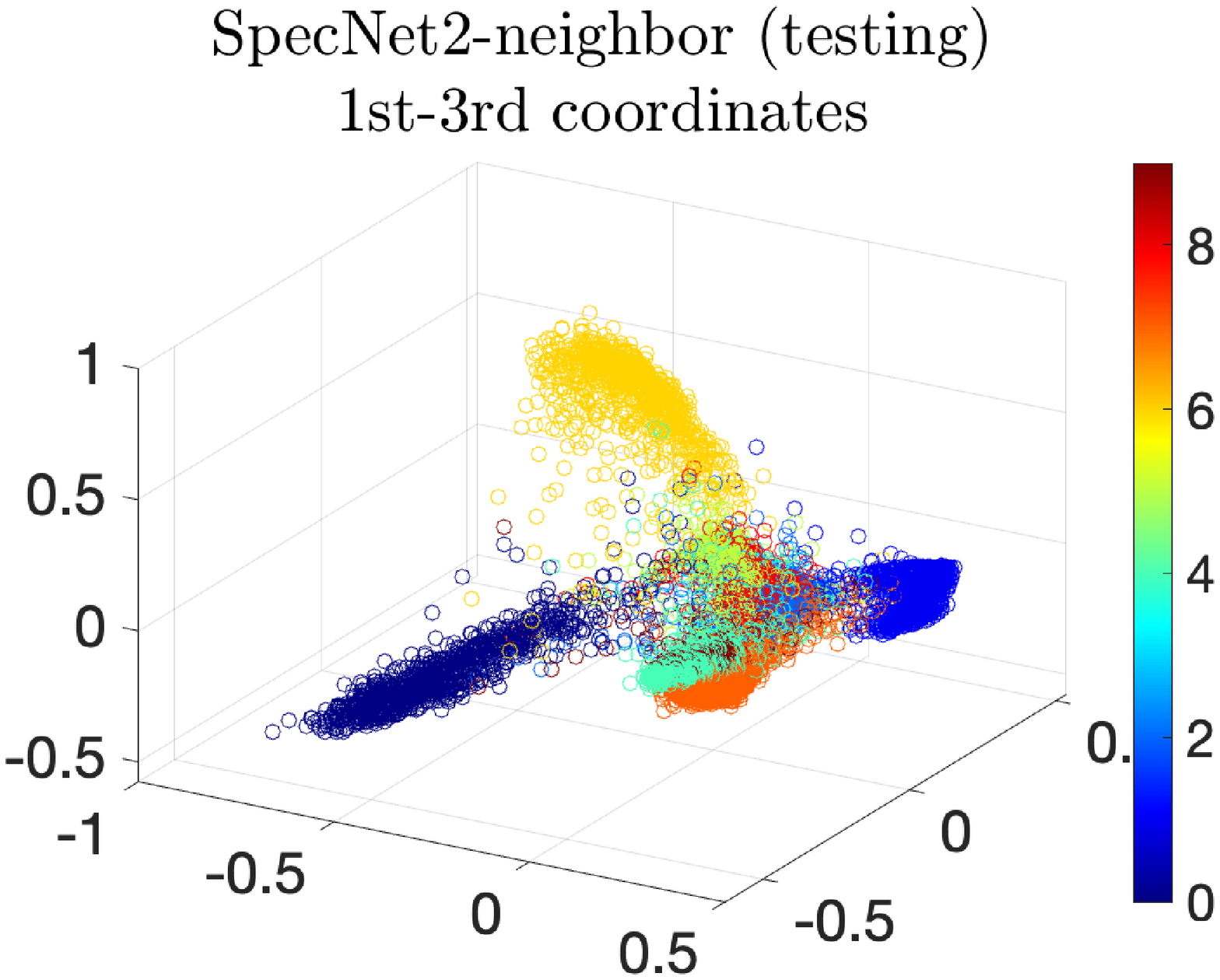}
    \includegraphics[width=0.245\textwidth]{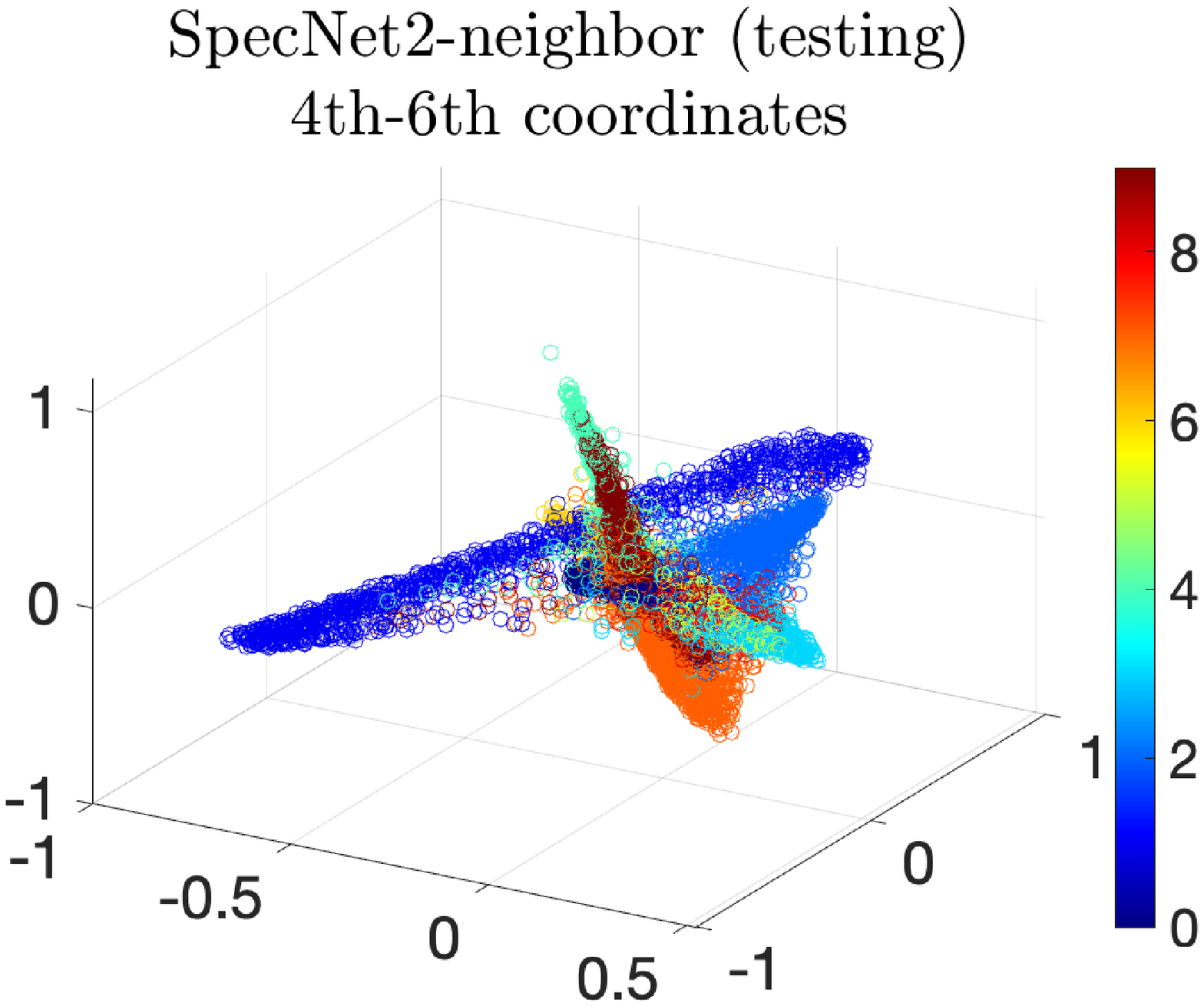}
    \caption{
    Embeddings of the MNIST datset.
    Top row: embeddings of the training set using the first six nontrivial eigenvectors of $D^{-1}W$; 
    Middle row: embeddings computed by SpecNet1-local with batch size 45 at the 50-th epoch; 
    Bottom row: embeddings computed by SpecNet2-neighbor with batch size 2 at the 50-th epoch. 
    }
    \label{fig:MNIST-embed}
\end{figure*}

\section{Discussion}
\label{sec:conclusion}

The current paper develops a new spectral network approach, which removes
the orthogonalization layer in the original
SpectralNet~\cite{shaham2018spectralnet}. We first propose an
unconstraint orthogonalization-free optimization problem to reveal the
leading $K$ eigenvectors of a given matrix pencil $(W,D)$. Iterative
algorithms with three different mini-batch gradient evaluation schemes,
namely local scheme, full scheme, and neighbor scheme, are proposed and
extended to the neural network training setting. The energy landscape of
the optimization problem is analyzed, and the global convergence to the
minimizer is guaranteed for all initial points up to a measure zero set.
Numerically, SpecNet2-neighbor achieves almost the same accuracy as
SpecNet1-full and SpecNet2-full while its computational cost is
significantly lower due to the neighborhood tracking trick.

There are several directions to extend the work. Theoretically, the
current analysis is in the sense of linear algebra. Further analysis is
needed to obtain optimization guarantee with the neural network
parametrization. Method-wise, the current approach assumes a graph
affinity matrix is provided, while in practice when only data samples are
provided one also needs to explore how to efficiently construct the graph
affinity, which can be used by the SpecNet2 neural network. Finally,
application to other real-world datasets could be explored, which would
potentially leads to more efficient implementations. 

\section{Proofs}

\subsection{Proof of Theorem~\ref{thm:minimizers}}
\label{app:thmproof}

We prove Theorem~\ref{thm:minimizers} in three steps. First we explicitly
give the expressions for all stationary points of \eqref{eq:uopt}. Then we
show that many of these stationary points are strict saddle points, \ie,
there exists decay direction at these points. Finally, we prove the rest
stationary points are of form as \eqref{eq:minimizerY} and are global
minimizers.

Recall the gradient of the objective function $f_2(Y)$ is of form
\eqref{eq:gradY}. We can also derive the Hessian of the objective function
and its bilinear form satisfies,
\begin{align*}
    S^\top \nabla^2 f_2(Y) S = & - 4 \trace{S^\top \frac{W}{n} S}
    + 4 \trace{S^\top \frac{D}{n} S Y^\top \frac{D}{n} Y} \\
    & + 4 \trace{S^\top \frac{D}{n} Y S^\top \frac{D}{n} Y}
    + 4 \trace{S^\top \frac{D}{n} Y Y^\top \frac{D}{n} S},
\end{align*}
where $S^\top \nabla^2 f_2(Y) S$ is a symbolic notation.

Stationary points of \eqref{eq:uopt} satisfy the first order condition, \ie,
\begin{equation} \label{eq:stationarycond}
    \nabla f_2(Y) = 0 \Leftrightarrow
    W Y = DY Y^\top \frac{D}{n} Y \Leftrightarrow
    \left( D^{-\frac{1}{2}} W D^{-\frac{1}{2}} \right)
    \left( D^{\frac{1}{2}} Y \right)
    = \left( D^{\frac{1}{2}} Y \right) Y^\top \frac{D}{n} Y.
\end{equation}
The right most equality in \eqref{eq:stationarycond} implies that
$D^{\frac{1}{2}} Y$ lies in an invariant subspace of $D^{-\frac{1}{2}} W
D^{-\frac{1}{2}}$, which is formed by eigenvectors of $D^{-\frac{1}{2}} W
D^{-\frac{1}{2}}$. Denote the invariant subspace by eigenvectors $V_r \in
\bbR^{n \times r}$, where $r \leq K$ is the dimension. The corresponding
eigenvalues are denoted by a diagonal matrix $\Lambda_r \in \bbR^{r \times
r}$. In connection to \eqref{eq:evp}, $\Lambda_r$ consists $r$ eigenvalues
of $(W,D)$ and $V_r$ consists of $r$ eigenvectors of $(W,D)$ transformed
by $D^{\frac{1}{2}}$. $D^{\frac{1}{2}} Y$ then can be written as
$D^{\frac{1}{2}} Y = V_r A$ for $A \in \bbR^{r \times K}$ being a full row
rank matrix. Substituting the expression back into
\eqref{eq:stationarycond}, we obtain,
\begin{equation}
    V_r \Lambda_r A = V_r A A^\top A \Leftrightarrow
    \Lambda_r A = A A^\top A \Leftrightarrow
    \Lambda_r = A A^\top,
\end{equation}
where the equivalences are due to the orthogonality of $V_r$ and full-rankness
of $A$. Therefore, $A$ admit the follow expression,
\begin{equation}
    A = \Lambda_r^{\frac{1}{2}} Q,
\end{equation}
where $Q \in \bbR^{r \times K}$ is a unitary matrix such that $QQ^\top = I$.
Putting the above analysis together, we conclude that the stationary points
of \eqref{eq:uopt} are of form,
\begin{equation} \label{eq:stationarypoints}
    Y = U \Lambda^{\frac{1}{2}} P Q,
\end{equation}
where $\Lambda$ and $U$ are the eigenvalue and the corresponding eigenvector
matrix of $(W,D)$, $P \in \bbR^{n \times r}$ is the first $r$ columns of an
arbitrary permutation matrix for $r\leq K$, and $Q \in \bbR^{r \times K}$
is an arbitrary orthogonal matrix.

Next we will show that many of these stationary points are saddle points.
Consider a stationary point $Y_0$ which does not include one of the leading
$K$ eigenvectors, \eg, $Y_0^\top D U_i = 0$ and $i$ is an index smaller than
$K$. If $r < K$, then we have a unitary vector $Q_\perp \in \bbR^{1 \times K}$
such that $Q_\perp Q^\top = 0$. Selecting a direction $S_0 = U_i Q_\perp$,
the Hessian at $Y_0$ evaluated at $S_0$ is,
\begin{equation}
    S_0^\top \nabla^2 f_2(Y_0) S_0 = - 4 \trace{S_0^\top \frac{W}{n} S_0}
    + 4 \trace{S_0^\top \frac{D}{n} S_0 Y_0^\top \frac{D}{n} Y_0}
    = - 4 \lambda_i < 0.
\end{equation}
If $r = K$, then there are $K$ eigenvectors selected by $P$ and one of them
must have index greater than $K$. Without loss of generality, we assume the
first column of $U \Lambda^{\frac{1}{2}} P$ is eigenvector with index $K+1$.
Then we choose a specific $S_0 = \begin{bmatrix} U_i & 0 & \cdots &
0\end{bmatrix} Q \in \bbR^{n \times K}$ and obtain,
\begin{equation}
    S_0^\top \nabla^2 f_2(Y_0) S_0 = - 4 \trace{S_0^\top \frac{W}{n} S_0}
    + 4 \trace{S_0^\top \frac{D}{n} S_0 Y_0^\top \frac{D}{n} Y_0}
    = -4 \lambda_i + 4 \lambda_{K+1} < 0,
\end{equation}
where the last inequality is due to the assumption on the nonzero
eigengap between the $K$-th and $(K+1)$-th eigenvalues. Therefore, we
conclude that when any of the leading $K$ eigenvectors is not selected in
\eqref{eq:stationarypoints}, the stationary point is a strict saddle point.
Besides these strict saddle points, the rest stationary points are of form,
\begin{equation} \label{eq:globalminimizers}
    Y = U_K \Lambda_K^{\frac{1}{2}} Q,
\end{equation}
where $\Lambda_K$ consists $K$ leading eigevalues and $U_K$ consists the
corresponding $K$ eigenvectors, $Q \in \bbR^{K \times K}$ is an arbitrary
orthogonal matrix.

\subsection{Proof of Corollary \ref{cor:minimizers}}\label{proof:cor5}

\begin{proof}
$f_2(Y)$ is a smooth function of $Y$ and note that the second term inside the trace of $f_2$ is $Y^\top DYY^\top DY$, which is a fourth-order term of $Y$, and $D$ is positive-definite, so $f_2(Y)\to+\infty$ as $\norm{Y}\to +\infty$ and $f_2(Y)$ is bounded from below. Hence
global minimizers of $f_2(Y)$ exist and are among local minimizers.
Substituting all local minimizers as shown in Theorem~\ref{thm:minimizers}
into $f_2(Y)$, we have,
\begin{equation}
    f_2(Y^\star) = - \sum_{i=1}^K \lambda_i^2,
\end{equation}
which means all local minimizers are of the same objective function value.
They are all global minimizers.
\end{proof}

\subsection{Proof of Theorem \ref{thm:convergence}}\label{app:proofthm3}

\begin{proof}[Proof of Lemma \ref{lem:bound}]

Let $Y^+ = g_p(Y)$ for any $p = 1, \dots, b$. By the definition of $W_0$,
it suffices to show $\norm{D_i^{\frac{1}{2}}Y_i^+}_2<R$ for $i \in S_p$ to
prove the lemma.

First, recall the iterative expression for $i$-th coordinate,
\begin{align*}
    Y_i^+ &= Y_i + 4\alpha(W_{i,:}Y-D_i Y_i(Y^\top D Y))\\
    &= Y_i + 4\alpha(W_{i,i}Y_i+W_{i,i^c}Y_{i^c}-D_i Y_i(Y_i^\top D_i Y_i)
    -D_i Y_i(Y_{i^c}^\top D_{i^c} Y_{i^c})).
\end{align*}
Left multiplying $D_i^{\frac{1}{2}}$ for rescaling purpose, we obtain,
\begin{align*}
    D_i^{\frac{1}{2}}Y_i^+
    = & D_i^{\frac{1}{2}}Y_i + 4 \alpha \left(
    W_{i,i}D_i^{\frac{1}{2}}Y_i + D_i^{\frac{1}{2}}
    W_{i,i^c} D_{i^c}^{-\frac{1}{2}}
    (D_{i^c}^{\frac{1}{2}} Y_{i^c}) \right. \\
    & \left. - D_i(D_i^{\frac{1}{2}}Y_i)(Y_i^\top D_i Y_i)
    - D_i (D_i^{\frac{1}{2}}Y_i)(Y_{i^c}^\top D_{i^c} Y_{i^c}) \right) \\
    =: & D_i^{\frac{1}{2}}Y_i + 4 \alpha T_i,
\end{align*}
where $T_i$ denotes all terms in the parentheses. Denote $X:=DY$,
$X_i:=D_i^{\frac{1}{2}}Y_i$, $X_{i^c}:=D_{i^c}^{\frac{1}{2}}Y_{i^c}$ and
$X_i^+:=D_i^{\frac{1}{2}}Y_i^+$. Then we have
\begin{align*}
    & \norm{X_i^+}_2^2\\
    = & \norm{X_i}_2^2 + 16 \alpha^2\norm{T_i}_2^2 + 8 \alpha \left(
    W_{i,i}\norm{X_i}_2^2 - D_i \norm{X_i}_2^4
    + D_i^{\frac{1}{2}} W_{i,i^c} D_{i^c}^{-\frac{1}{2}}X_{i^c} X_i^\top
    - D_i \norm{X_i X_{i^c}^\top}_2^2 \right)\\
    \leq & \norm{X_i}_2^2 + 16 \alpha^2\norm{T_i}_2^2 \\
    & + 8 \alpha \left(
    W_{i,i}\norm{X_i}_2^2 - D_i \norm{X_i}_2^4
    + D_i^{\frac{1}{2}} \norm{W_{i,i^c} D_{i^c}^{-\frac{1}{2}}}
    \norm{X_{i^c} X_i^\top} - D_i \norm{X_i X_{i^c}^\top}_2^2 \right).
\end{align*}

First, we bound $\norm{T_i}_2^2$ as,
\begin{align*}
    \norm{T_i}_2^2
    \leq & 3\left(\max_i W_{i,i}^2 R^2 + \norm{D_i X_i}_2^2
    \norm{X^\top X}_2^2 + \norm{D_i^{\frac{1}{2}}W_{i,i^c}
    D_{i^c}^{-\frac{1}{2}}}_2^2 \norm{X_{i^c}}_2^2\right)\\
    \leq & 3\left(\max_i W_{i,i}^2 R^2 + \max_i D_i^2\cdot n^2K^2R^6
    + \max_i D_i\norm{W_{i,i^c}D_{i^c}^{-\frac{1}{2}}}_2^2\cdot nR^2 \right)
    = M(R),
\end{align*}
where we adopts $\max_i \norm{X_i}_2<R$, $\norm{X^\top X}_2^2 \leq
\norm{X^\top X}_F^2 \leq K^2(nR^2)^2$, and $\norm{X_{i^c}}_2^2 < nR^2$.

Then, we estimate the coefficient of linear term in $\alpha$. By the
argument of second order polynomial, we have,
\begin{align*}
W_{i,i}\norm{X_i}_2^2 - D_i\norm{X_i}_2^4
\leq & \frac{W_{i,i}^2}{4D_i}, \\
D_i^{\frac{1}{2}} \norm{W_{i,i^c}D_{i^c}^{-\frac{1}{2}}}_2
\norm{X_i X_{i^c}^\top}_2 - D_i\norm{X_i X_{i^c}^\top}_2^2
\leq & \frac{\norm{W_{i,i^c}D_{i^c}^{-\frac{1}{2}}}_2^2}{4}.
\end{align*}

Next we discuss the inequality of $\norm{X^+_i}_2^2$ in two cases:
$\norm{X_i}_2 \leq \frac{R}{2}$ and $\frac{R}{2} < \norm{X_i}_2 < R$.

When $\norm{X_i}_2\leq\frac{R}{2}$, we have
\begin{align*}
    \norm{X_i^+}_2^2 \leq \frac{R^2}{4} + 16\alpha^2 M(R)
    + 8 \alpha M_2 < R^2,
\end{align*}
where the last inequality can be verified using $\alpha <
\frac{-2M_2+\sqrt{4M_2^2+3M(R)R^2}}{8M(R)}$.

When $\frac{R}{2} < \norm{X_i}_2 < R$, again by the argument of second
order polynomial, we have
\begin{equation*}
    W_{i,i} \norm{X_i}_2^2 - D_i\norm{X_i}_2^4
    + \frac{\norm{W_{i,i^c}D_{i^c}^{-\frac{1}{2}}}_2^2}{4}
    < -\frac{1}{8}
\end{equation*}
due to the fact that $R\geq 2\sqrt{M_1}$. Substituting into the inequality
of $\norm{X_i^+}_2^2$, we have
\begin{align*}
    \norm{X_i^+}_2^2 \leq \norm{X_i}_2^2 + 16 \alpha^2 M(R) - \alpha
    < \norm{X_i}_2^2 < R^2,
\end{align*}
where the second inequality can be verified using $\alpha <
\frac{1}{16M(R)}$.

\end{proof}

\begin{proof}[Proof of Lemma~\ref{lem:Lip}]

First, through a direct calculation, we have
\begin{align*}
    \frac{\partial f_2}{\partial Y_{i_1,k_1}}
    = -4\sum_{j=1}^n W_{i_1,j}Y_{j,k_1}
    + 4 D_{i_1} \sum_{k=1}^K Y_{i_1,k}
    \left(\sum_{\ell=1}^n Y_{\ell,k}D_\ell Y_{\ell,k_1} \right).
\end{align*}
And the second order partial derivative admits,
\begin{align*}
    \frac{\partial^2 f_2}{\partial Y_{i_1,k_1}\partial Y_{i_2,k_2}}
    = & -4 \delta_{k_1k_2} W_{i_1,i_2} + 4 D_{i_1} \delta_{i_1i_2}
    \left( \sum_{\ell=1}^n Y_{\ell,k_2}D_\ell Y_{\ell,k_1} \right)\\
    & + 4D_{i_1}Y_{i_1,k_2} D_{i_2}Y_{i_2,k_1}
    + 4D_{i_1}\sum_{k=1}^K Y_{i_1,k} Y_{i_2,k} D_{i_2} \delta_{k_1k_2},
\end{align*}
where $\delta_{ij}=1$ if $i=j$ and $\delta_{ij}=0$ otherwise. By
assumption that $\max_i \norm{D_i^{\frac{1}{2}}Y_i}_2 < R$, we have
$\max_{i,j} \abs{ D_i^{\frac{1}{2}}Y_{i,j}}<R$, and $\max_j
\norm{D^{\frac{1}{2}}Y_{:,j}}_2^2 < nR^2$. Therefore,
\begin{align*}
    \abs{\frac{\partial^2 f_2}{\partial Y_{i_1,k_1}\partial Y_{i_2,k_2}}}
    \leq & 4\max_{i,j}W_{i,j} + 4 D_{i_1}nR^2 + 4 \max_iD_{i} KR^2 \\
    \leq & 4\max_{i,j}W_{i,j} + 4 \max_iD_{i} (n+K)R^2.
\end{align*}

\end{proof}

\begin{proof}[Proof of Lemma \ref{lem:grad}]
Applying the updating expression, we have
\begin{align*}
    f_2(Y^{(\ell+1)}) \leq f_2(Y^{(\ell)})
    - \alpha \sum_{i\in S_\ell} \sum_{j=1}^K
    \left(\nabla_{i,j}f_2(Y^{(\ell)}) \right)^2
    + \alpha^2 L \sum_{i\in S_\ell}
    \sum_{j=1}^K \left(\nabla_{i,j}f_2(Y^{(\ell)}) \right)^2,
\end{align*}
where we abuse noation $S_\ell$ to denote the batch at $\ell$-th
iteration. Since $1-\frac{\alpha L}{2}>0$, we have
\[
    \sum_{i\in S_\ell} \sum_{j=1}^K \left(\nabla_{i,j}
    f_2(Y^{(\ell)})\right)^2
    \leq \frac{1}{\alpha(1-\alpha L)}
    \left(f_2(Y^{(\ell)})-f_2(Y^{(\ell+1)})\right).
\]
Summing over all $\ell$ from 0 to $T-1$, for $T = bP$ and any large
integer $P$, we have
\begin{align*}
    \sum_{p=0}^{P-1} \sum_{\ell=bp}^{b(p+1)-1}
    \left[ \sum_{i\in S_\ell}\sum_{j=1}^K
    \left(\nabla_{i,j}f_2(Y^{(\ell)}) \right)^2 \right]
    & \leq \frac{1}{\alpha(1-\alpha L)}
    \left(f_2(Y^{(0)})-f_2(Y^{(T)}) \right)\\
    &\leq \frac{1}{\alpha(1-\alpha L)}
    \left(f_2(Y^{(0)})- f_2^* \right),
\end{align*}
where $f_2^*$ denotes the minimum of $f_2$. Hence 
\[
    \lim_{ \ell\to\infty} \sum_{i\in S_\ell}
    \sum_{j=1}^K \left(\nabla_{i,j}f_2(Y^{(\ell)}) \right)^2 = 0.
\]
That is, for any $\epsilon>0$, there exists an integer $P_0>0$, such that
for any $p\geq P_0$, we have
\begin{equation*}
    \sum_{i\in S_\ell} \sum_{j=1}^K
    \left(\nabla_{i,j}f_2(Y^{(\ell)})\right)^2 \leq \epsilon^2,
    \quad \text{for } \ell = pb, \dots, (p+1)b-1.
\end{equation*}

For any two iterations, $\ell_1$ and $\ell_2$ such that $pb \leq \ell_1
\leq \ell_2 < (p+1)b$, and for any $i\in S_{\ell_1}$, $1\leq j \leq K$,
we have
\begin{align*}
    \abs{\nabla_{i,j}f_2(Y^{(\ell_1)})-\nabla_{i,j}f_2(Y^{(\ell_2)})}
    \leq & \sum_{\ell=\ell_1}^{\ell_2-1}
    \abs{\nabla_{i,j}f_2(Y^{(\ell)})-\nabla_{i,j}f_2(Y^{(\ell+1)})}\\
    \leq & L\sum_{\ell=\ell_1}^{\ell_2-1}\norm{Y^{(\ell)}-Y^{(\ell+1)}}_2\\
    \leq & L\sum_{\ell=\ell_1}^{\ell_2-1}
    \alpha \sqrt{\sum_{i\in S_\ell}\sum_{j=1}^K
    \left(\nabla_{i,j}f_2(Y^{(\ell)}) \right)^2}\\
    < & b\epsilon,
\end{align*}
where the last inequality is due to $\alpha L< 1$. 

Let $\ell_0$ be an iteration within $pb$ and $(p+1)b-1$, $p \geq P_0$.
Note that $\cup_{\ell=pb}^{(p+1)b-1} S_\ell = [n]$. Then we have
\begin{align*}
    \norm{\nabla f_2(Y^{(\ell_0)})}_2^2
    = & \sum_{\ell=pb}^{(p+1)b-1} \sum_{i\in S_\ell}\sum_{j=1}^K
    \left(\nabla_{i,j}f_2(Y^{(\ell_0)}) \right)^2\\
    = & \sum_{\ell=pb}^{(p+1)b-1}\sum_{i\in S_\ell}\sum_{j=1}^K
    \left(\nabla_{i,j}f_2(Y^{(\ell_0)}) - \nabla_{i,j}f_2(Y^{(\ell)})
    + \nabla_{i,j}f_2(Y^{(\ell)}) \right)^2\\
    \leq & \sum_{\ell=pb}^{(p+1)b-1} \sum_{i\in S_\ell}\sum_{j=1}^K
    \Big[ \left(\nabla_{i,j}f_2(Y^{(\ell)}) \right)^2
    + 2\epsilon\abs{\nabla_{i,j}f_2(Y^{(\ell_0)})
    - \nabla_{i,j}f_2(Y^{(\ell)})}\\
    & + \abs{\nabla_{i,j}f_2(Y^{(\ell_0)})
    - \nabla_{i,j}f_2(Y^{(\ell)})}^2\Big]\\
    < & (b+2nKb+nKb^2)\epsilon^2.
\end{align*}
Since $\epsilon$ can be arbitrarily small, we proved the lemma.

\end{proof}

\begin{proof}[Proof of Theorem~\ref{thm:convergence}]

Lemma~\ref{lem:bound} states that for any $Y\in W_0$ and $1\leq i\leq b$,
we have $g_i(Y)\in W_0$. Hence we have for any $Y\in W_0$, $g(Y)\in W_0$.
Lemma \ref{lem:Lip} states that $f_2$ has bounded Lipschitz coordinate
gradient in $W_0$, and the stepsize $\alpha$ satisfies $\alpha
<\frac{1}{KL\max_{i\in[b]}\abs{S_i}}$. Note that
$\max_{i\in[b]}\norm{\nabla^2
f_2(Y)_{S_i}}_2\leq\max_{i\in[b]}\norm{\nabla^2
f_2(Y)_{S_i}}_F\leq\sqrt{(K\cdot \max_{i\in[b]}\abs{S_i})^2 L^2}
=KL\max_{i\in[b]}\abs{S_i}$, Proposition 6 in \cite{lee2019first} shows
that under these conditions, we have $\det(Dg(x)) \neq 0$. Corollary 5 in
\cite{lee2019first} tells us that
$\mu(\{Y^{(0)}:\lim_{j\to\infty}g^j(Y^{(0)}) \in \chi^s\})=0$ for $\chi^s$
being the set of unstable stationary points and local maximizers.
Combining with the conclusion of Lemma~\ref{lem:grad}, we obtain the
conclusion of Theorem~\ref{thm:convergence}.

\end{proof}

\section*{Acknowledgement}
The work is supported by NSF DMS-2031849.
Ziyu Chen is supported by Simons Foundation Award and Simons Foundation - Math+X Investigators;
Xiuyuan Cheng is partially supported by NSF DMS-2007040, NIH and the Alfred P. Sloan Foundation.

\bibliographystyle{plain}
\bibliography{spec}

\begin{thebibliography}{10}

\bibitem{belabbas2009landmark}
Mohamed-Ali Belabbas and Patrick~J Wolfe.
\newblock On landmark selection and sampling in high-dimensional data analysis.
\newblock {\em Philosophical Transactions of the Royal Society A: Mathematical,
  Physical and Engineering Sciences}, 367(1906):4295--4312, 2009.

\bibitem{belkin2003laplacian}
Mikhail Belkin and Partha Niyogi.
\newblock Laplacian eigenmaps for dimensionality reduction and data
  representation.
\newblock {\em Neural computation}, 15(6):1373--1396, 2003.

\bibitem{belkin2007convergence}
Mikhail Belkin and Partha Niyogi.
\newblock Convergence of laplacian eigenmaps.
\newblock In {\em Advances in Neural Information Processing Systems}, pages
  129--136, 2007.

\bibitem{bermanis2013multiscale}
Amit Bermanis, Amir Averbuch, and Ronald~R Coifman.
\newblock Multiscale data sampling and function extension.
\newblock {\em Applied and Computational Harmonic Analysis}, 34(1):15--29,
  2013.

\bibitem{burago2014graph}
Dmitri Burago, Sergei Ivanov, and Yaroslav Kurylev.
\newblock A graph discretization of the laplace-beltrami operator.
\newblock {\em Journal of Spectral Theory}, 4(4):675--714, 2014.

\bibitem{calder2019improved}
Jeff Calder and Nicolas~Garcia Trillos.
\newblock Improved spectral convergence rates for graph laplacians on
  epsilon-graphs and k-nn graphs.
\newblock {\em arXiv preprint arXiv:1910.13476}, 2019.

\bibitem{calder2020lipschitz}
Jeff Calder, Nicolas~Garcia Trillos, and Marta Lewicka.
\newblock Lipschitz regularity of graph laplacians on random data clouds.
\newblock {\em arXiv preprint arXiv:2007.06679}, 2020.

\bibitem{cheng2020convergence}
Xiuyuan Cheng and Hau-Tieng Wu.
\newblock Convergence of graph laplacian with knn self-tuned kernels.
\newblock {\em Information and Inference: A Journal of the IMA}, 2021.

\bibitem{cheng2021eigen}
Xiuyuan Cheng and Nan Wu.
\newblock Eigen-convergence of gaussian kernelized graph laplacian by manifold
  heat interpolation.
\newblock {\em arXiv preprint arXiv:2101.09875}, 2021.

\bibitem{coifman2006diffusion}
Ronald~R Coifman and St{\'e}phane Lafon.
\newblock Diffusion maps.
\newblock {\em Applied and computational harmonic analysis}, 21(1):5--30, 2006.

\bibitem{dunson2021spectral}
David~B Dunson, Hau-Tieng Wu, and Nan Wu.
\newblock Spectral convergence of graph laplacian and heat kernel
  reconstruction in $l_\infty$ from random samples.
\newblock {\em Applied and Computational Harmonic Analysis}, 55:282--336, 2021.

\bibitem{Gao2020}
Weiguo Gao, Yingzhou Li, and Bichen Lu.
\newblock Triangularized orthogonalization-free method for solving extreme
  eigenvalue problems, may 2020.
\newblock http://arxiv.org/abs/2005.12161.

\bibitem{Gao2021}
Weiguo Gao, Yingzhou Li, and Bichen Lu.
\newblock Global convergence of triangularized orthogonalization-free method,
  2021.

\bibitem{Golub2013}
Gene~H. Golub and Charles~F. {Van Loan}.
\newblock {\em Matrix Computations}.
\newblock The Johns Hopkins University Press, 4th edition, 2013.

\bibitem{hadsell2006dimensionality}
Raia Hadsell, Sumit Chopra, and Yann LeCun.
\newblock Dimensionality reduction by learning an invariant mapping.
\newblock In {\em 2006 IEEE Computer Society Conference on Computer Vision and
  Pattern Recognition (CVPR'06)}, volume~2, pages 1735--1742. IEEE, 2006.

\bibitem{lee2019first}
Jason~D Lee, Ioannis Panageas, Georgios Piliouras, Max Simchowitz, Michael~I
  Jordan, and Benjamin Recht.
\newblock First-order methods almost always avoid strict saddle points.
\newblock {\em Mathematical programming}, 176(1):311--337, 2019.

\bibitem{Lei2016}
Qi~Lei, Kai Zhong, and Inderjit~S. Dhillon.
\newblock Coordinate-wise power method.
\newblock In D~D Lee, M~Sugiyama, U~V Luxburg, I~Guyon, and R~Garnett, editors,
  {\em Adv. Neural Inf. Process. Syst. 29}, pages 2064--2072. Curran
  Associates, Inc., 2016.

\bibitem{li2020variational}
Henry Li, Ofir Lindenbaum, Xiuyuan Cheng, and Alexander Cloninger.
\newblock Variational diffusion autoencoders with random walk sampling.
\newblock In {\em European Conference on Computer Vision}, pages 362--378.
  Springer, 2020.

\bibitem{Li2019c}
Yingzhou Li, Jianfeng Lu, and Zhe Wang.
\newblock Coordinatewise descent methods for leading eigenvalue problem.
\newblock {\em SIAM J. Sci. Comput.}, 41(4):A2681--A2716, jan 2019.

\bibitem{Liu2015c}
Xin Liu, Zaiwen Wen, and Yin Zhang.
\newblock An efficient {G}auss-{N}ewton algorithm for symmetric low-rank
  product matrix approximations.
\newblock {\em SIAM J. Optim.}, 25(3):1571--1608, 2015.

\bibitem{Mauri1993}
Francesco Mauri, Giulia Galli, and Roberto Car.
\newblock Orbital formulation for electronic-structure calculations with linear
  system-size scaling.
\newblock {\em Phys. Rev. B}, 47(15):9973, apr 1993.

\bibitem{mishne2019diffusion}
Gal Mishne, Uri Shaham, Alexander Cloninger, and Israel Cohen.
\newblock Diffusion nets.
\newblock {\em Applied and Computational Harmonic Analysis}, 47(2):259--285,
  2019.

\bibitem{nystrom1930praktische}
Evert~J Nystr{\"o}m.
\newblock {\"U}ber die praktische aufl{\"o}sung von integralgleichungen mit
  anwendungen auf randwertaufgaben.
\newblock {\em Acta Mathematica}, 54:185--204, 1930.

\bibitem{Ordejon1993}
Pablo Ordejon, David~A. Drabold, Matthew~P. Grumbach, Richard~M. Martin, Pablo
  Ordej{\'{o}}n, David~A. Drabold, Matthew~P. Grumbach, and Richard~M. Martin.
\newblock Unconstrained minimization approach for electronic computations that
  scales linearly with system size.
\newblock {\em Phys. Rev. B}, 48(19):14646, nov 1993.

\bibitem{shaham2018spectralnet}
Uri Shaham, Kelly Stanton, Henry Li, Ronen Basri, Boaz Nadler, and Yuval
  Kluger.
\newblock Spectralnet: Spectral clustering using deep neural networks.
\newblock In {\em International Conference on Learning Representations}, 2018.

\bibitem{shen2020scalability}
Chao Shen and Hau-Tieng Wu.
\newblock Scalability and robustness of spectral embedding: landmark diffusion
  is all you need.
\newblock {\em arXiv preprint arXiv:2001.00801}, 2020.

\bibitem{singer2016spectral}
Amit Singer and Hau-Tieng Wu.
\newblock Spectral convergence of the connection laplacian from random samples.
\newblock {\em Information and Inference: A Journal of the IMA}, 6(1):58--123,
  2016.

\bibitem{trillos2020error}
Nicol{\'a}s~Garc{\'\i}a Trillos, Moritz Gerlach, Matthias Hein, and Dejan
  Slep{\v{c}}ev.
\newblock Error estimates for spectral convergence of the graph laplacian on
  random geometric graphs toward the laplace--beltrami operator.
\newblock {\em Foundations of Computational Mathematics}, 20(4):827--887, 2020.

\bibitem{von2008consistency}
Ulrike {Von Luxburg}, Mikhail Belkin, and Olivier Bousquet.
\newblock Consistency of spectral clustering.
\newblock {\em The Annals of Statistics}, pages 555--586, 2008.

\bibitem{Wang2019}
Zhe Wang, Yingzhou Li, and Jianfeng Lu.
\newblock Coordinate descent full configuration interaction.
\newblock {\em J. Chem. Theory Comput.}, 15(6):3558--3569, jun 2019.

\bibitem{williams2000using}
Christopher Williams and Matthias Seeger.
\newblock Using the nystr{\"o}m method to speed up kernel machines.
\newblock {\em Advances in neural information processing systems}, 13, 2000.

\end{thebibliography}

\appendix

\section{Gradient evaluation schemes for SpecNet1}

\subsection{Gradient evaluation schemes for $f_1$}
\label{subsec:f1scheme}
The gradient descent of the formulation in \eqref{eq:specnet1}
can be written as $Y = Y -2\alpha(I-D^{-\frac{1}{2}}WD^{-\frac{1}{2}})Y$, where $\alpha$ is the stepsize. We multiply both sides by $D^{-\frac{1}{2}}$ on the left and we have $D^{-\frac{1}{2}}Y = D^{-\frac{1}{2}}Y - 2\alpha (I-D^{-1}W)D^{-\frac{1}{2}}Y$. So instead of updating $Y$, we update $\tilde{Y}:=D^{-\frac{1}{2}}Y$ at each iteration, i.e., $\tilde{Y} = \tilde{Y} - 2\alpha (I-D^{-1}W)\tilde{Y}$. And the constraint will be $\tilde{Y}^\top D\tilde{Y} = n^2 I$. To keep the consistency of the notation, we will abuse the notion of gradient and still call that $2(I-D^{-1}W)\tilde{Y}$ as the gradient of $f_1$ for the rest of the paper, while keeping in mind that we are updating $D^{-\frac{1}{2}}Y$ in \eqref{eq:specnet1}. Then the gradient evaluation schemes for $f_1$ with orthogonalization constraint works as follows
\begin{itemize}
    \item \textbf{\emph{Local evaluation scheme:}} One can evaluate the gradient on each
mini-batch as
\begin{equation} \label{eq:gradYlocal1}
    \nabla_\calB \tilde{f}_{1}(Y_{\calB}) = 2(I-\tilde{D}_{\calB}^{-1}W_{\calB,\calB})Y_{\calB}.
\end{equation}
The iterative algorithm then conducts the update as,
\begin{equation}\label{eq:iterYlocal1}
    Y_{\calB} = Y_{\calB} - \alpha \nabla_\calB \tilde{f}_{1}(Y_{\calB}),
\end{equation}
where $\alpha>0$ is the stepsize, which is followed by an orthogonalization step $Y = bYR^{-1}$, where $\tilde{D}_{\calB}^{\frac{1}{2}}Y_{\calB} = QR$ is the QR decomposition of $\tilde{D}_{\calB}^{\frac{1}{2}}Y_{\calB}$.

\item \textbf{\emph{Full evaluation scheme:}} We evaluate the gradient on batch
$\calB$ as
\begin{equation} \label{eq:gradYfull1}
    \nabla_\calB f_1(Y) = 2(I-D_{\calB}^{-1}W_{\calB,X})Y,
\end{equation}
and the update is then conducted as
\begin{equation} \label{eq:iterYfull1}
    Y_{\calB} = Y_{\calB} - \alpha \nabla_{\calB} f_{1}(Y),
\end{equation}
where $\alpha>0$ is the stepsize. It follows by an orthogonalization step $Y = nYR^{-1}$, where $D^{\frac{1}{2}}Y = QR$ is the QR decomposition of $D^{\frac{1}{2}}Y$. The full scheme of $f_1$ is equivalent to the power method with mini-batch and dynamic shift.

\item \textbf{\emph{Neighbor evaluation scheme:}} The gradient of batch $\calB$ is
evaluated as
\begin{equation} \label{eq:gradYneighbor1}
    \nabla_{\calB} \bar{f}_{1}(Y_\calN) = 2( Y_\calB - D_\calB^{-1} W_{\calB,\calN}Y_\calN).
\end{equation}
The iterative algorithm then conduct the update as,
\begin{equation} \label{eq:iterYneighbor1}
    Y_{\calB} = Y_{\calB} - \alpha \nabla_{\calB} \bar{f}_{1}(Y_\calN)
\end{equation}
for $\alpha$ being the stepsize.

It follows by an orthogonalization step such that $Y = nY(L^{-1})^\top$, where $Y^\top DY = LL^\top$ is the Cholesky decomposition of $Y^\top DY$, and as in \eqref{eq:gradYneighbor}, we only update $Y^\top DY$ on $\calB$ at each iteration.
\end{itemize}

\subsection{Network Training for SpecNet1}\label{subsec:net1-training}
Different from SpecNet2, we have one additional orthogonalization layer, denoted by $R\in \mathbb{R}^{K\times K}$, appended to $G_\theta$ for SpecNet1. Therefore, the mapping given by SpecNet1 is $x\mapsto G_\theta(x)\cdot R$ for any $x\in\mathbb{R}^m$. We also introduce the training of SpecNet1 that incorporates those gradient evaluation schemes in section \ref{subsec:f1scheme}.

\textbf{\emph{Local evaluation scheme:}} At each batch step, we
compute $Y_\calB =
G_{\theta}(\calB)$. The orthogonalization layer is computed as in the QR factorization $\tilde{D}_{\calB}^{\frac{1}{2}}Y_{\calB} = QR$, and the output after that is then $\tilde{Y}_\calB = b Y_\calB R^{-1}$. So we can obtain $\nabla_\calB
\tilde{f}_{1}(\tilde{Y}_\calB)$ by plugging $\tilde{Y}_\calB$ into
\eqref{eq:gradYlocal1}. Then we minimize
$\trace{\tilde{Y}_\calB(\theta)^\top \nabla_\calB \tilde{f}_{1}(\tilde{Y}_\calB))}$ and
update $\theta$ using the gradient of $\trace{\tilde{Y}_\calB(\theta)^\top
\nabla_\calB \tilde{f}_{2}(Y_\calB)}$ with respect to $\theta$ through
the chain rule, where inside the trace we write the first term $\tilde{Y}_\calB$
as $\tilde{Y}_\calB(\theta)$ to emphasize it is a function of $\theta$; and the
second term $\nabla_\calB \tilde{f}_{2}(\tilde{Y}_\calB)$ is detached and viewed
as constant. We shall mention that SpecNet1 with local
evaluation scheme is the method in the original SpecNet1
paper~\cite{shaham2018spectralnet}, except that here we also update weights of the orthogonalization layer using the gradient by back-propagation, which turns out to improve the performance of the original SpecNet1 significantly. 

\textbf{\emph{Full evaluation scheme:}} At each batch step, we
compute $Y =
G_{\theta}(X)$. The orthogonalization layer is computed as in the QR factorization $\tilde{D}^{\frac{1}{2}}Y = QR$, and the output after that is then $\tilde{Y} = n Y R^{-1}$. So we can obtain $\nabla_\calB
f_{1}(\tilde{Y})$ by plugging $\tilde{Y}$ into
\eqref{eq:gradYfull1}. Then we want to minimize
$\trace{\tilde{Y}_\calB(\theta)^\top \nabla_\calB f_{1}(\tilde{Y}))}$ and update
$\theta$ using the gradient of $\trace{\tilde{Y}_\calB(\theta)^\top \nabla_\calB f_{1}(\tilde{Y})}$ through the chain rule. And similarly, inside the trace
we only view $\tilde{Y}_\calB(\theta)$ as a function of $\theta$ but $\nabla_\calB f_{1}(\tilde{Y})$ as constant when computing the gradient.

\textbf{\emph{Neighbor evaluation scheme:}} We keep a record of two
matrices $(YDY)_{\star}$ and $Y_{0}$ throughout the training, where they
are initialized at the first iteration: $(YDY)_{\star} = Y^\top DY$ and
$Y_{0} = Y$, and detach both of them. At each batch step, we compute
$Y_\calN = G_{\theta}(\calN)$. Then we update $(YDY)_{\star} =
(YDY)_{\star} - Y_{0}(\calN)^\top D_{\calN}Y_{0}(\calN) + Y_\calN^\top
D_{\calN}Y_\calN$ followed by an update of $Y_{0}$ on $\calN$ as
$Y_{0}(\calN) = Y_\calN$. Both matrices are again detached. The orthogonalization layer is computed as in the Cholesky factorization $(YDY)_{\star} = LL^\top$, and the output after that is then $\tilde{Y}_\calN = n Y_\calN (L^{-1})^\top$. So we can obtain $\nabla_{\calB} \bar{f}_{1}(\tilde{Y}_\calN)$ by plugging $\tilde{Y}_\calN$, which includes $\tilde{Y}_\calB$,  into
\eqref{eq:gradYneighbor1}.
Then we minimize $\trace{\tilde{Y}_\calB(\theta)^\top \nabla_{\calB} \bar{f}_{1}(\tilde{Y}_\calN)}$ and update $\theta$ by computing the gradient of
$\trace{\tilde{Y}_\calB(\theta)^\top \nabla_{\calB} \bar{f}_{1}(\tilde{Y}_\calN)}$ by the
chain rule. Similarly, inside the trace we only view $\tilde{Y}_\calB(\theta)$ as
a function of $\theta$ but $\nabla_{\calB} \bar{f}_{1}(\tilde{Y}_\calN)$ as constant
when computing the gradient.

\section{Details of the numerical examples in Section \ref{sec:numres}}\label{detail:numres}

\begin{figure}[t]
    \centering
    \includegraphics[width=0.46 \textwidth]{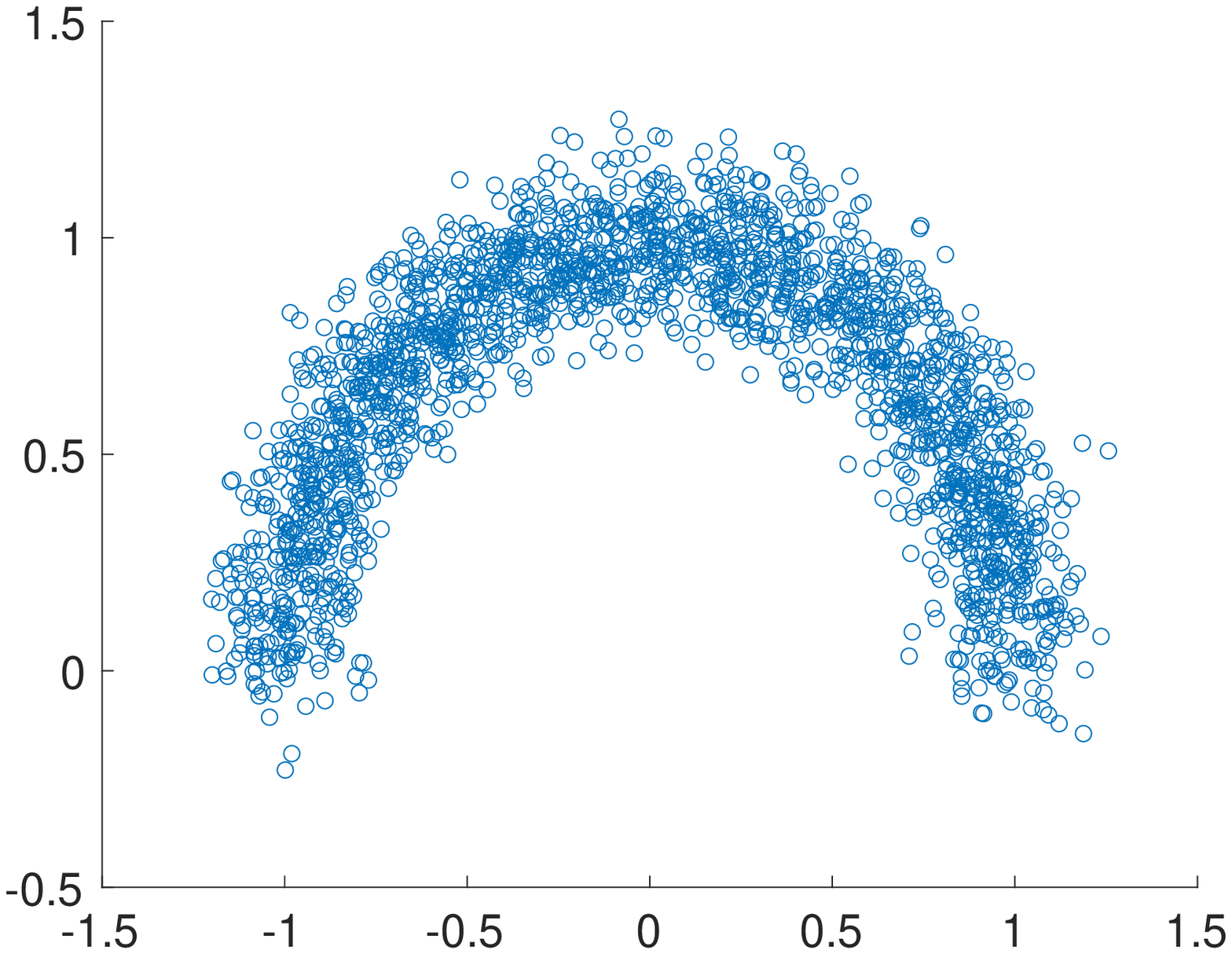}
    \includegraphics[width=0.46 \textwidth]{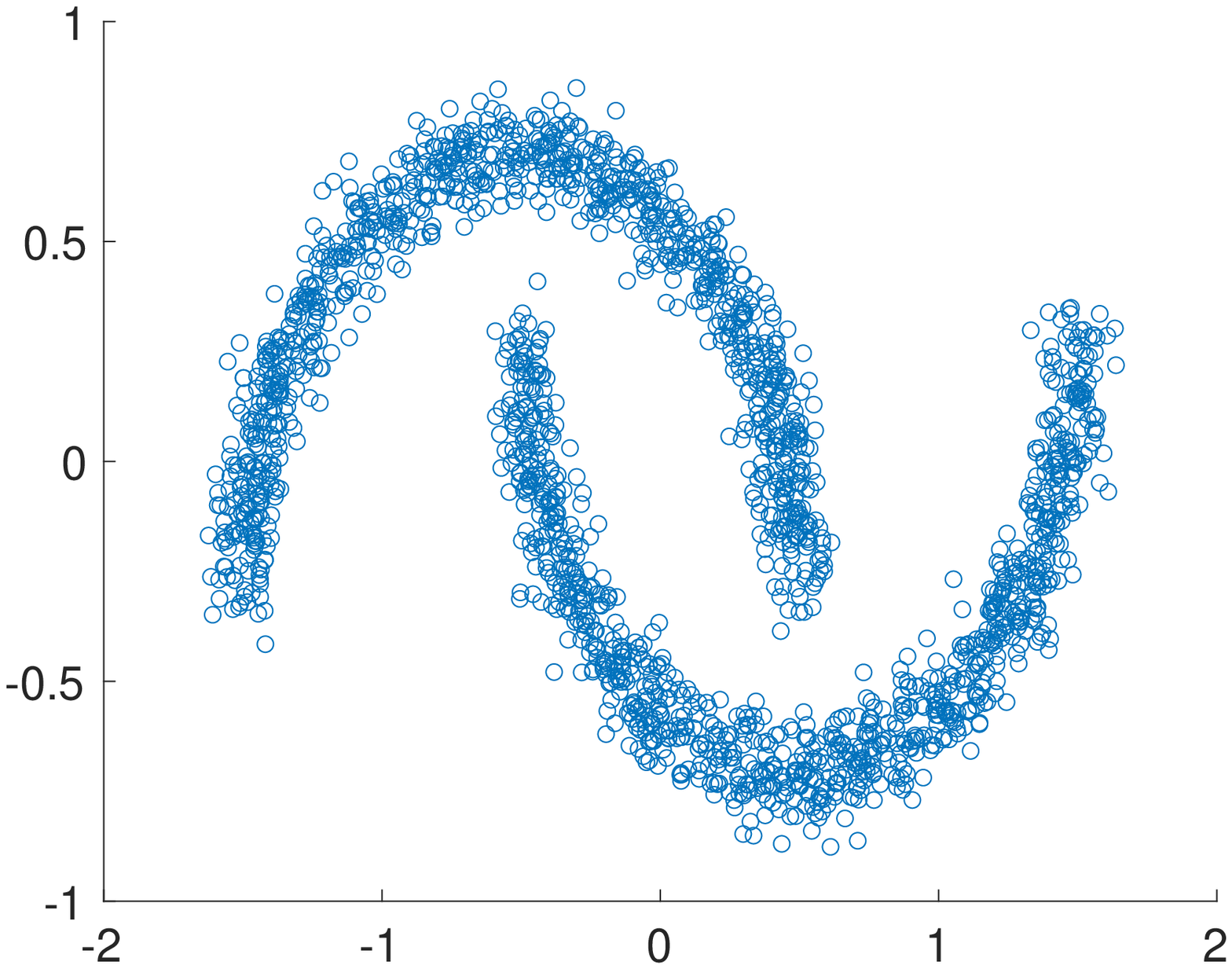}
    \caption{
    Visualizations of the one moon and two moons training dataset. 
    For each example, the testing set are i.i.d. sampled from the same distribution as the training set.
    }
    \label{fig:moon}
\end{figure}

\subsection{One moon data}\label{detail:onemoon}
\textbf{Data generation:}
The training set consists of $n = 2000$ points in
$\mathbb{R}^2$, and is generated by $x_i = (\cos
\eta_i, \sin \eta_i) + \xi_i$, $i = 1,\dots,2000$, where $\eta_i$ are i.i.d. uniformly sampled on $[0,\pi]$ and
$\xi_i$ are i.i.d. Gaussian random variables of dimension two drawn from
$\calN(0, 0.01 I_2)$. The testing set consists of 2000 points and is generated in the same way as the training set with a different realization. The sparse affinity matrix associated with the training set is generated via Gaussian
kernel with bandwidth $\sigma=0.1$, and truncated at threshold 0.6.

\noindent\textbf{Network training:}
We use a fully-connected feedforward neural network with a single 128-unit hidden layer:

SpecNet1: $2 \xrightarrow[]{\text{fc}} 128 - \text{ReLU} \xrightarrow[]{\text{linear}} 3 \xrightarrow[]{\text{orthogonal}} 3$;

SpecNet2: $2 \xrightarrow[]{\text{fc}} 128 - \text{ReLU} \xrightarrow[]{\text{linear}} 2$,

where ``fc'' stands for fully-connected layers. The batch size is 4, and we use Adam as the optimizer with learning rate $10^{-3}$ for SpecNet2 and $10^{-4}$ for SpecNet1.

\noindent\textbf{Error evaluation:} We
evaluate the network approximation of the first two nontrivial
eigenfunctions by computing the relative errors of the output functions of
the trained network with the underlying true eigenfunctions. The true
eigenfunctions are constructed via a fine grid discretization of the
continuous operator. We introduce
how the relative error is calculated. Suppose $\psi\in\mathbb{R}^n$ is the
limiting eigenfunction evaluated at $\{x_i\}$, and
$\tilde{\psi}\in\mathbb{R}^n$ is the network output function, which
approximates $\psi$, evaluated at $\{x_i\}$. The relative error
$\tau(\tilde{\psi},\psi)$ of $\tilde{\psi}$ with respect to $\psi$ is
defined as
\begin{equation}\label{eq:relativeerror}
    \tau(\tilde{\psi},\psi):= \frac{\norm{\psi-\beta\tilde{\psi}}_2}
    {\norm{\psi}_2},
\end{equation}
where $\beta = \frac{\psi^\top \tilde{\psi}}{\norm{\tilde{\psi}}_2^2}$ is
the number that minimizes $\norm{\psi-\beta\tilde{\psi}}_2$ serving the
role of aligning two eigenfunctions. To evaluate the relative error on the training set, $\psi$ will be the limiting eigenfunction evaluated at training samples and $\tilde{\psi}$ is the corresponding network output function evaluated at training samples; the relative error on the testing set can be defined similarly on testing samples.

To further
compare the computational efficiency of gradient evaluation schemes, we
plot the relative errors against the leading computational cost in
Figure~\ref{fig:onemoon-net-cost}, where the leading computational costs
are estimated as: $\frac{n^2}{|\mathcal{B}|} \cdot \text{epoch}$ for the
full gradient evaluation scheme; $\frac{n|\mathcal{N}|}{|\mathcal{B}|}
\cdot \text{epoch}$ for the neighbor gradient evaluation scheme, where the
averaged number of neighbors of a batch of size 4 is about 620. 
We also show the embedding results provided by different methods in Figure \ref{fig:onemoon-embed}.

Moreover, we seek to solve the generalized eigenvalue problem $(W,D)$, corresponding to the random walk Laplacian $D^{-1}W$ in both the SpecNet1 and SpecNet2 implementation here. We can also approximate the eigenvalue problem of $D^{-\frac{1}{2}}WD^{-\frac{1}{2}}$, the symmetrically normalized Laplacian, in our implementation: that is, setting $W = D^{-\frac{1}{2}}WD^{-\frac{1}{2}}$ and $D= I$ in \eqref{eq:uopt} for SpecNet2 and multiply the output by $D^{-\frac{1}{2}}$ on the left to get back to the generalized eigenvalue problem $(W,D)$, which approximates eigenfunctions of a continuous limiting operator. We show the result in Figure~\ref{fig:onemoon-compare} for the full scheme for the relative errors of approximations of first two nontrivial eigenfunctions on the training set. We see that the performance is similar if we switch from using $(W,D)$ to $D^{-\frac{1}{2}}WD^{-\frac{1}{2}}$ for the training objective (we may need to change the learning rate after switch, but we did not tune it here).

\begin{figure*}[htbp]
    \includegraphics[width=0.245\textwidth]{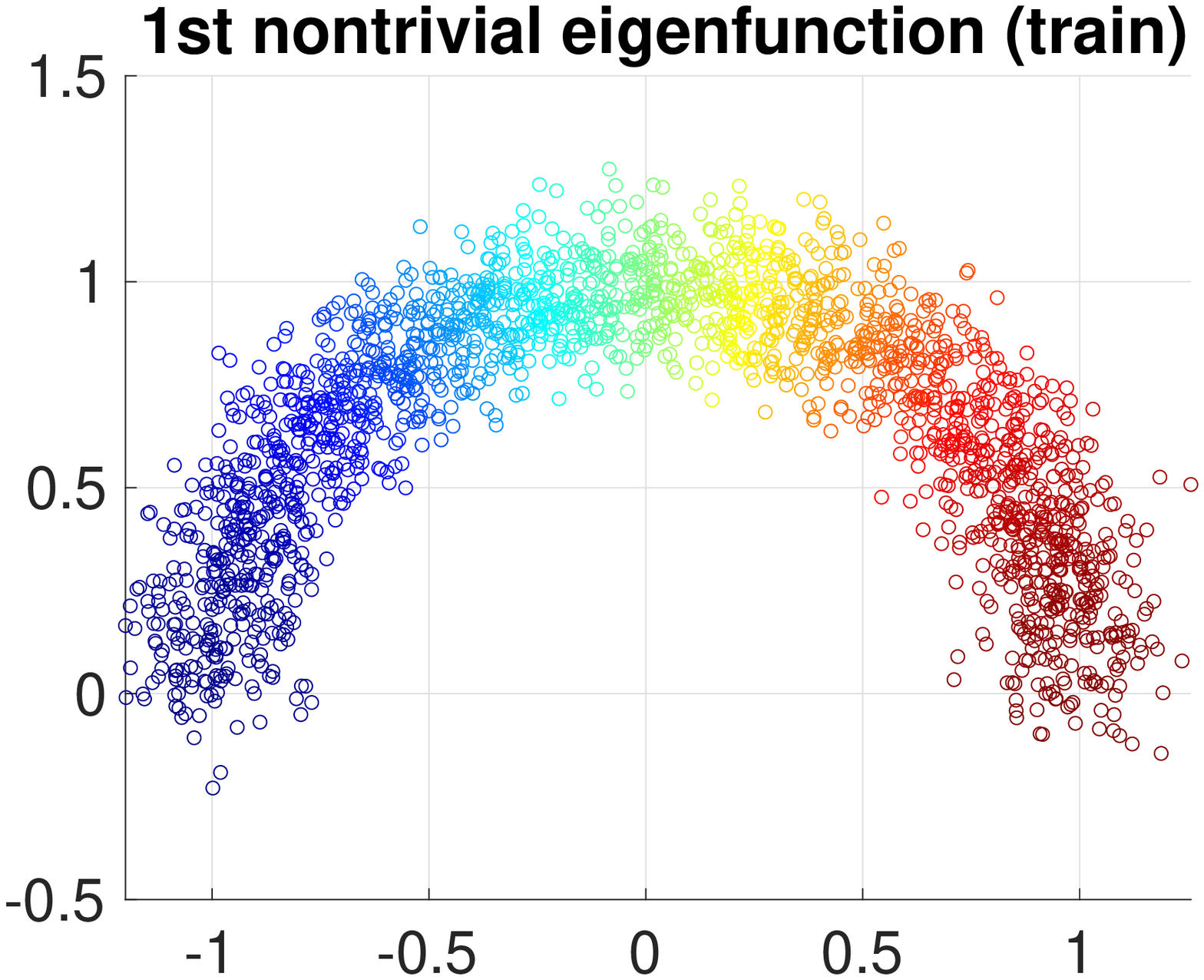}
    \includegraphics[width=0.245\textwidth]{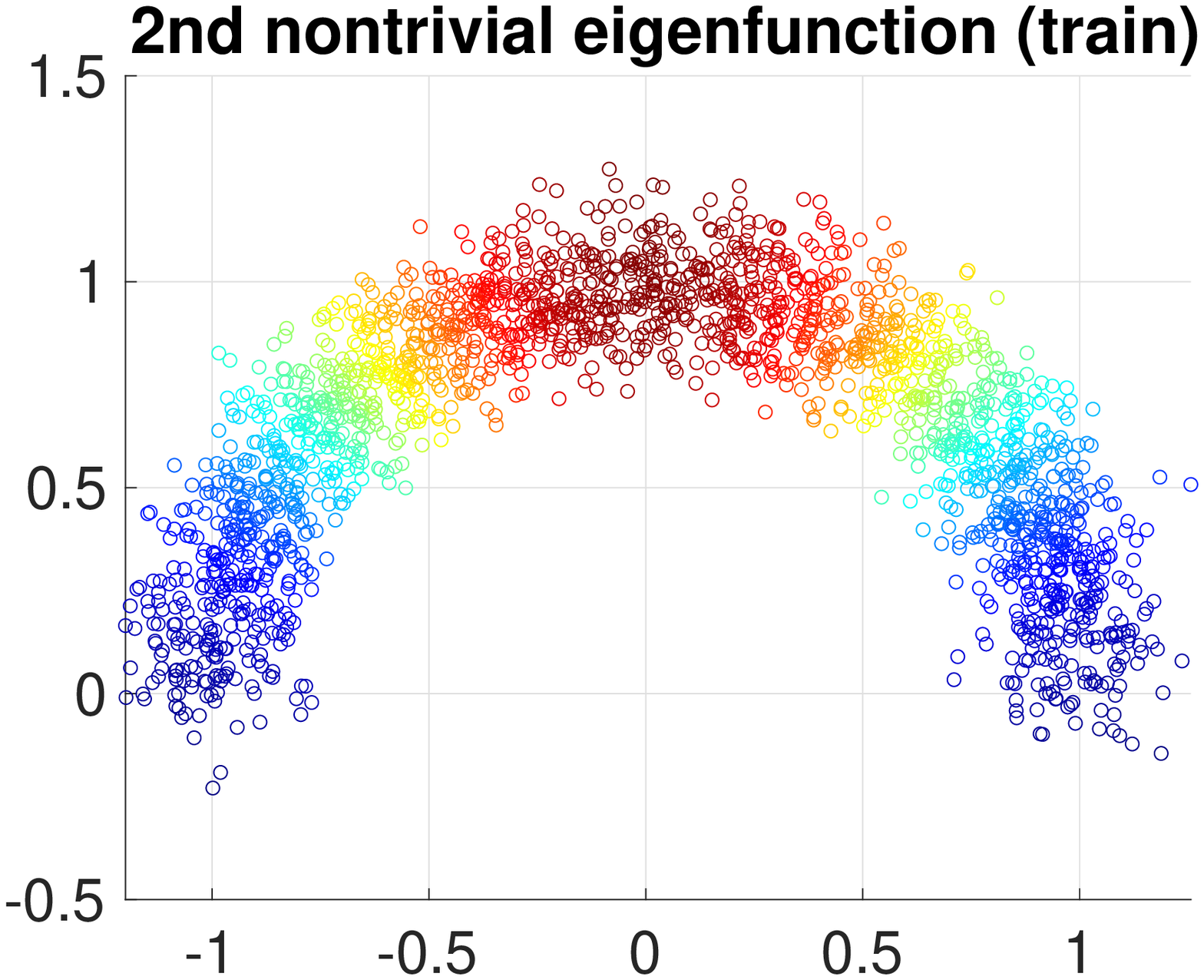}
    \includegraphics[width=0.245\textwidth]{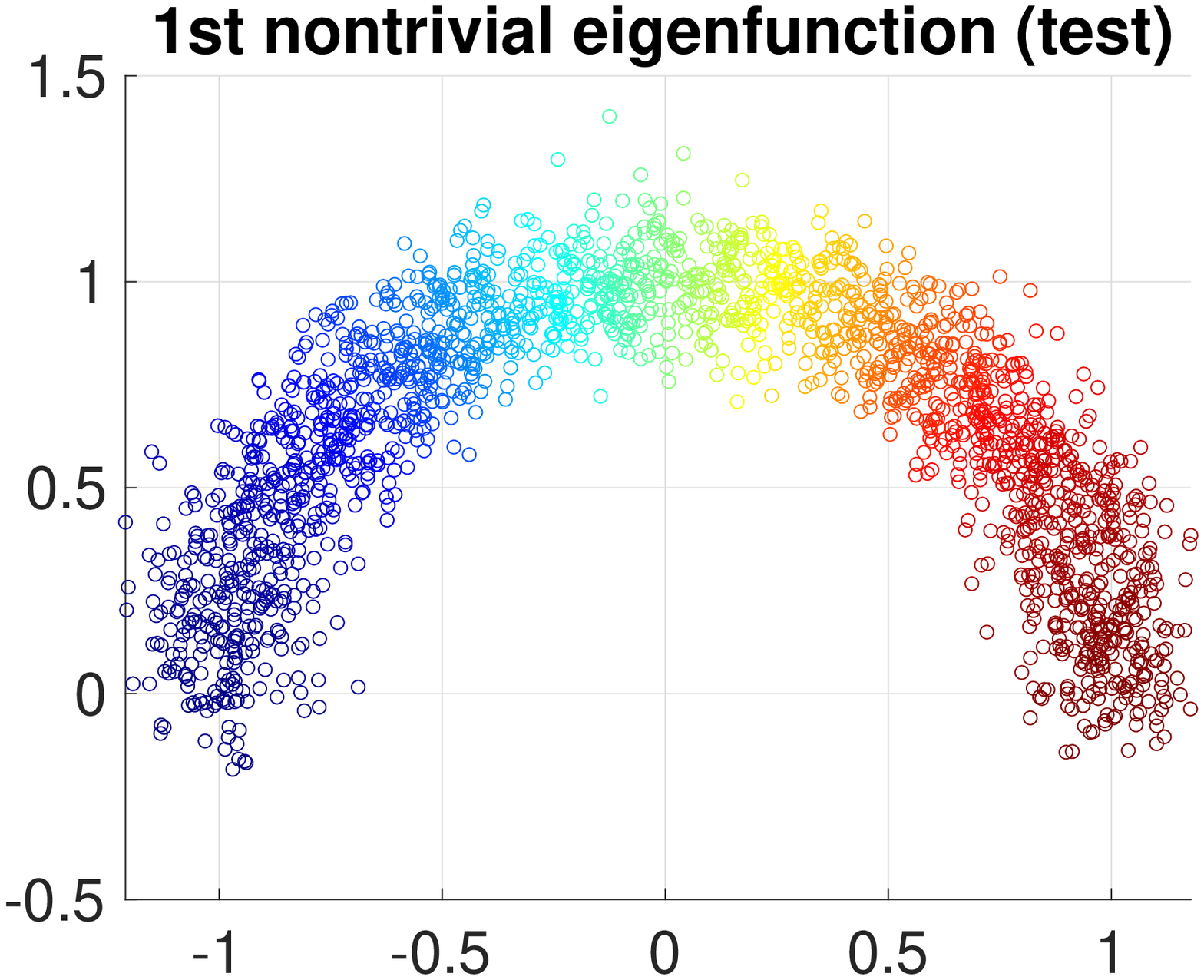}
    \includegraphics[width=0.245\textwidth]{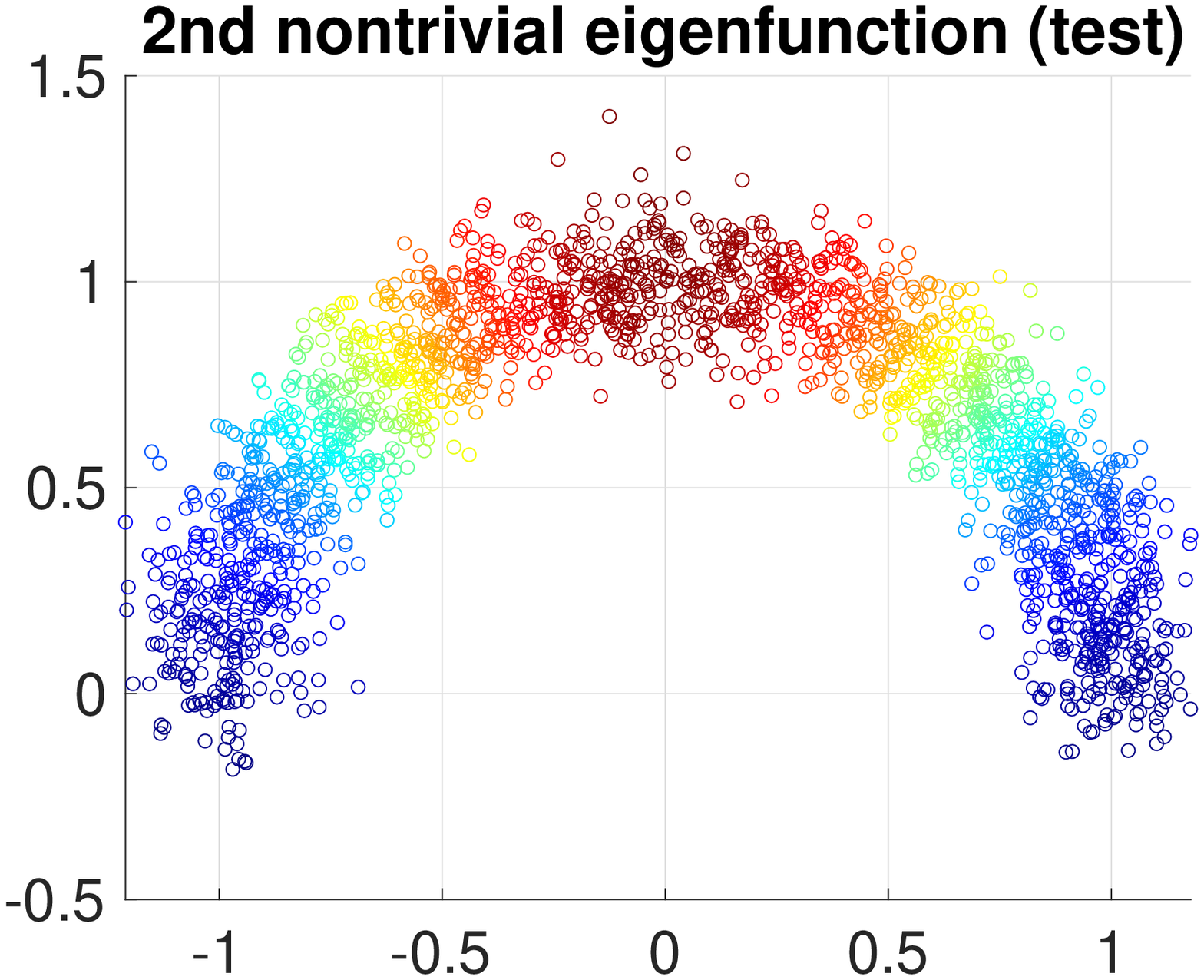}\\
    \includegraphics[width=0.245\textwidth]{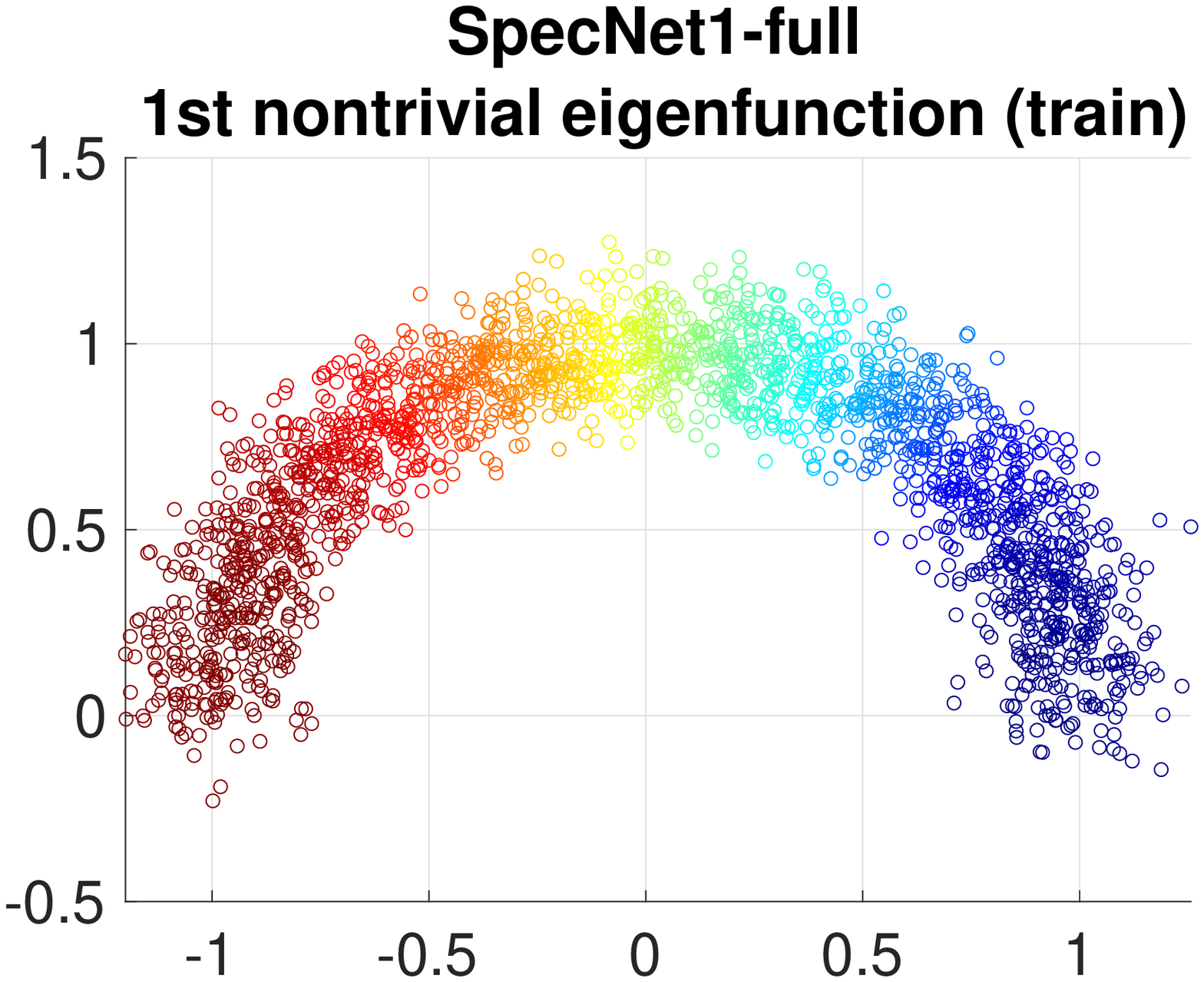}
    \includegraphics[width=0.245\textwidth]{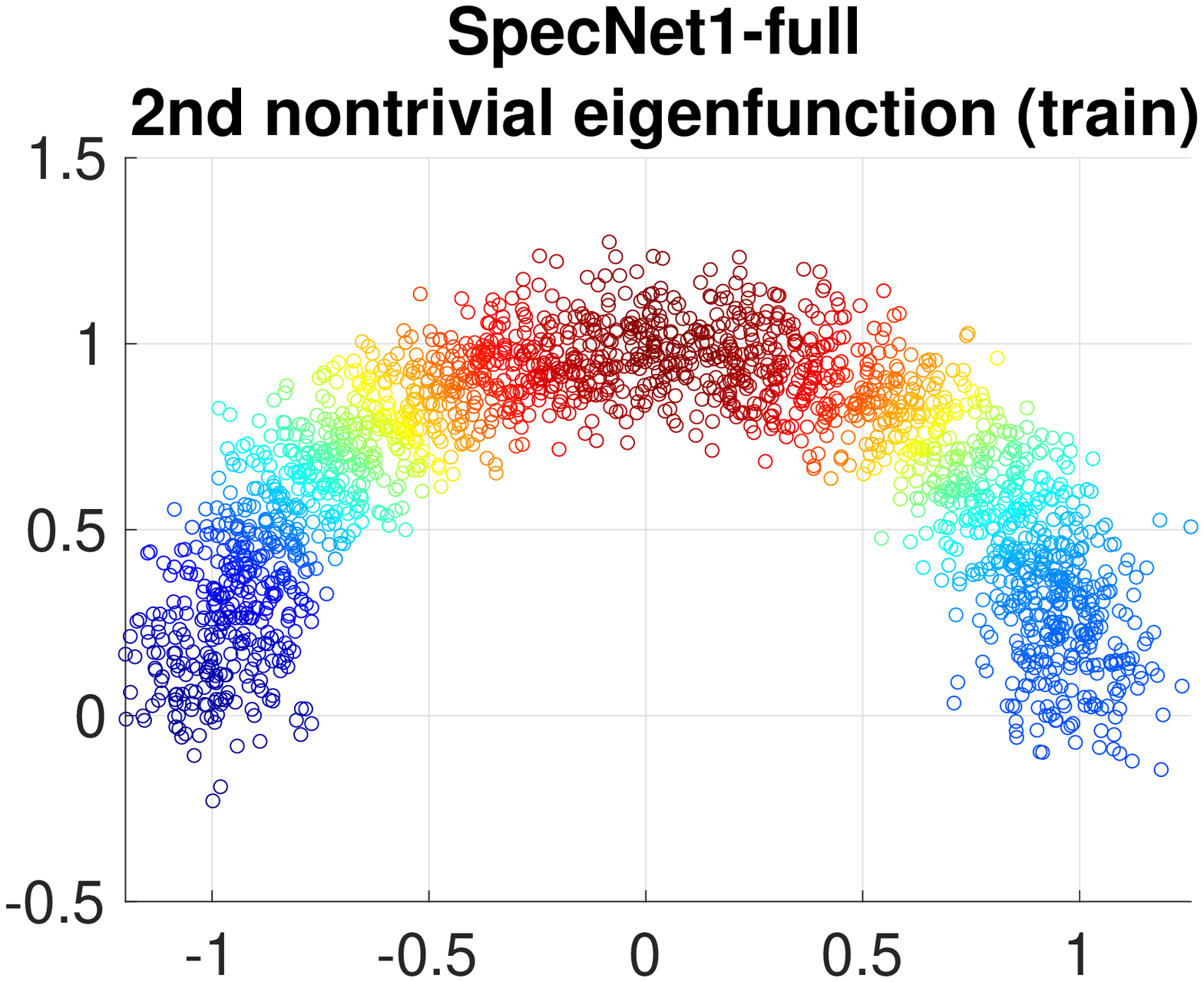}
    \includegraphics[width=0.245\textwidth]{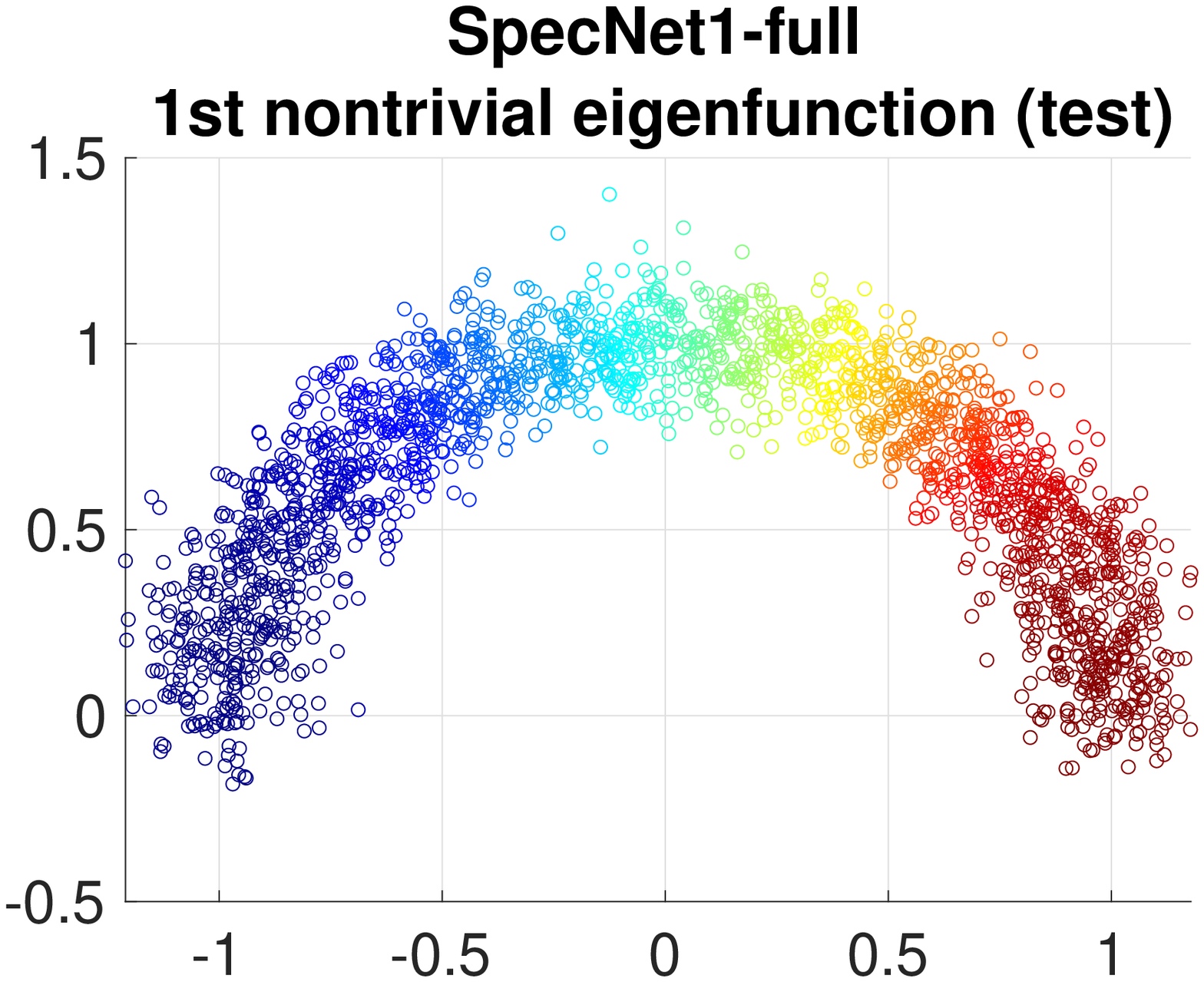}
    \includegraphics[width=0.245\textwidth]{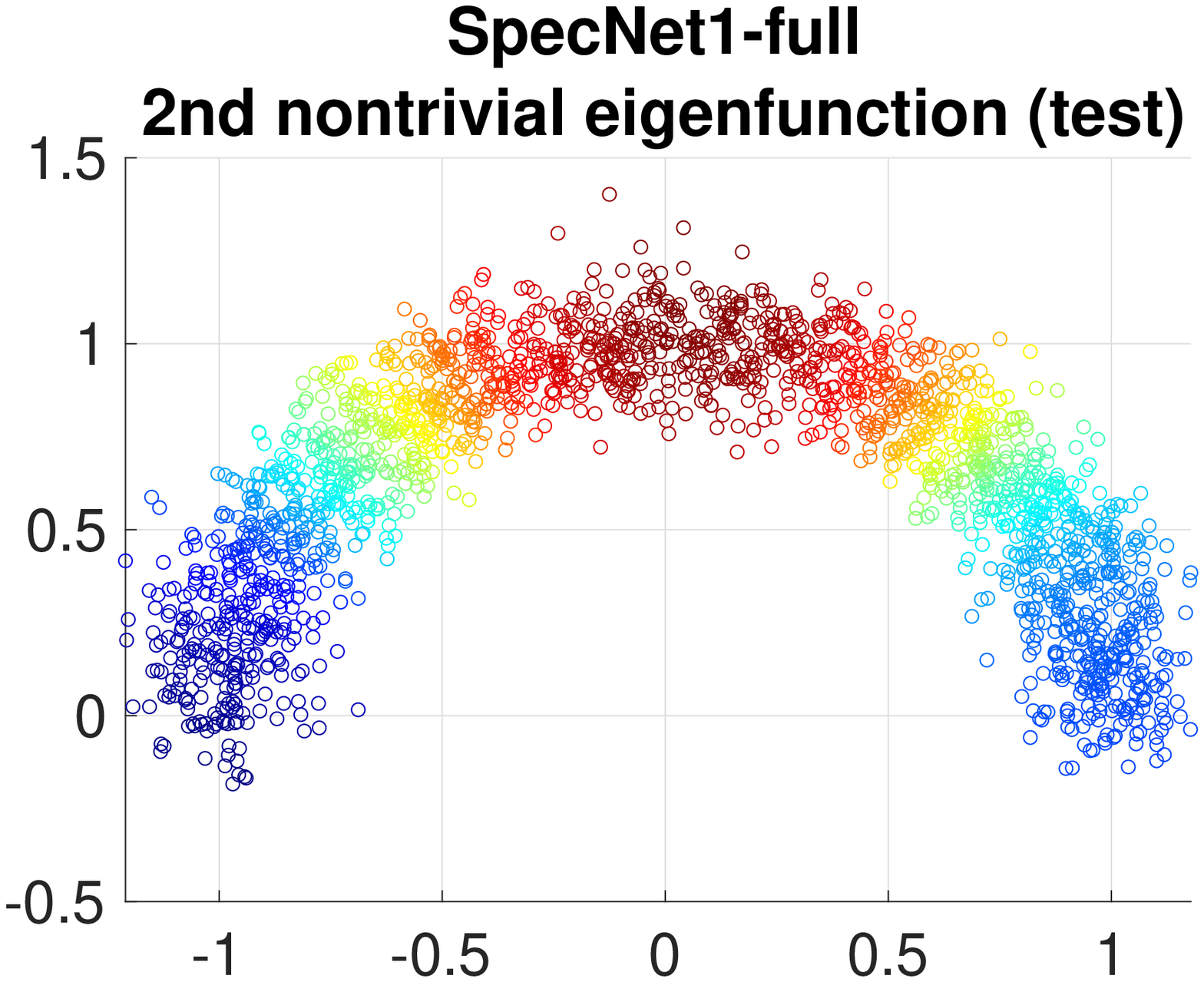}\\
    \includegraphics[width=0.245\textwidth]{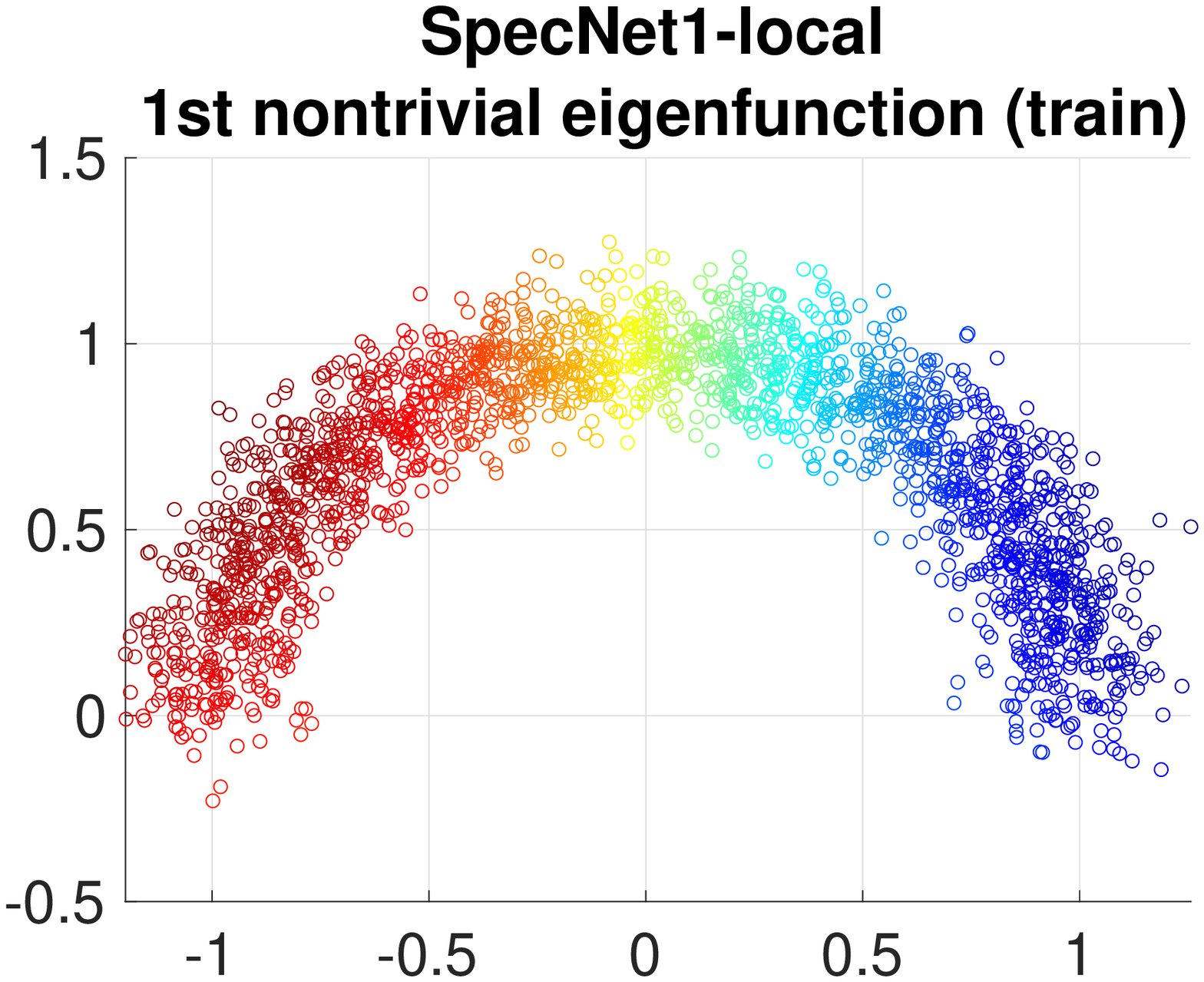}
    \includegraphics[width=0.245\textwidth]{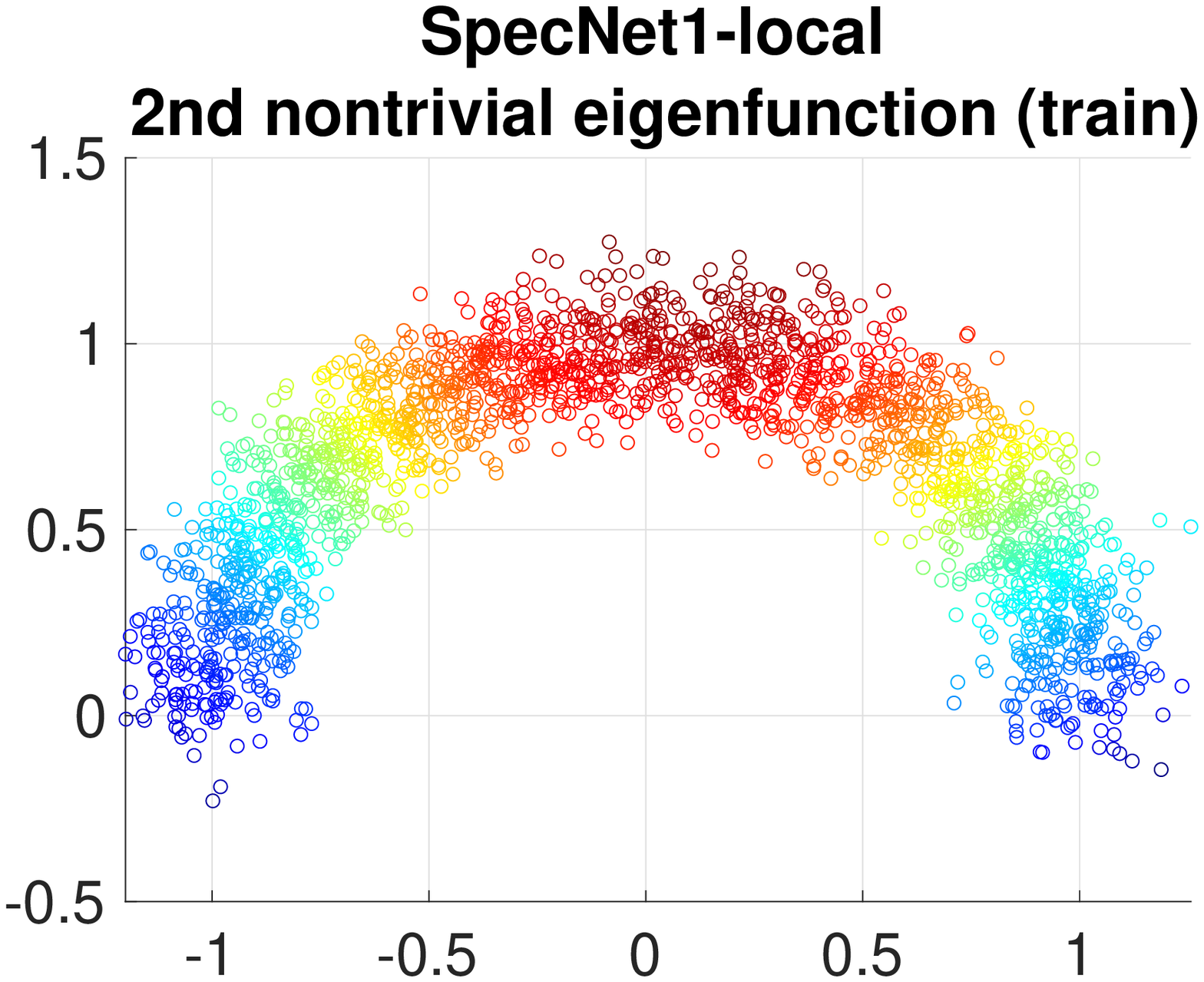}
    \includegraphics[width=0.245\textwidth]{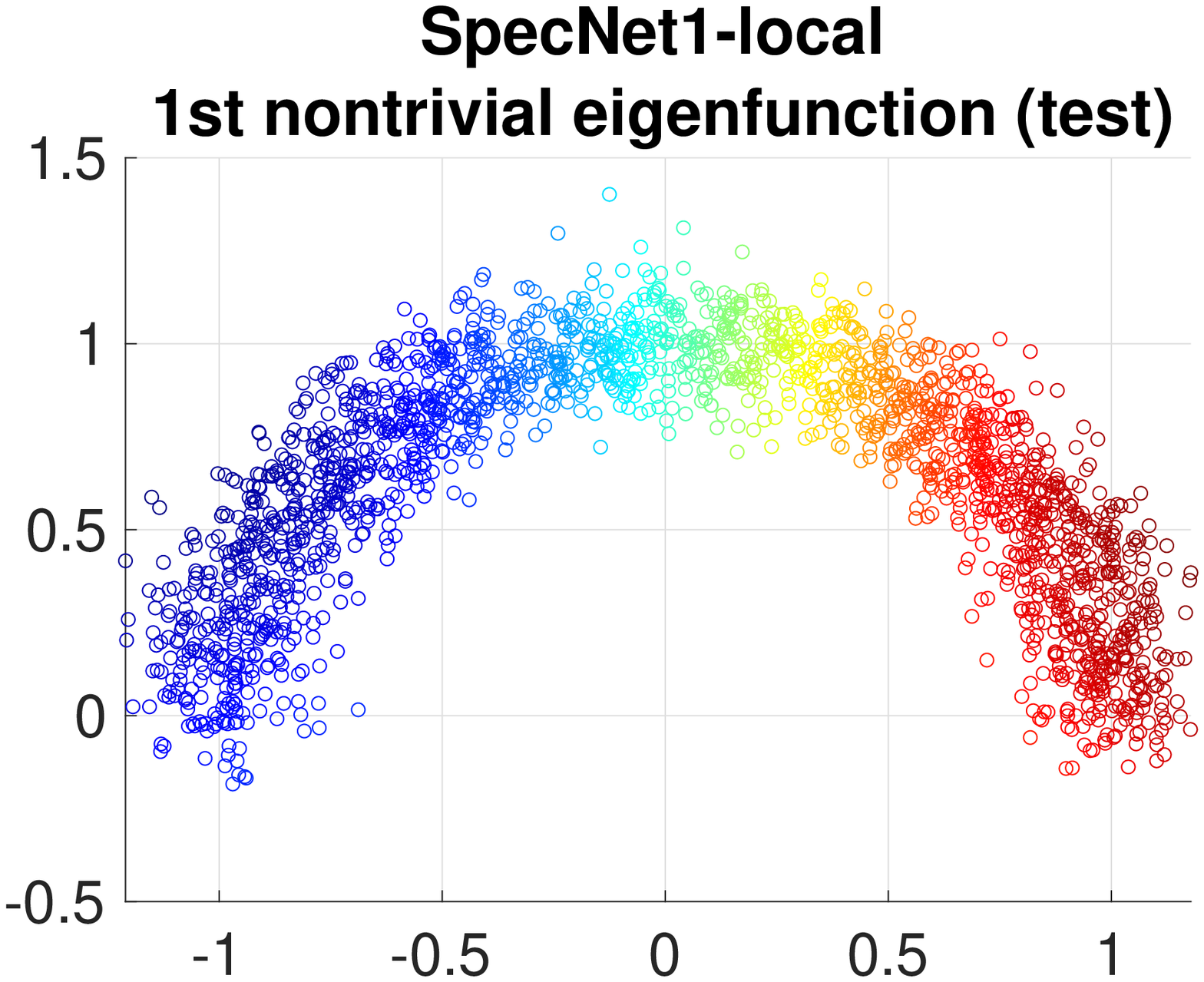}
    \includegraphics[width=0.245\textwidth]{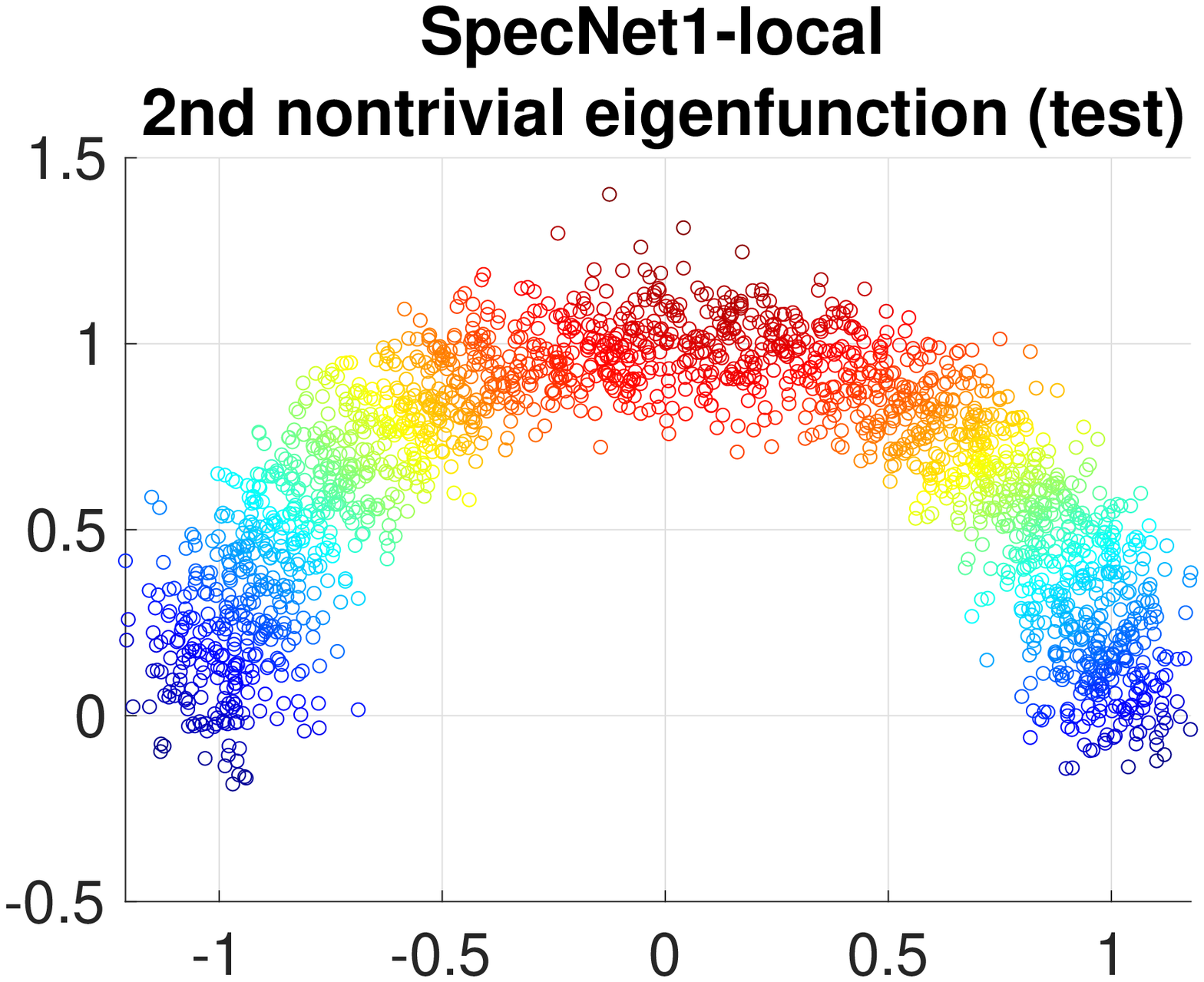}\\
    \includegraphics[width=0.245\textwidth]{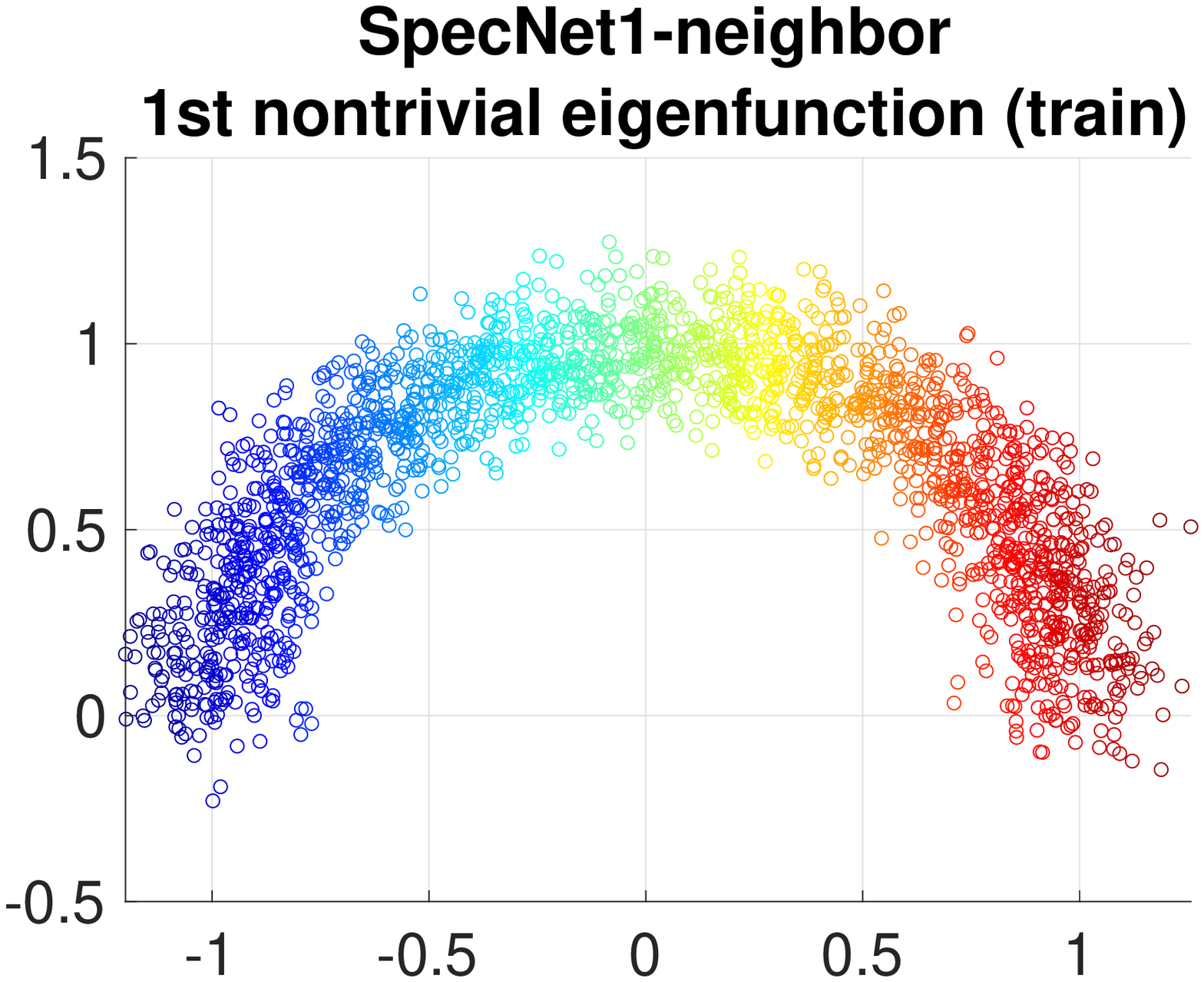}
    \includegraphics[width=0.245\textwidth]{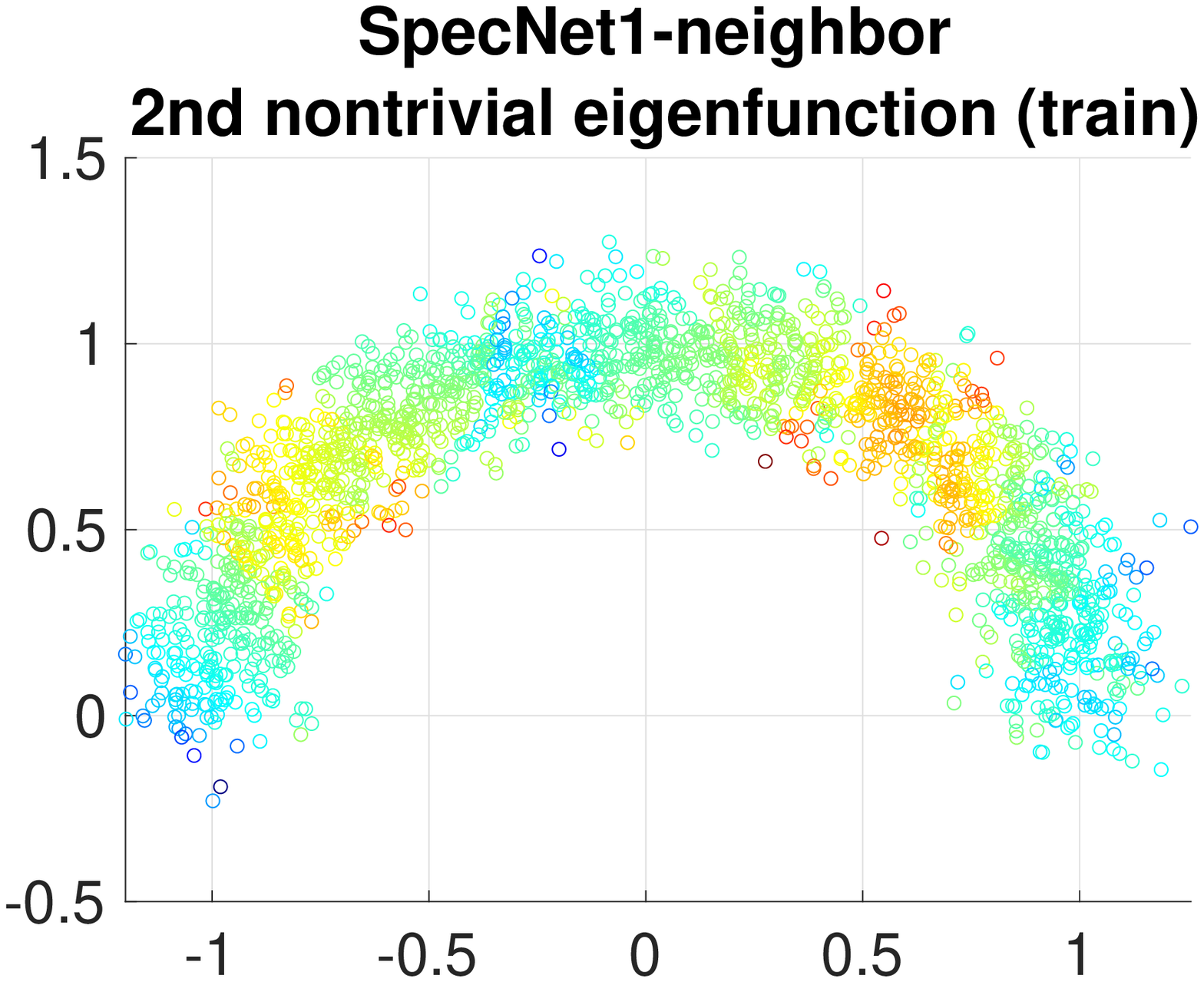}
    \includegraphics[width=0.245\textwidth]{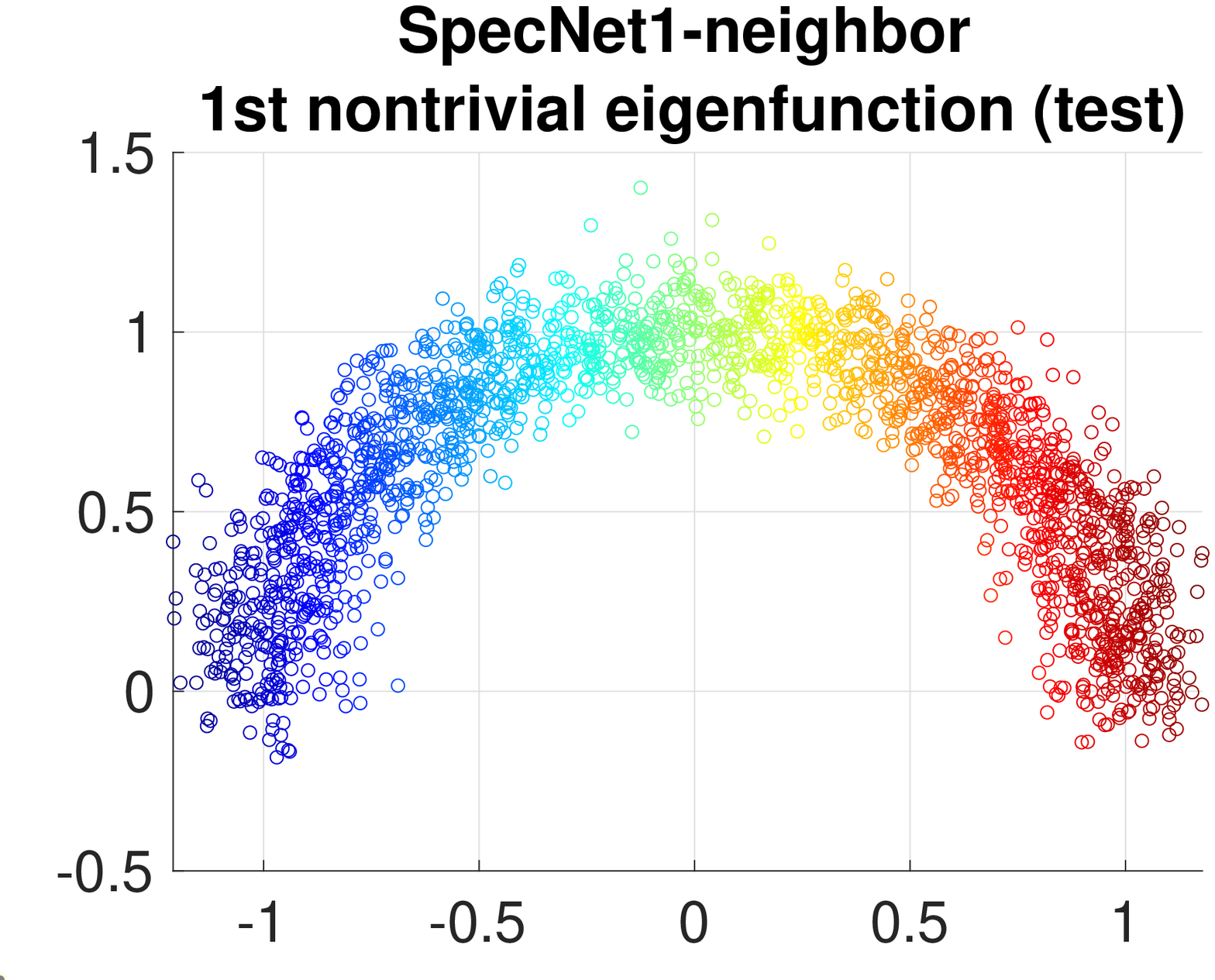}
    \includegraphics[width=0.245\textwidth]{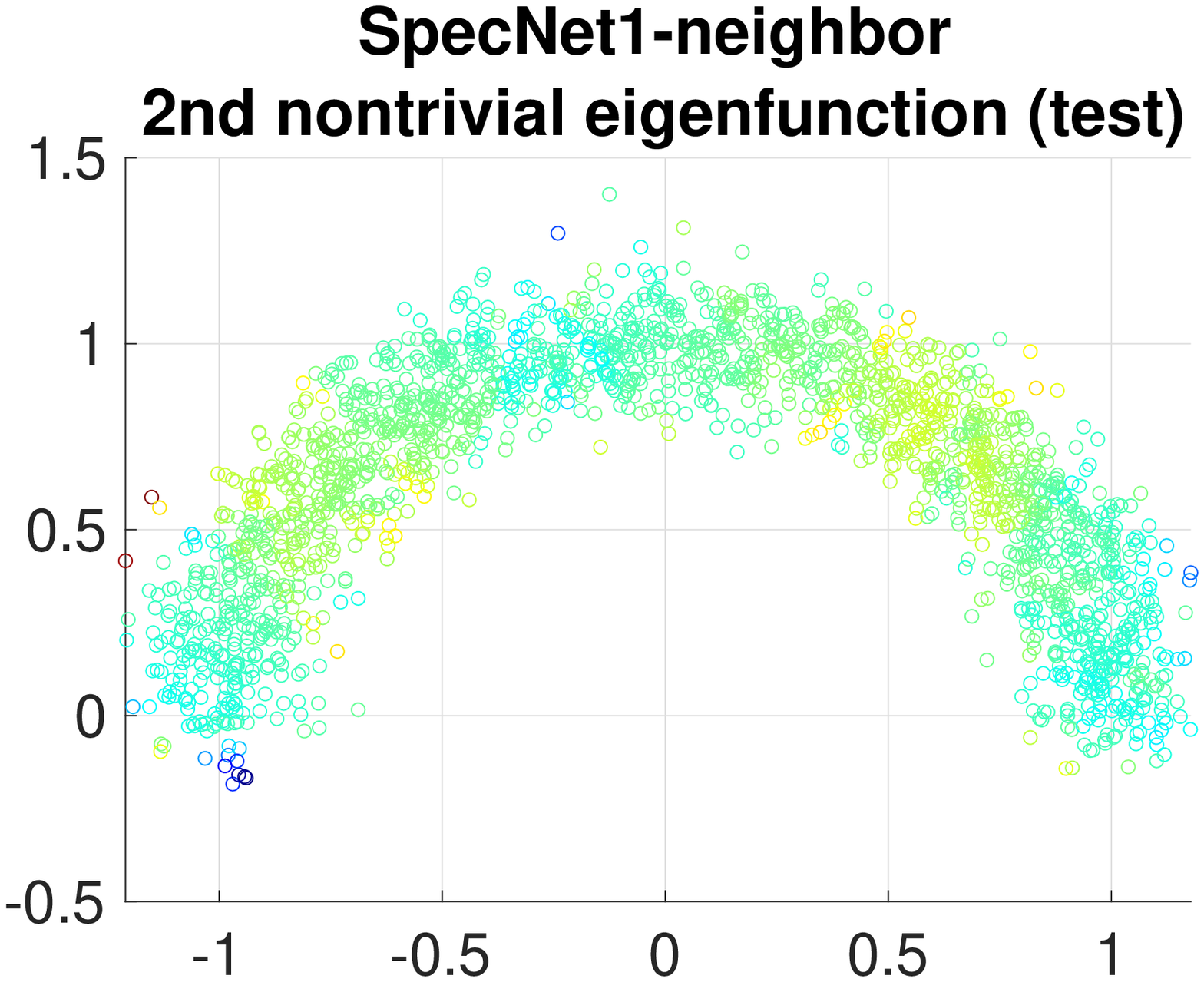}\\
    \includegraphics[width=0.245\textwidth]{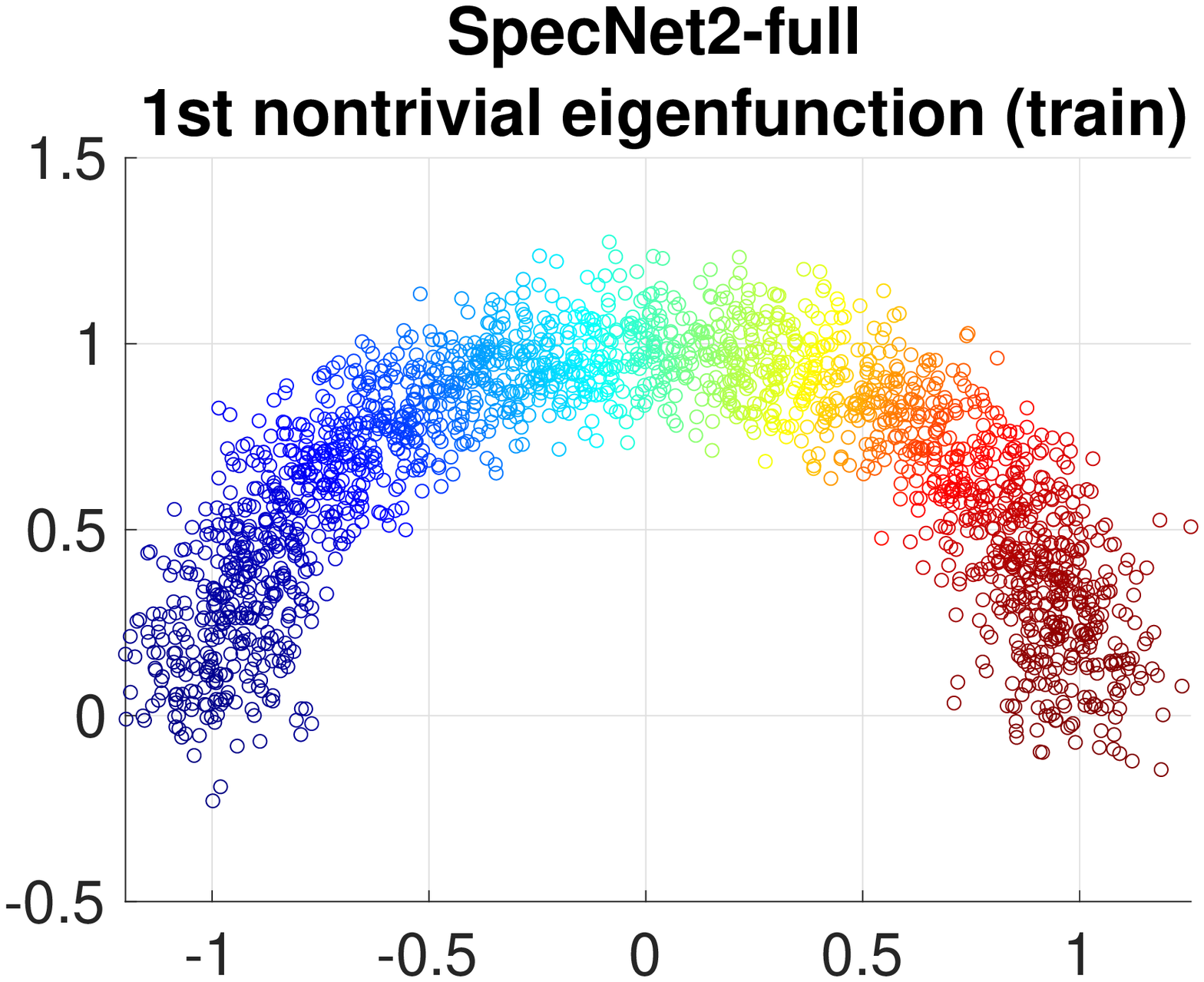}
    \includegraphics[width=0.245\textwidth]{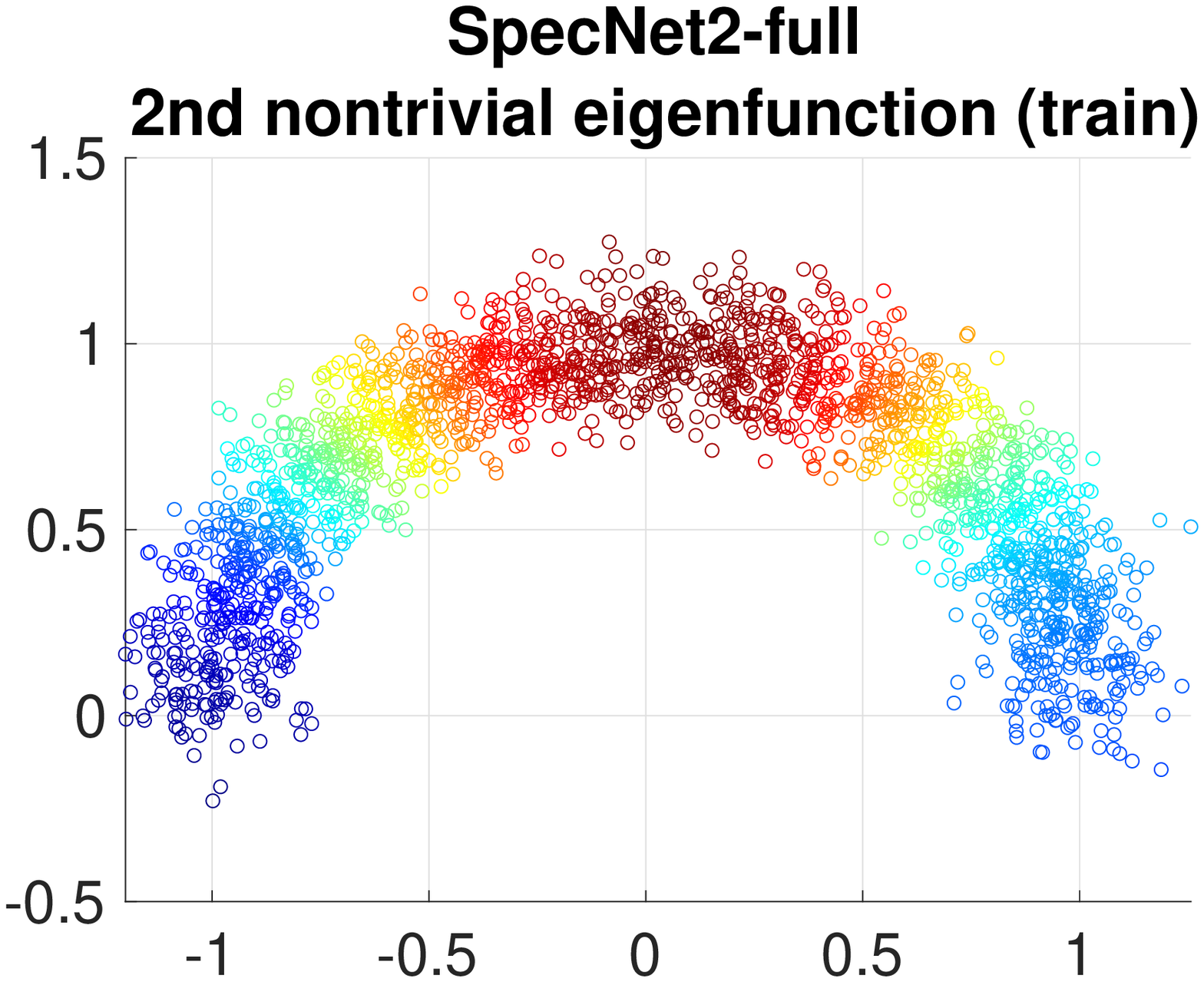}
    \includegraphics[width=0.245\textwidth]{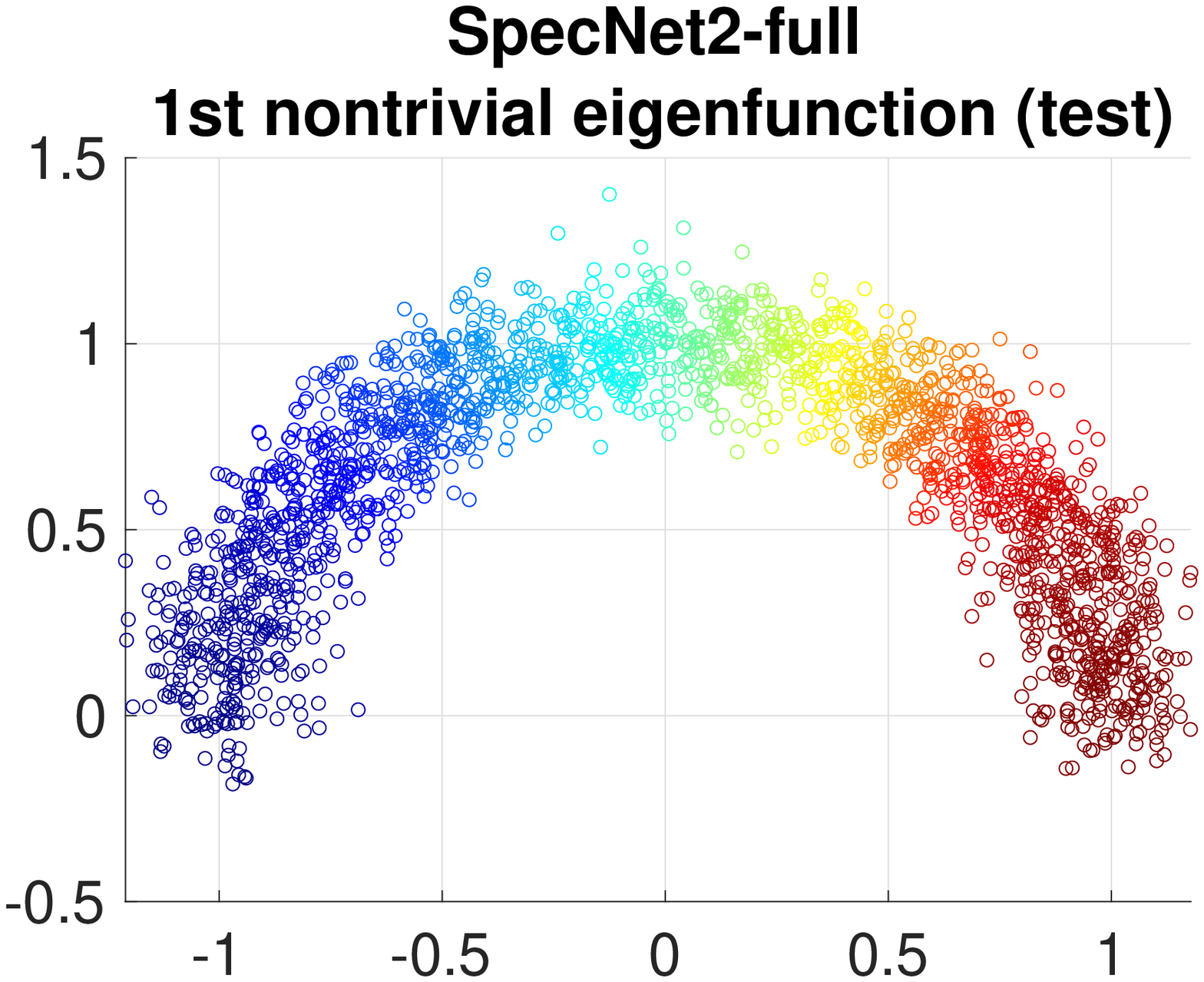}
    \includegraphics[width=0.245\textwidth]{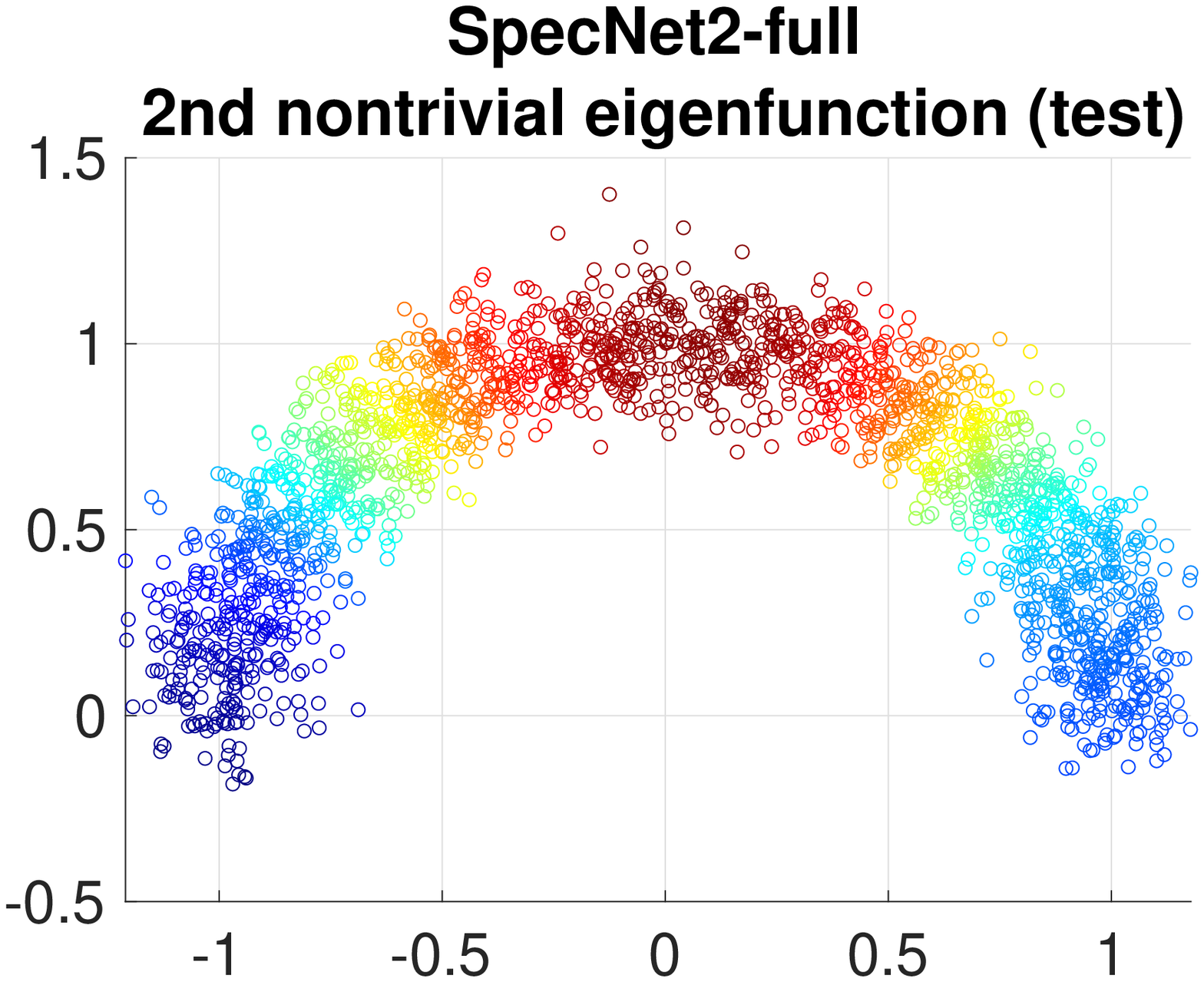}\\
    \includegraphics[width=0.245\textwidth]{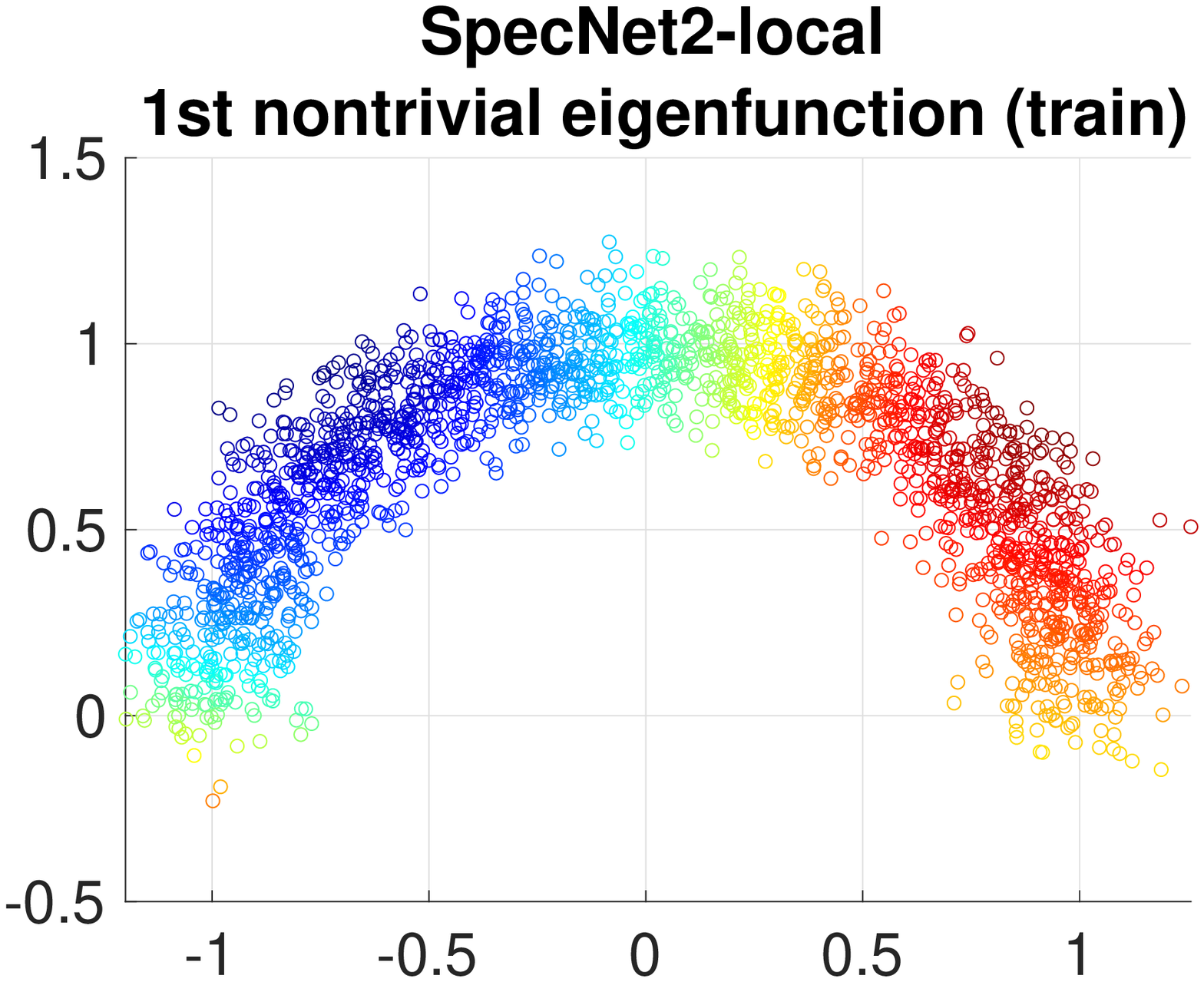}
    \includegraphics[width=0.245\textwidth]{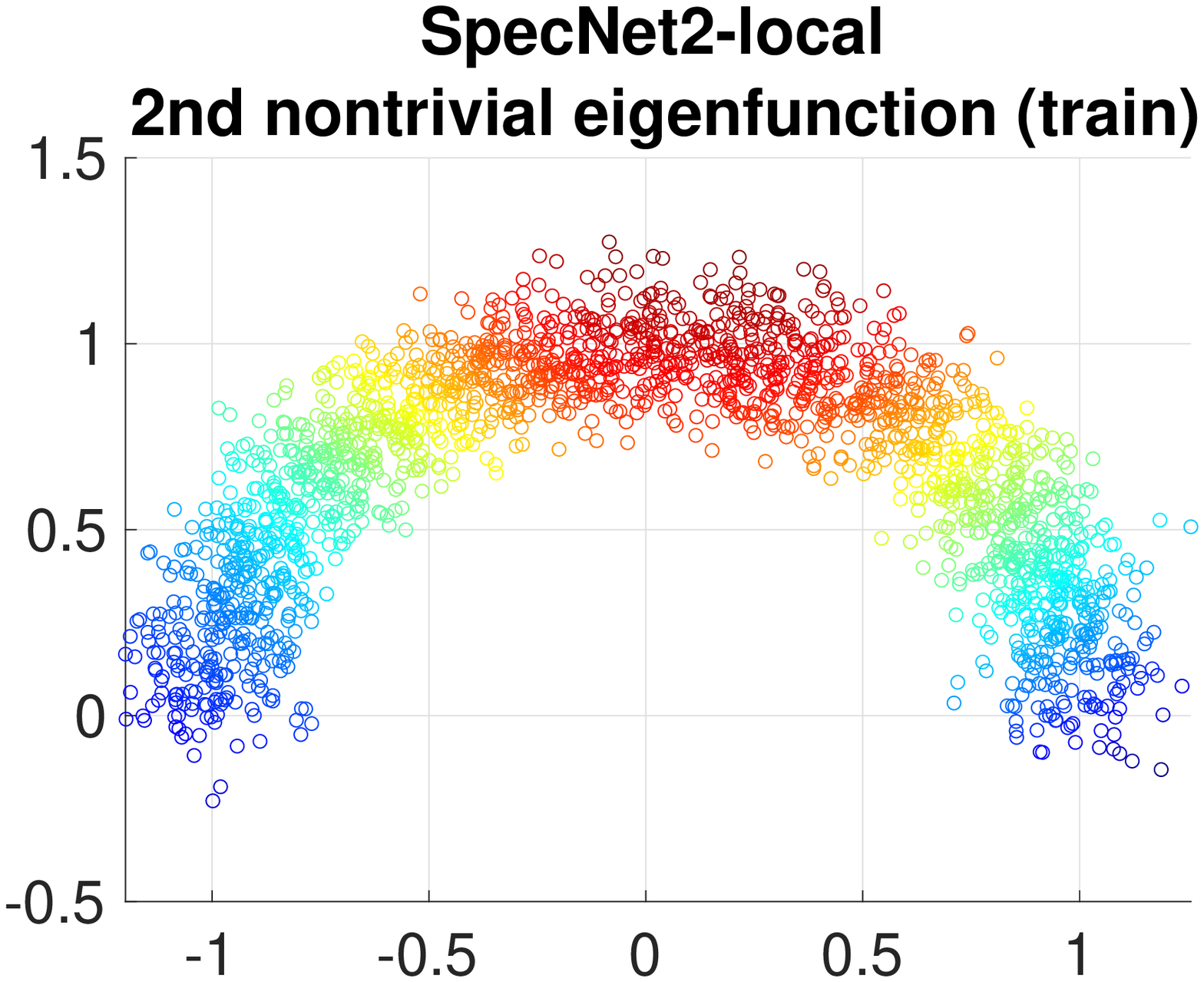}
    \includegraphics[width=0.245\textwidth]{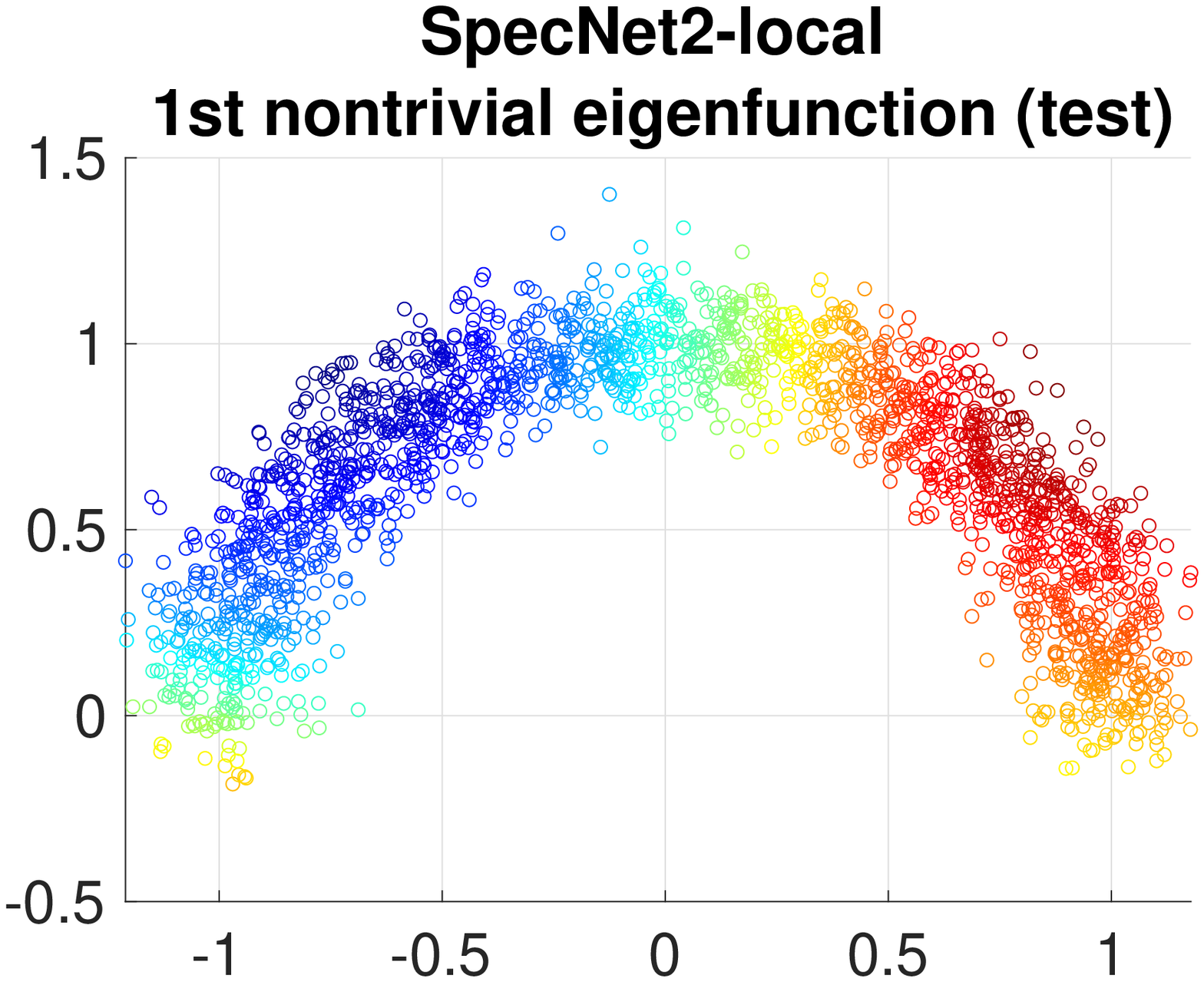}
    \includegraphics[width=0.245\textwidth]{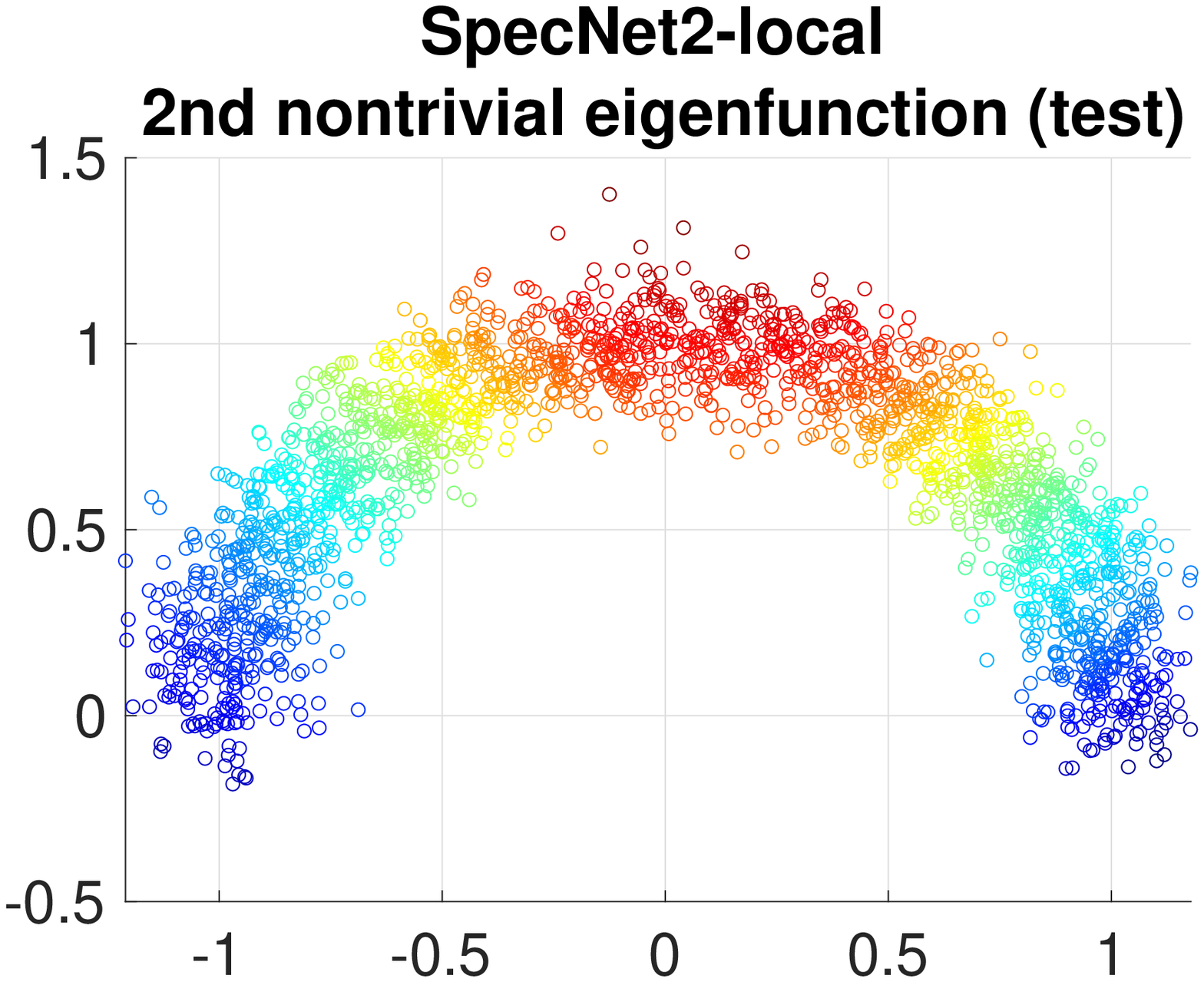}\\
    \includegraphics[width=0.245\textwidth]{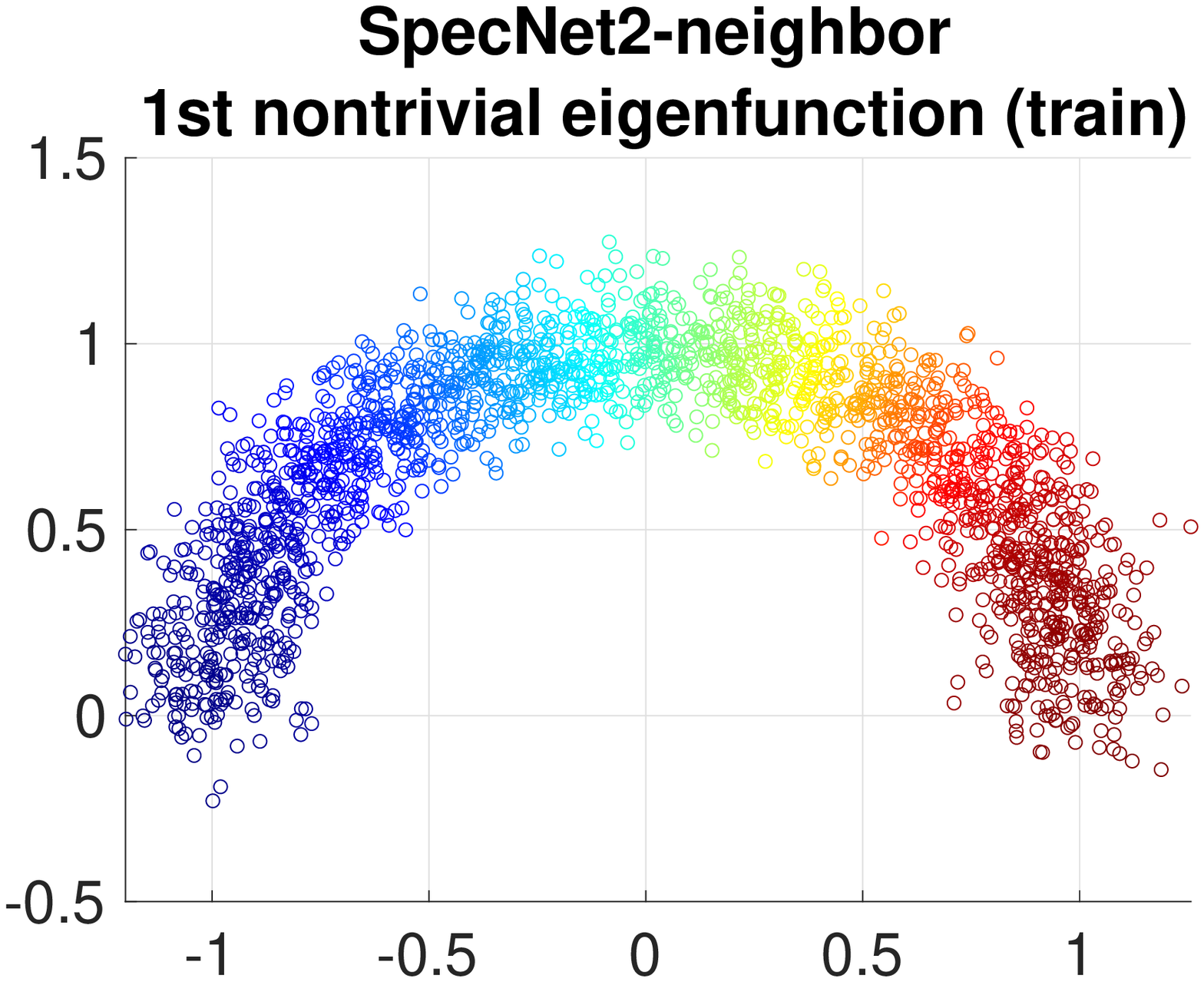}
    \includegraphics[width=0.245\textwidth]{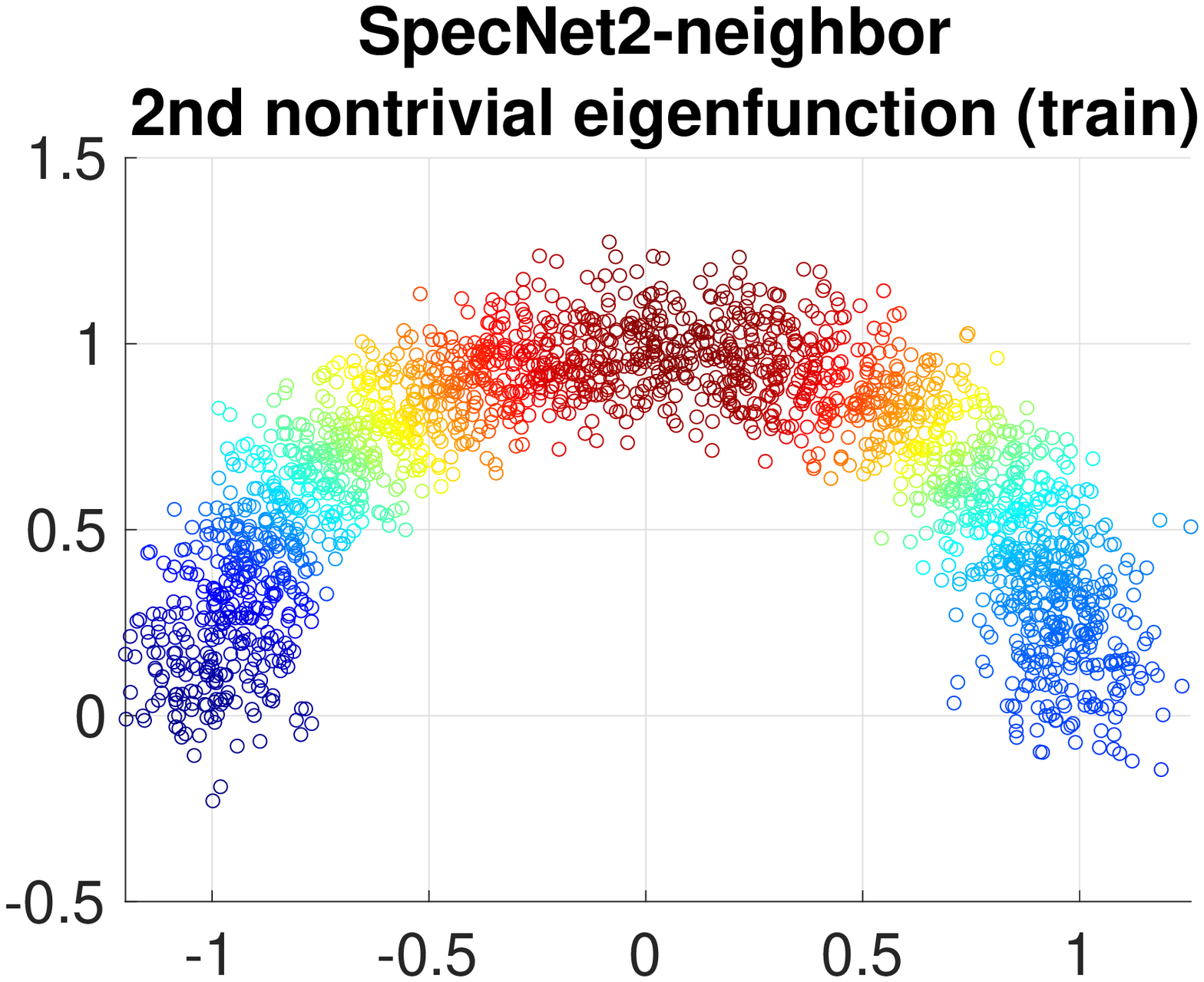}
    \includegraphics[width=0.245\textwidth]{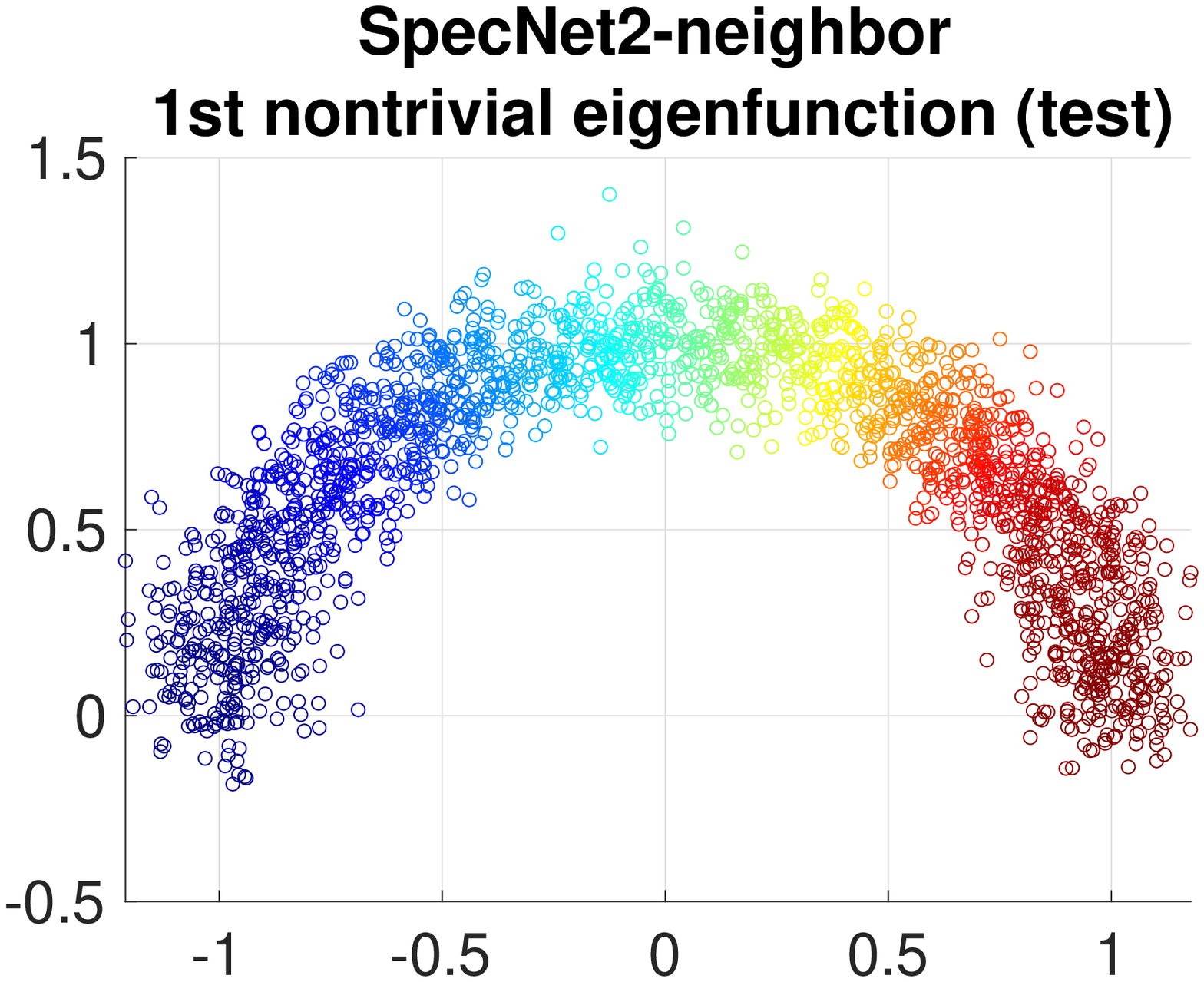}
    \includegraphics[width=0.245\textwidth]{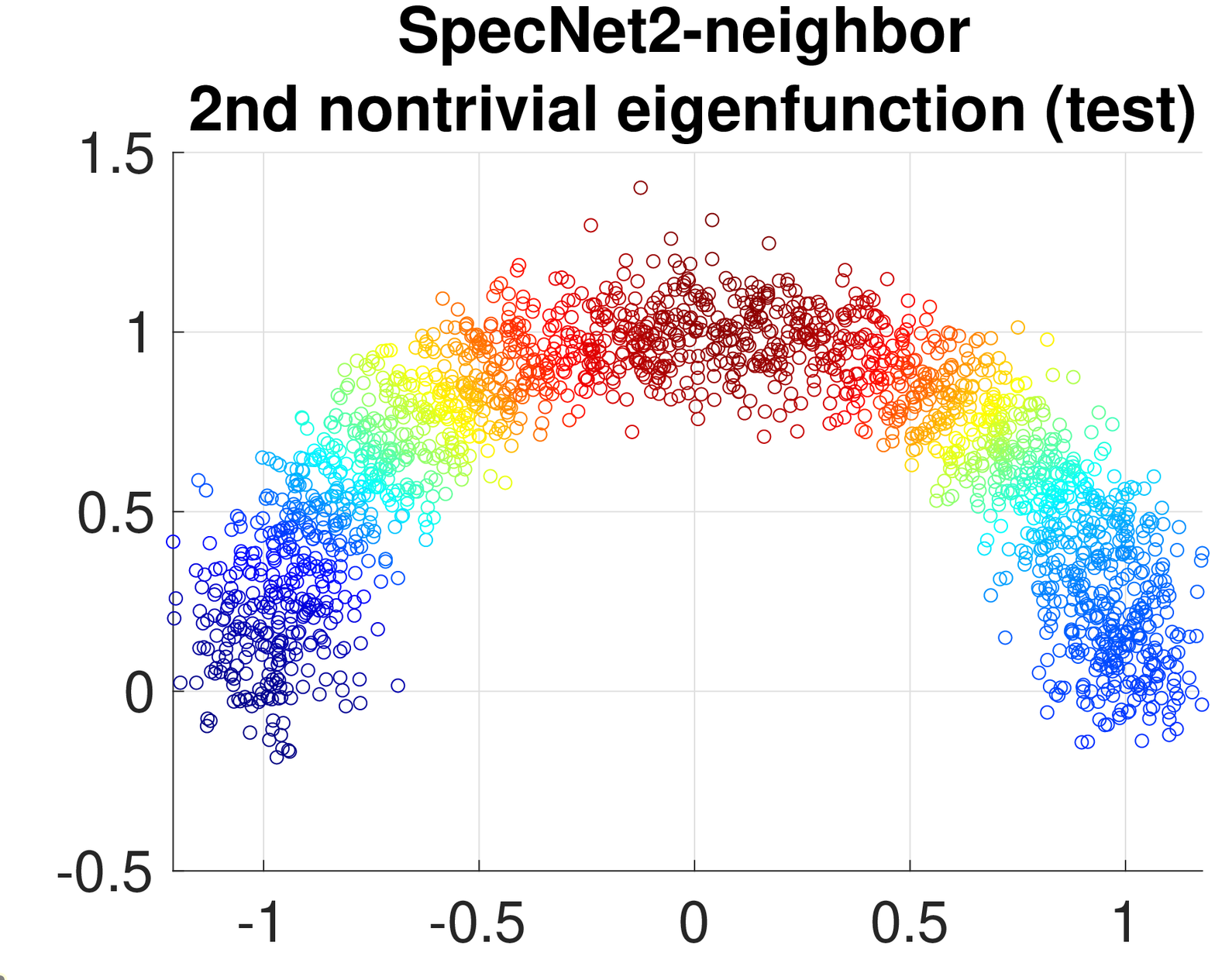}
    \caption{
    One moon datset: embeddings by different methods using the first two nontrivial eigenfunctions. The first row is the ground truth.}
    \label{fig:onemoon-embed}
\end{figure*}

\begin{figure*}[htbp]
    \centering
    \includegraphics[width=0.245\textwidth]{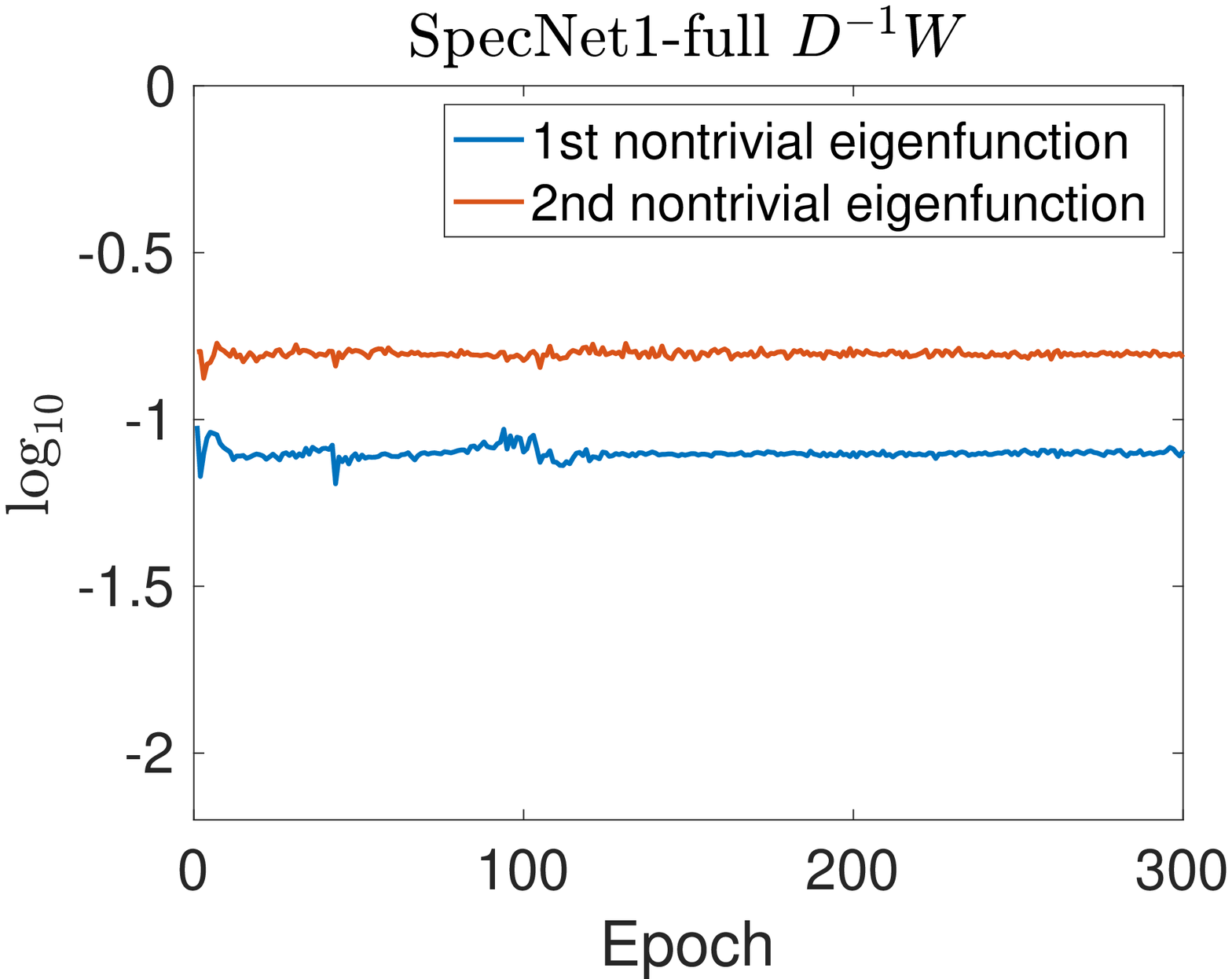}
    \includegraphics[width=0.245\textwidth]{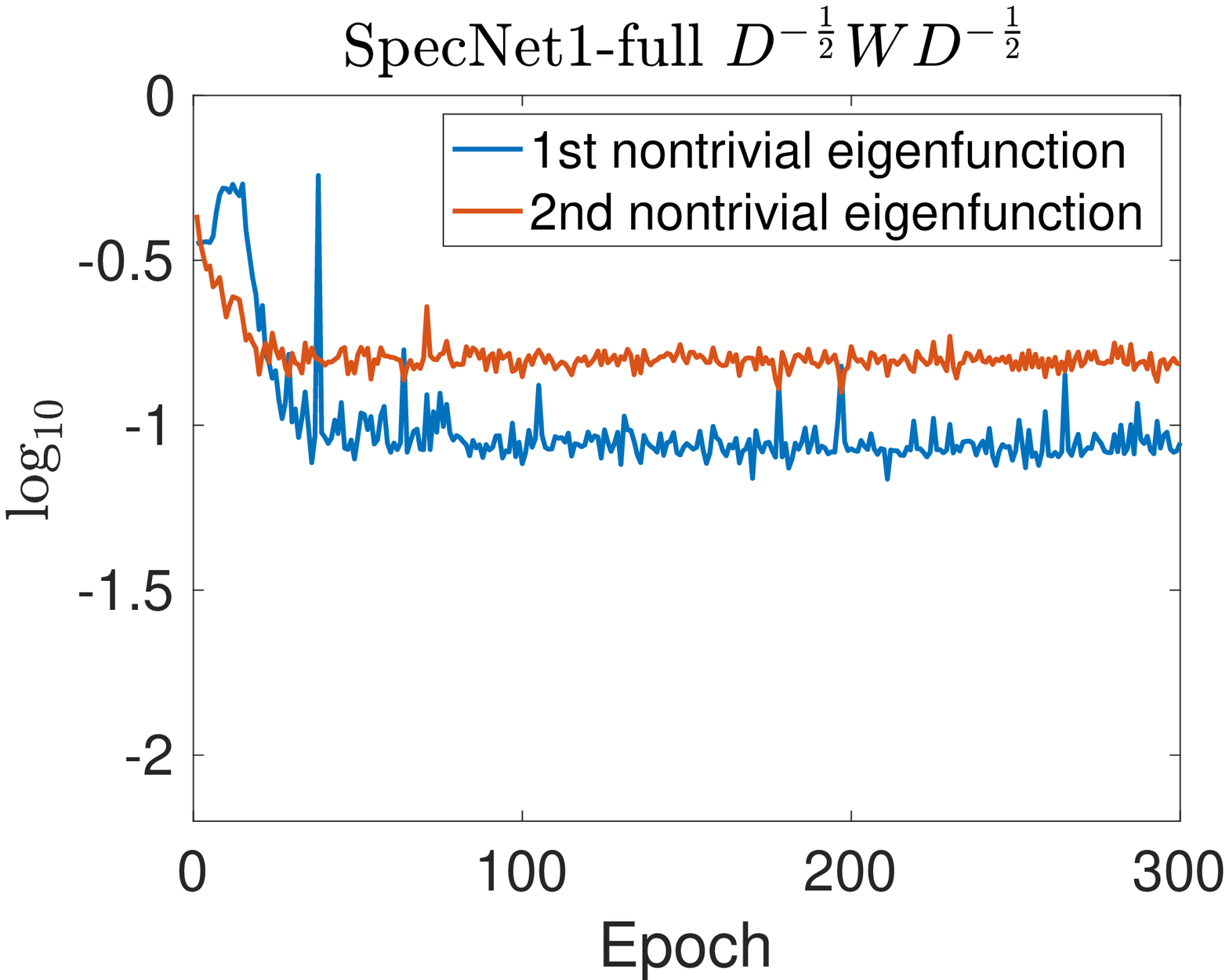}
    \includegraphics[width=0.245\textwidth]{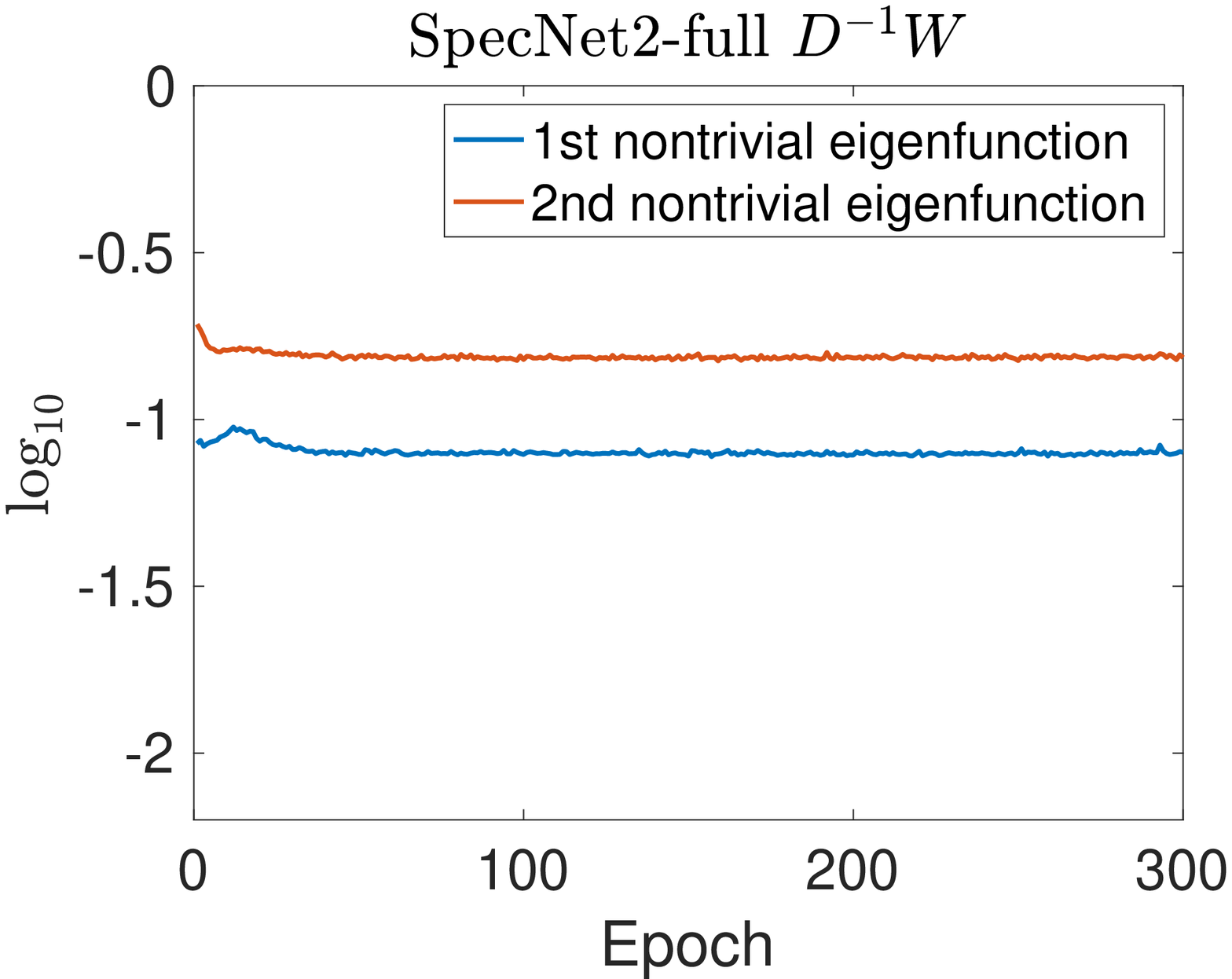}
    \includegraphics[width=0.245\textwidth]{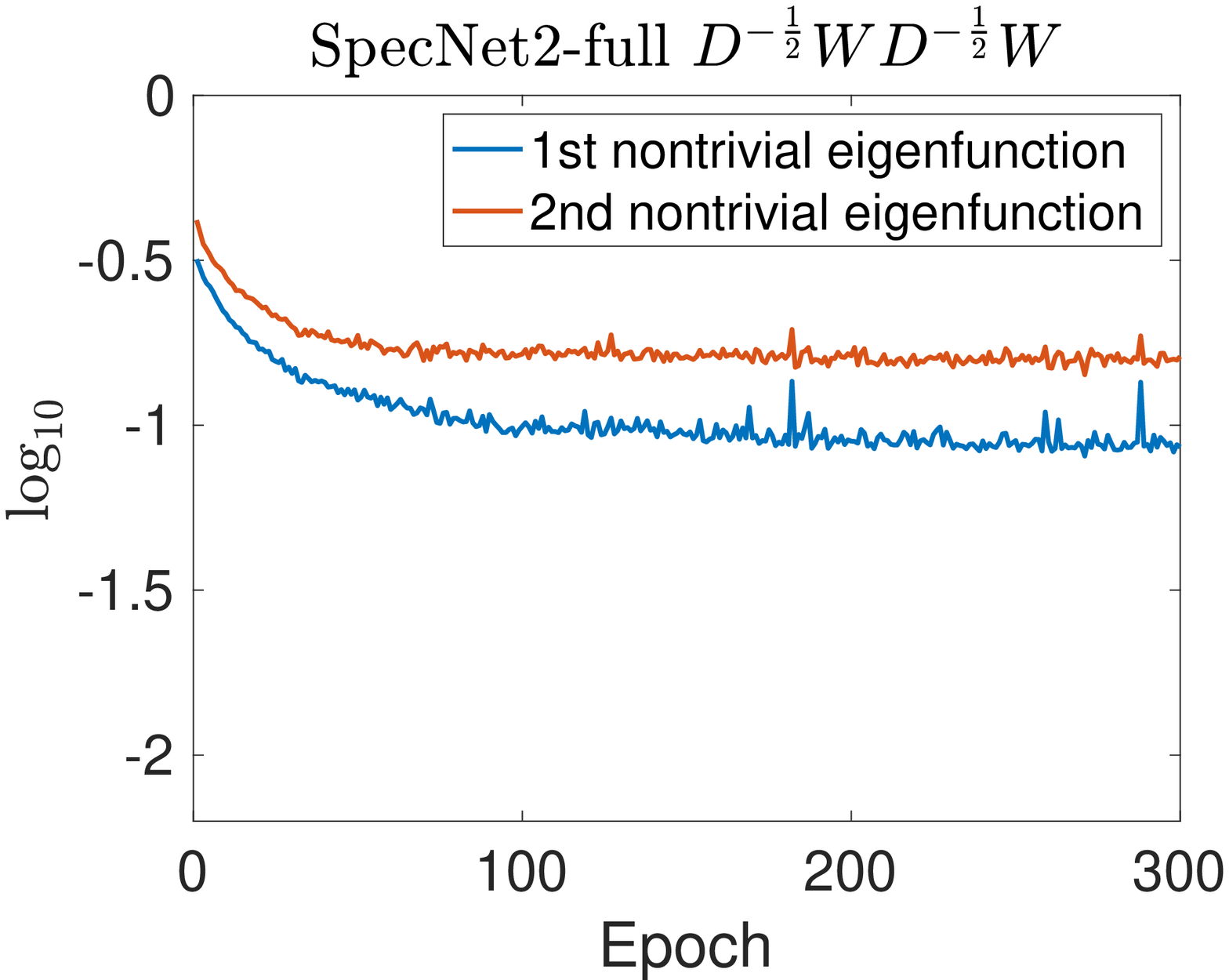}
    \caption{
    One moons datset: network approximations of eigenfunctions using different Laplacian matrix on the training set.}
    \label{fig:onemoon-compare}
\end{figure*}

\subsection{Two moons data}
\noindent\textbf{Data generation:} The two moons training set
consists of $n = 2000$ points in $\mathbb{R}^2$. One piece of moons is
generated by the equation $x_i = (\cos \eta_i-0.5, \sin
\eta_i-0.3) + \xi_i$, $i = 1,\dots,1000$, and the other piece is
generated by $x_i = (-\cos \eta_i+0.5,
-\sin\eta_i+0.3) + \xi_i$, $i = 1001,\dots,2000$, where $\eta_i$ are i.i.d. uniformly sampled on $[0,\pi]$ and
$\xi_i$ are i.i.d. Gaussian random variables of dimension two drawn from
$\calN(0, 0.0036 I_2)$. The testing set consists of 2000 points and is generated in the same way as the training set with a different realization. The sparse affinity matrix associated with the traning set is generated via
Gaussian kernel with bandwidth $\sigma=0.15$, and truncated at threshold 0.08.

\noindent\textbf{Network training:}
We use a fully-connected feedforward neural network with single 128-unit hidden layer:

SpecNet1: $2 \xrightarrow[]{\text{fc}} 128 - \text{ReLU} \xrightarrow[]{\text{linear}} 2 \xrightarrow[]{\text{orthogonal}} 2$;

SpecNet2: $2 \xrightarrow[]{\text{fc}} 128 - \text{ReLU} \xrightarrow[]{\text{linear}} 1$.

The batch size is 4 (the average number of neighbors of batches of size 4 in the sparse affinity matrix is about 670), and we use the Adam as the optimizer with learning rate $10^{-3}$ for SpecNet2 and $10^{-5}$ for SpecNet1.

\noindent\textbf{Error evaluation:} 
The classification is done in an unsupervised way.
Specifically, we label the training and testing samples that are generated
from one piece of moons as 1; label those samples generated from the other
piece of moons as 2; use them as the ground truth and train the network on
the training data without labels. Let us take the classification accuracy
on the training data as an example, we evaluate the network output
function corresponding to the first nontrivial eigenvector on the training
set and perform the standard $K$-means algorithm ($K=2$) to split their
one-dimensional embedding into two clusters also labeled as number 1 or 2,
denoted as $\tilde{\gamma}\in\mathbb{R}^n$. Denote the ground truth of
labels on the training set as $\gamma\in\mathbb{R}^n$. The classification
accuracy is computed by $\max\{\frac{\sum_{i=1}^n
\abs{\gamma_i-\tilde{\gamma}_i}}{n}, 1-\frac{\sum_{i=1}^n
\abs{\gamma_i-\tilde{\gamma}_i}}{n}\}$. The classification accuracy on the
testing set can be computed in a similar way.

\begin{figure*}[htbp]
    \centering
    \includegraphics[width=0.245\textwidth]{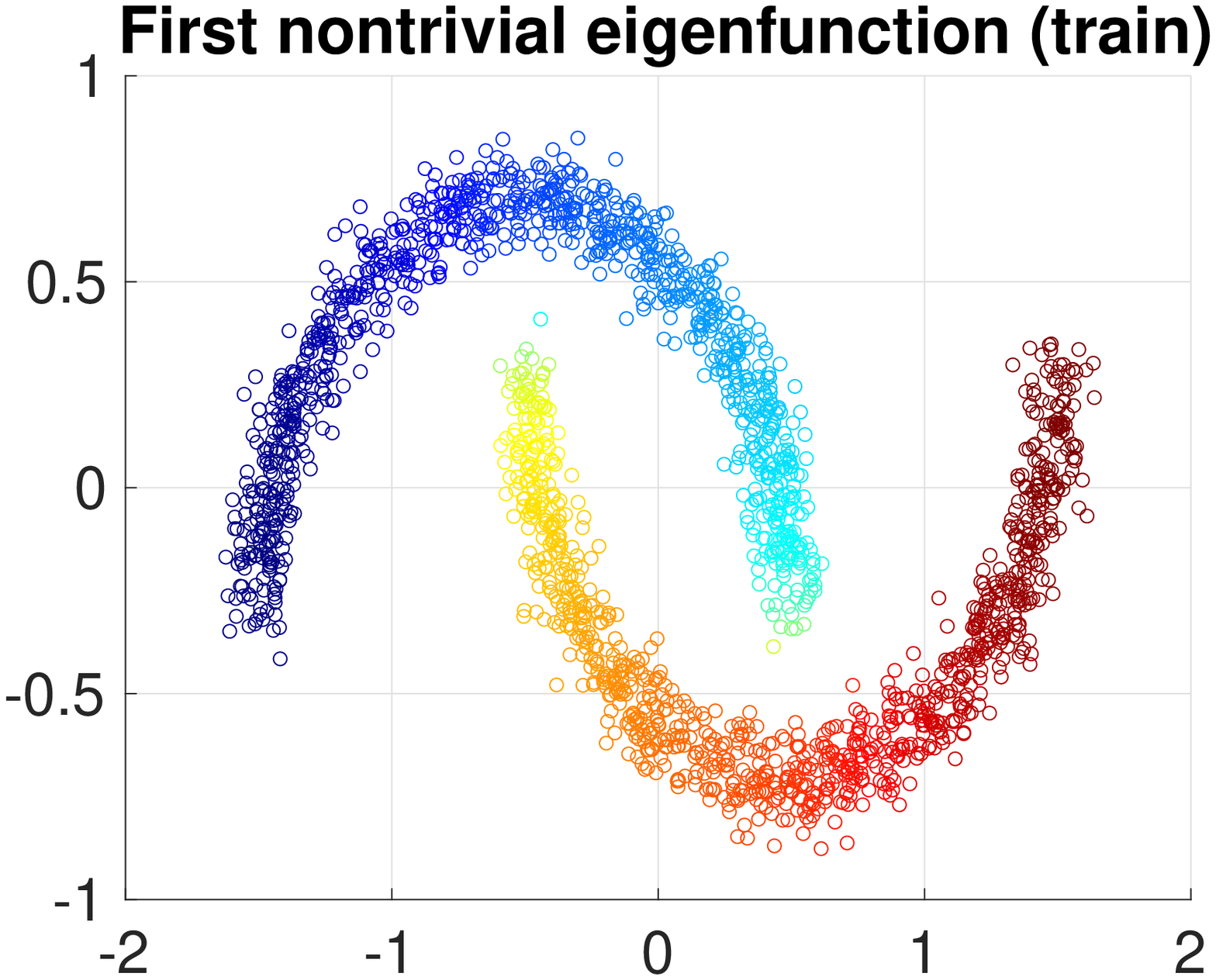}
    \includegraphics[width=0.245\textwidth]{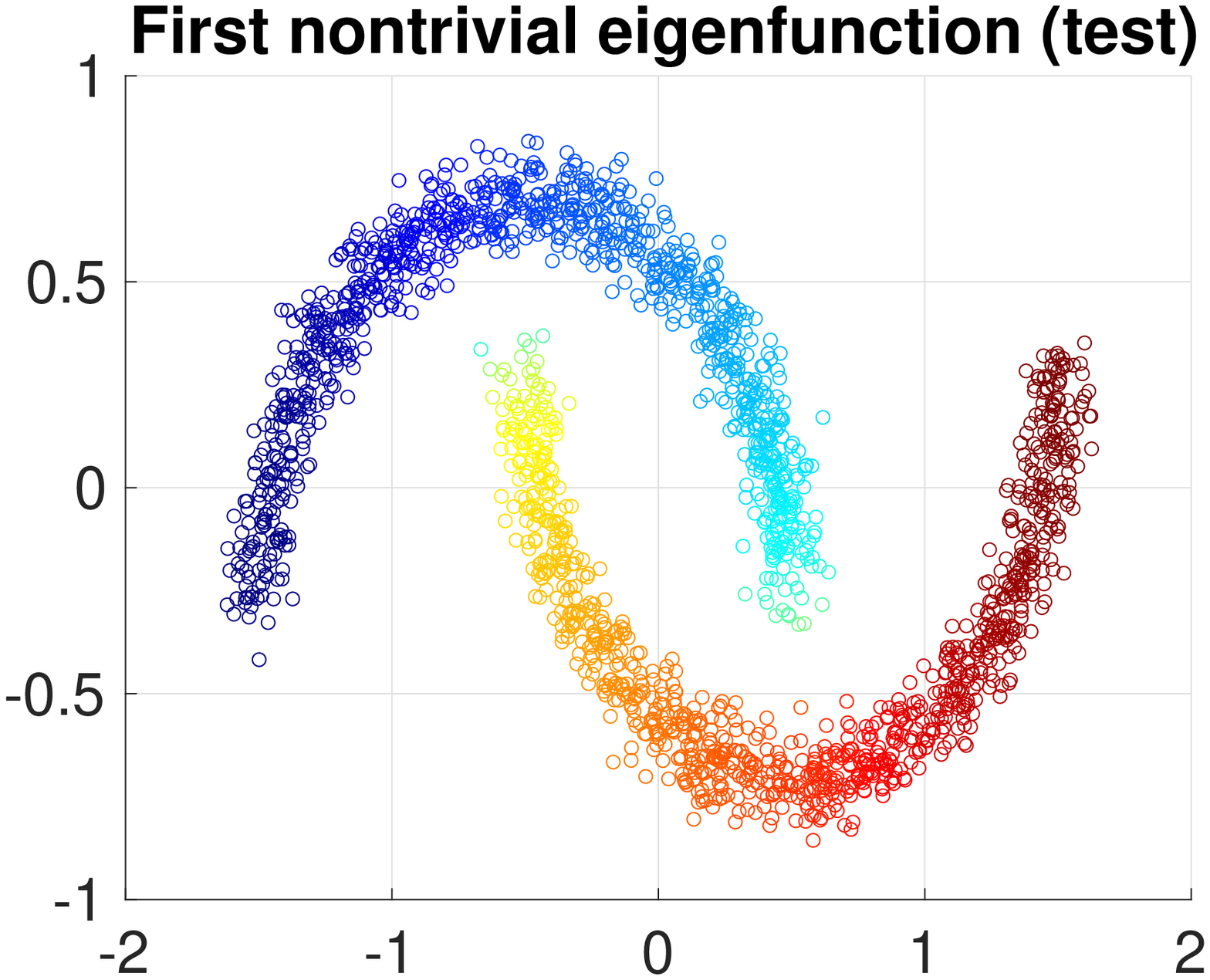}\\
    \includegraphics[width=0.245\textwidth]{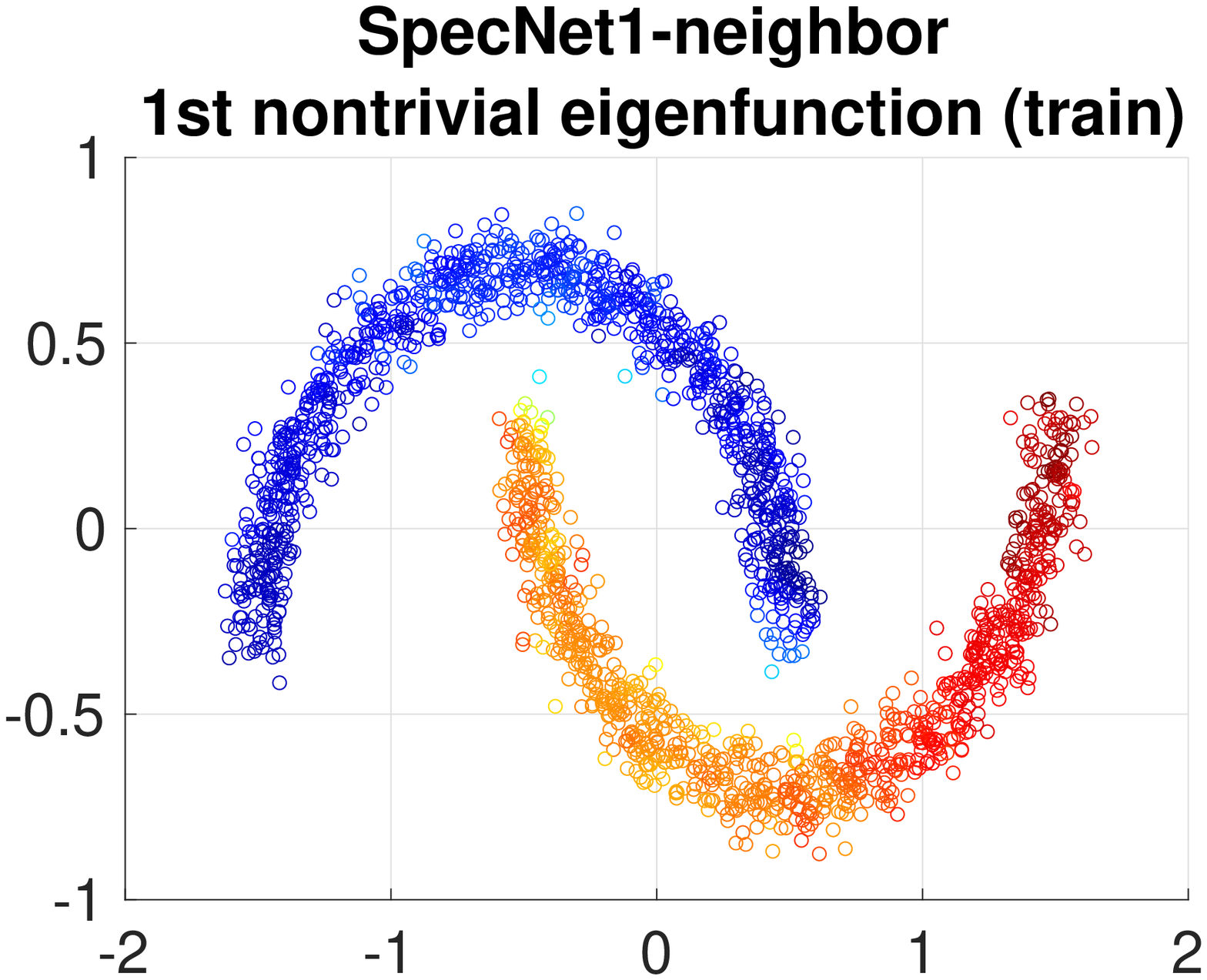}
    \includegraphics[width=0.245\textwidth]{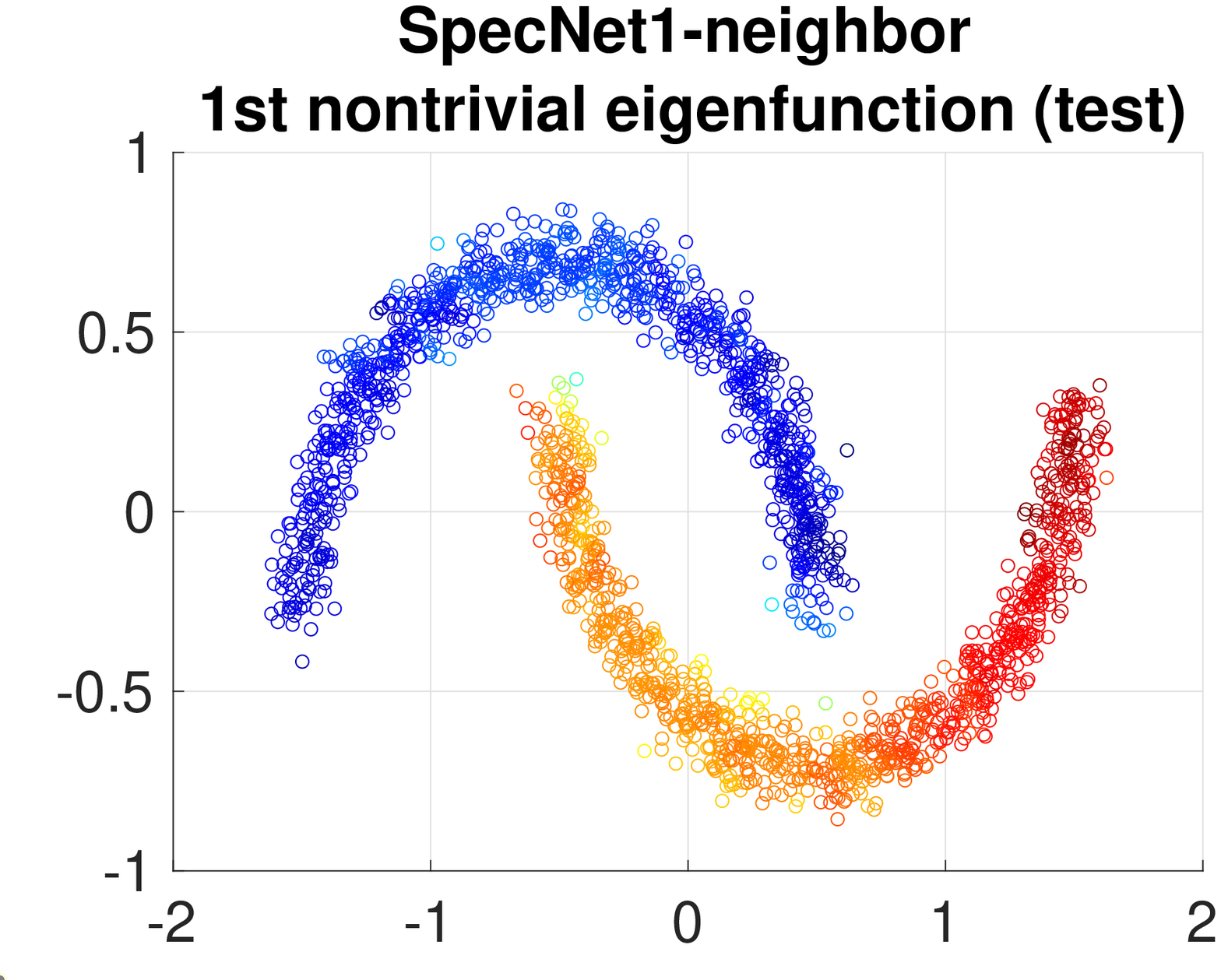}
    \includegraphics[width=0.245\textwidth]{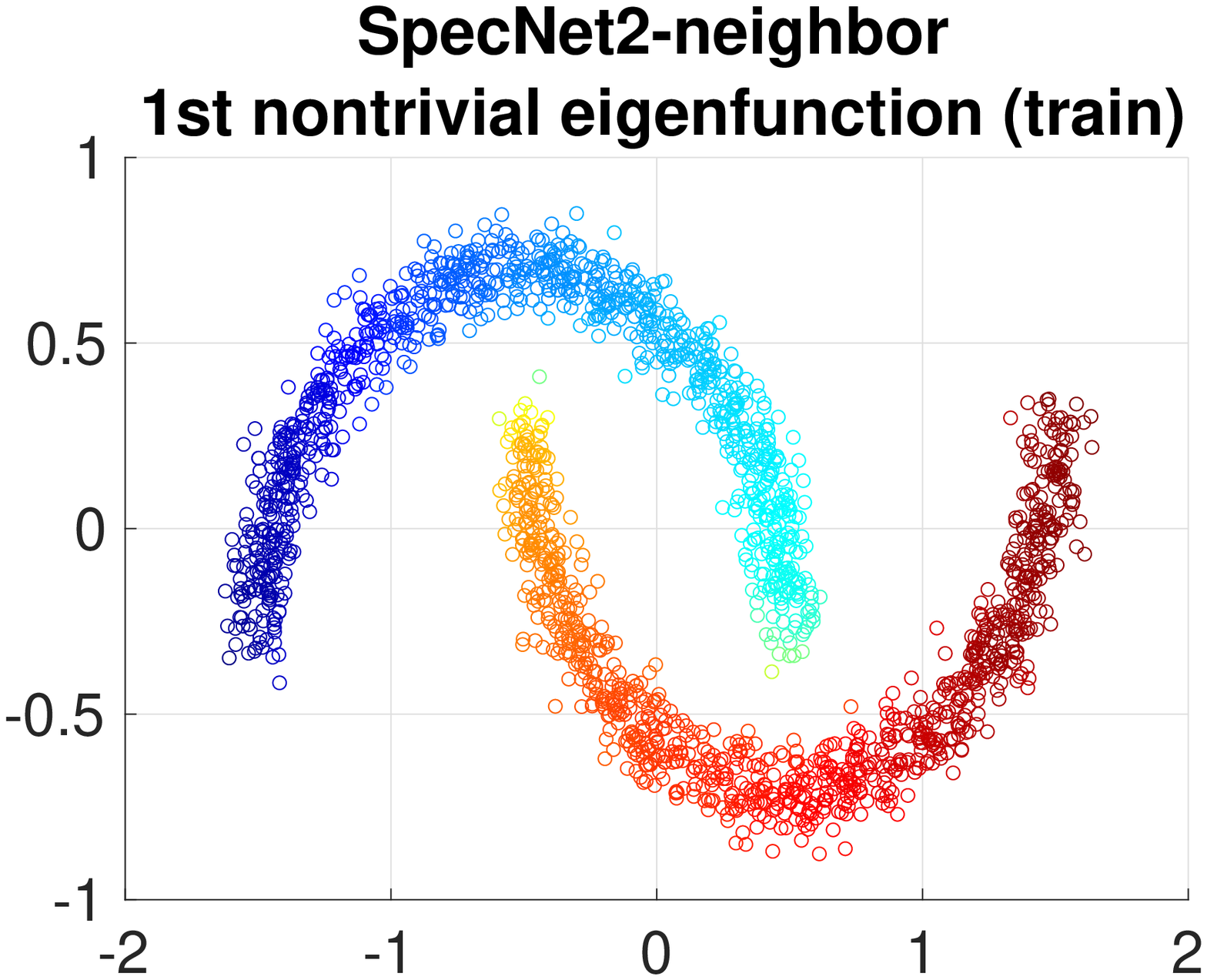}
    \includegraphics[width=0.245\textwidth]{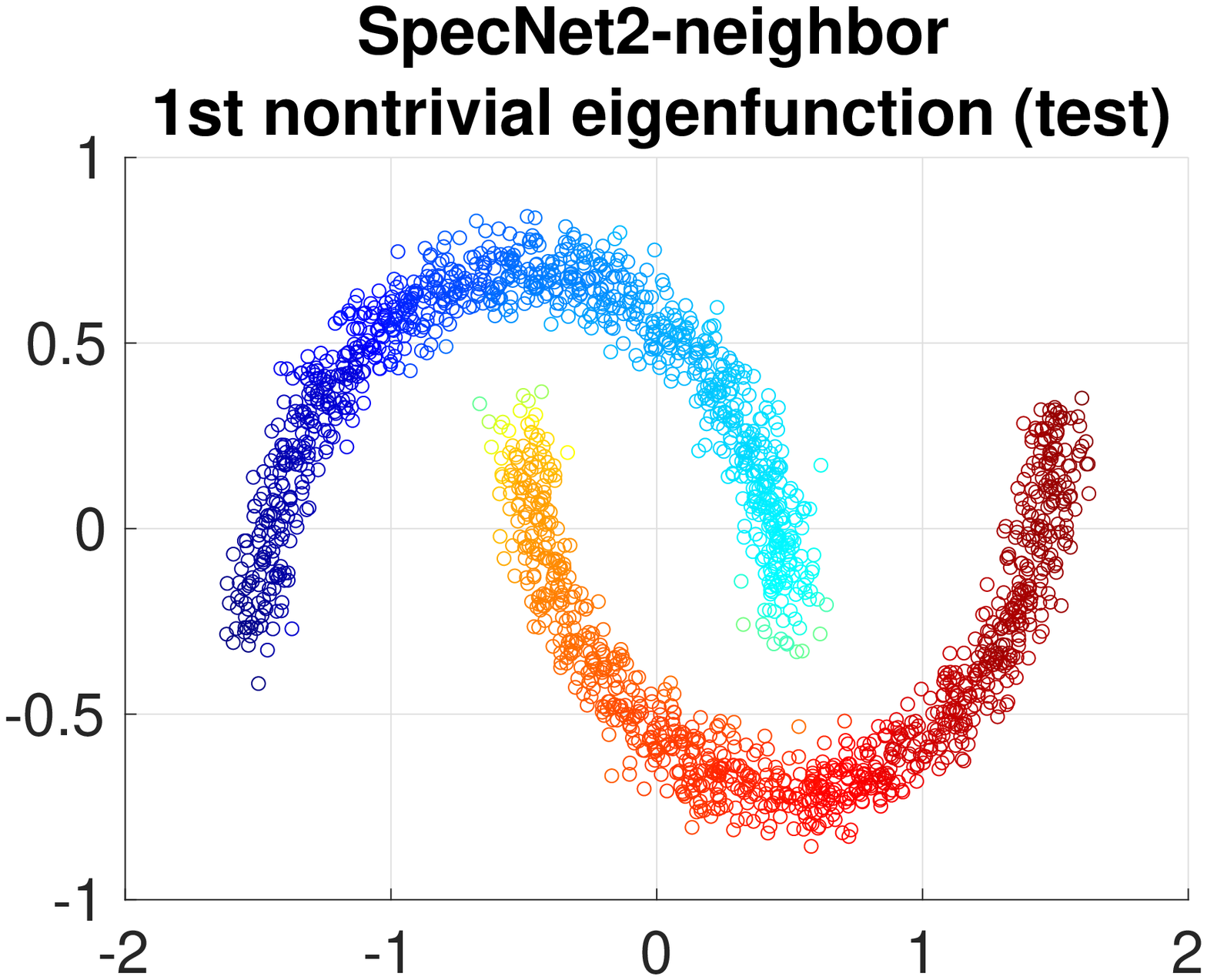}
    \caption{
    Two moons datset: embeddings by neighbor schemes using the first nontrivial eigenfunctions. The first row is the ground truth.}
    \label{fig:twomoon-embed}
\end{figure*}

\subsection{MNIST data}
\noindent\textbf{Data preprocessing:}
Our training set consists of 20000 sample images randomly selected from the MNIST training dataset and our testing set contains 10000 sample images from the MNIST testing dataset. Every sample in the training and testing set is vectorized as a vector in $\mathbb{R}^{784}$.

\noindent\textbf{Network training:}
We use a fully-connected feedforward neural network with two 256-unit hidden layers:

SpecNet1: $784 \xrightarrow[]{\text{fc}} 256 - \text{ReLU} \xrightarrow[]{\text{fc}} 256 -  \text{ReLU} \xrightarrow[]{\text{linear}} 7 \xrightarrow[]{\text{orthogonal}} 7$;

SpecNet2: $784 \xrightarrow[]{\text{fc}} 256 - \text{ReLU} \xrightarrow[]{\text{fc}} 256 -  \text{ReLU} \xrightarrow[]{\text{linear}} 6$.

We want to embed the training set using first six nontrivial eigenvectors of $D^{-1}W$, so the output dimension for SpecNet1 is 7 and that for SpecNet2 is 6, and we use the Adam as the optimizer with learning rate $10^{-4}$ for both SpecNet1 and SpecNet2.

We also show another example in Figure \ref{fig:MNIST-loss2} where we construct the adjacency matrix $A$ of an kNN graph on the training set by setting $A_{i,j} = 1$ if the $j$-th training sample is within 
$k$ nearest neighbors of the $i$-th training sample and $A_{i,j} = 0$ otherwise, and we use $k = 10$. As a result, the average numbers of neighbors of a batch of
size 2, 4 and 8 are about 28, 56 and 112 respectively. We observe that SpecNet2-neighbor with batch size 2 has higher variance, but it can still achieve better performance compared to SpecNet1-local with batch size 28 in the long run.

\begin{figure*}[htbp]
    \centering
    \includegraphics[width=0.245\textwidth]{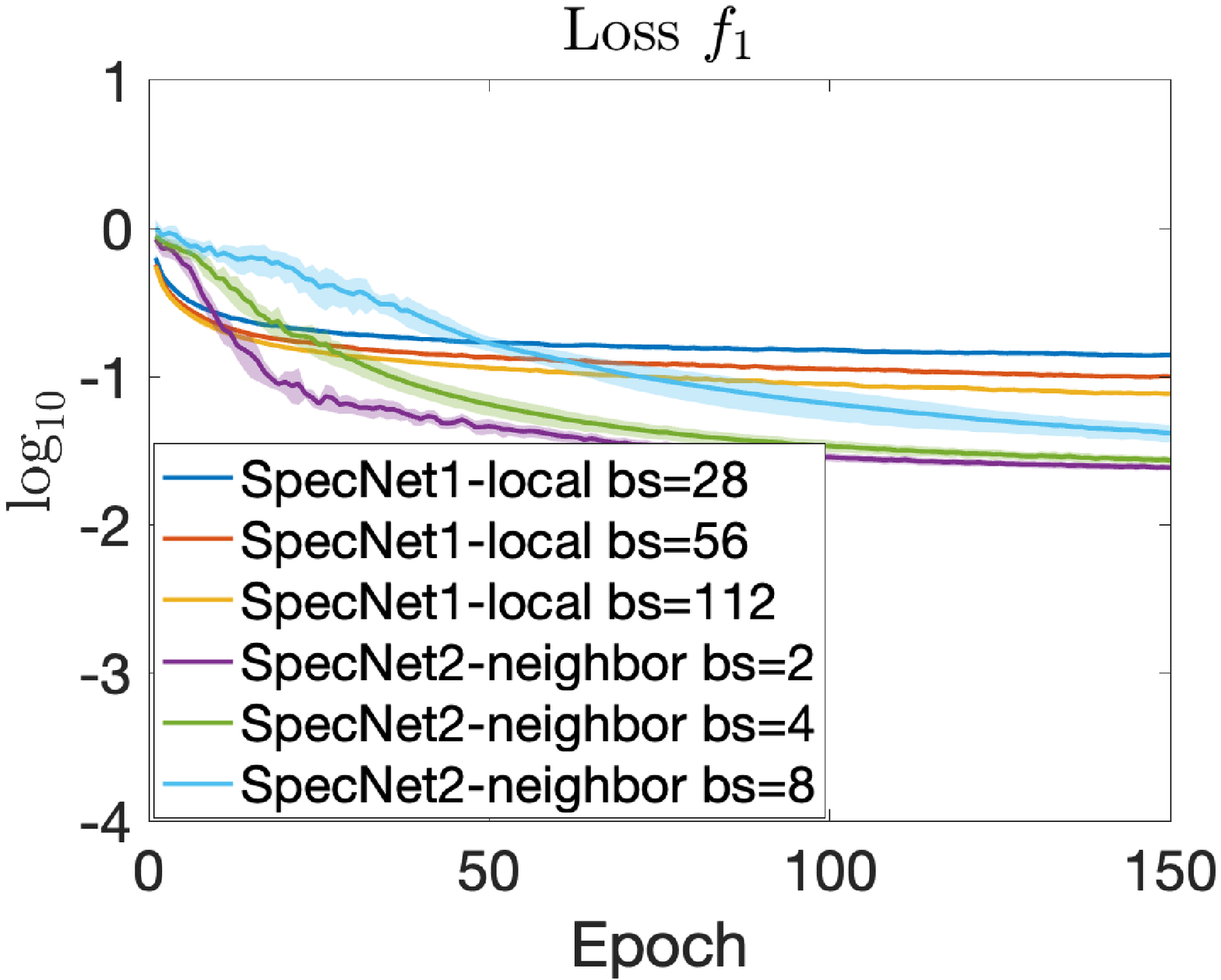}
    \includegraphics[width=0.245\textwidth]{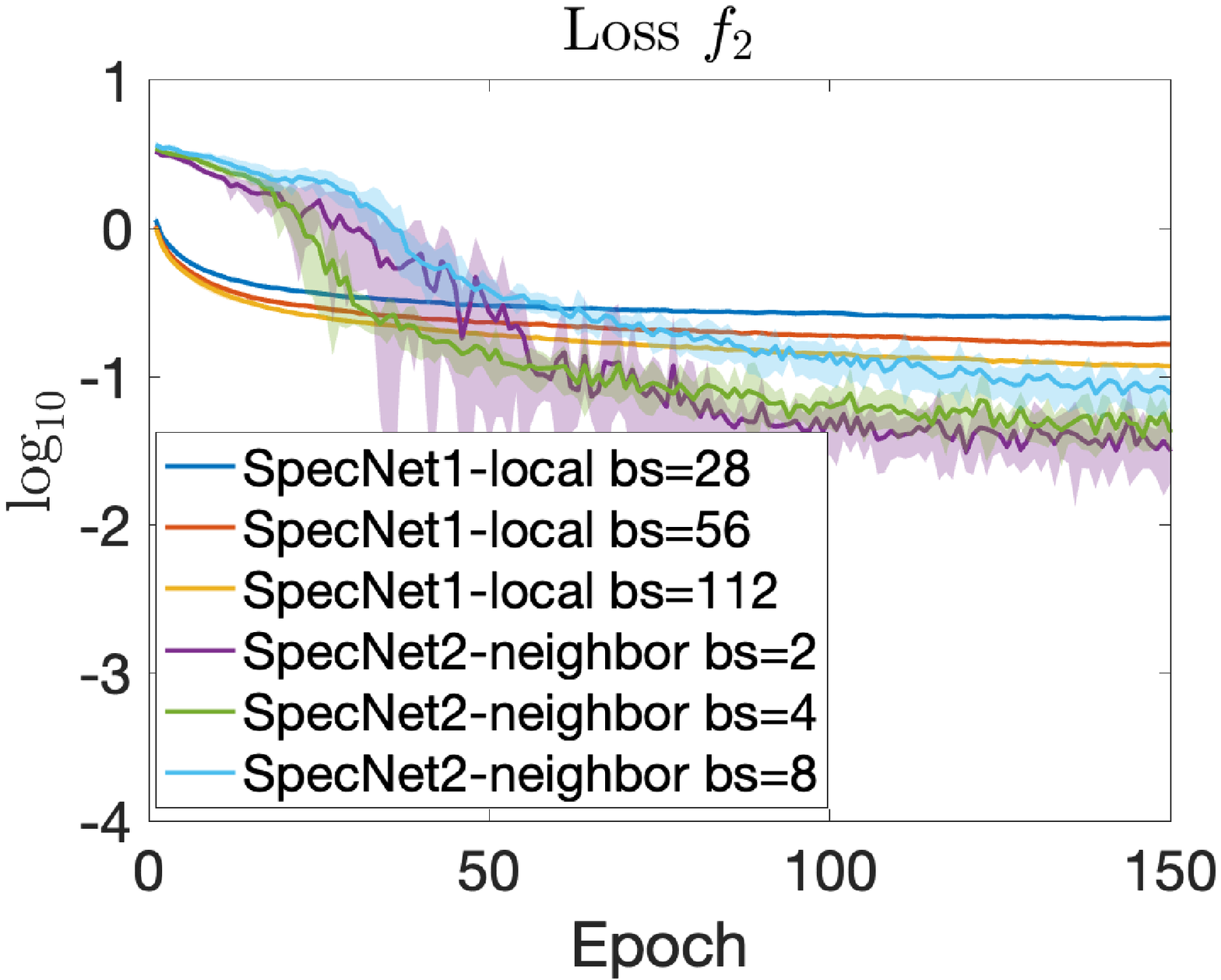}
    \includegraphics[width=0.245\textwidth]{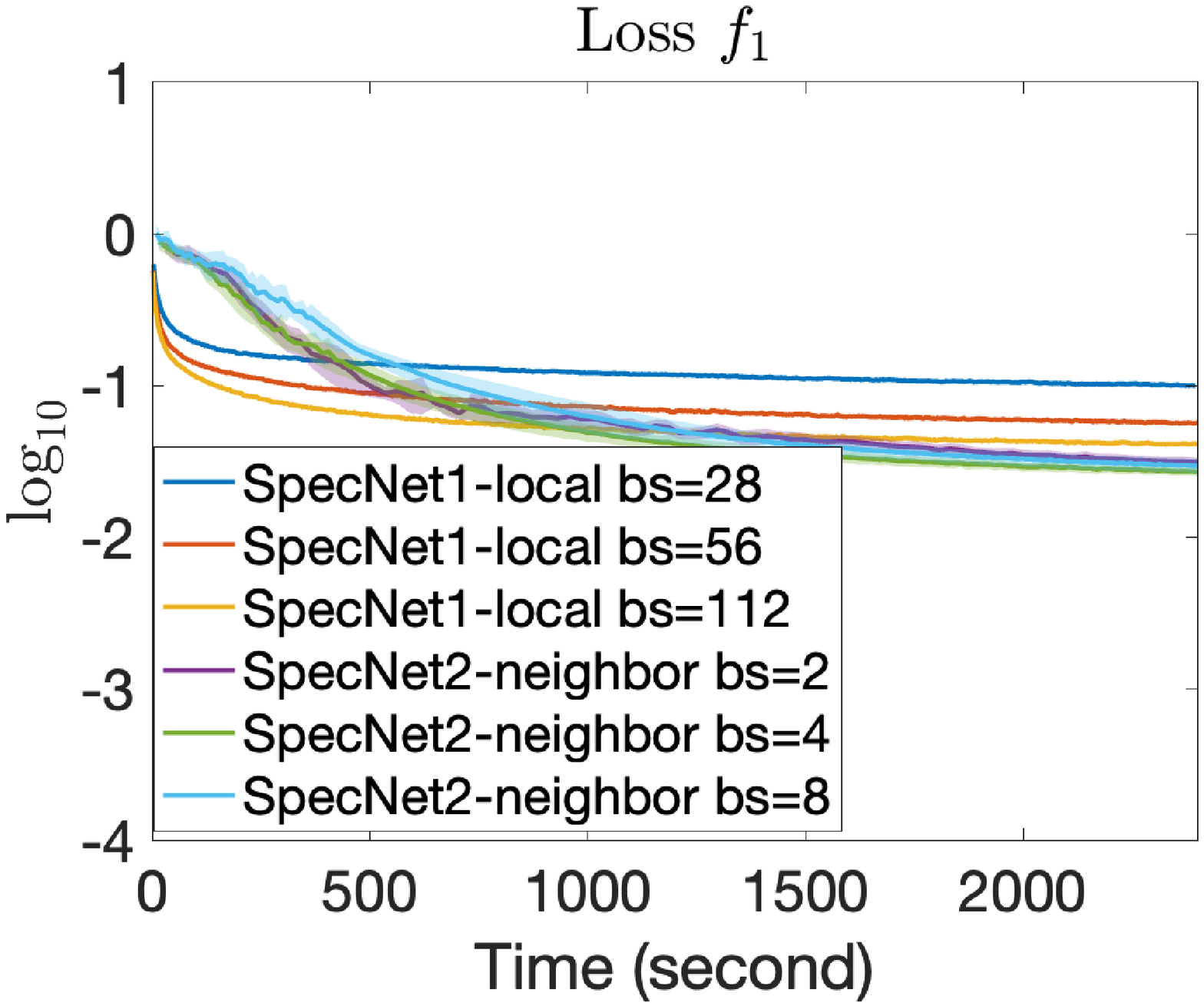}
    \includegraphics[width=0.245\textwidth]{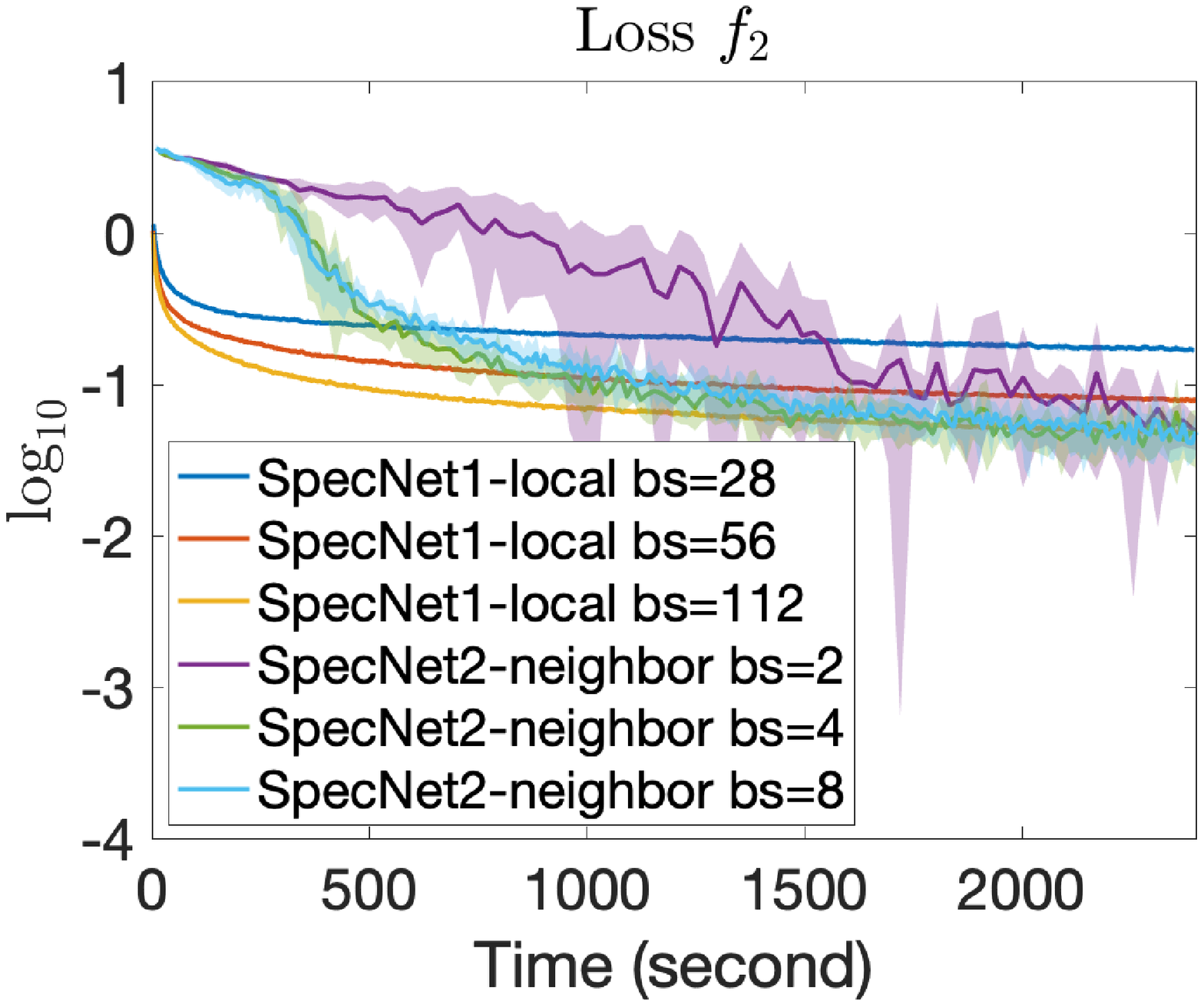}
    \caption{
    MNIST datset: plot of two different losses $\log_{10}(f_1(Y) - f_1^\star)$ and $\log_{10}(f_2(Y) - f_2^\star)$ over epochs (left two subfigures) and over time (right two subfigures), with $f_1$ and $f_2$ defined in \eqref{eq:specnet1} and \eqref{eq:uopt2}. Networks are trained on 20000 MNIST images on a 2021 14-inch Macbook Pro with an 8-core CPU. 
    }
    \label{fig:MNIST-loss2}
\end{figure*}

\section{Scaling of the losses and gradients}
\label{appendix:scaling}

We consider the generalized eigenvalue problem $WU = DU\Lambda$, where $W,
D \in \mathbb{R}^{n \times n}$, $X \in \mathbb{R}^{n \times K}$, $W_{i,j}
= k_{\sigma}(x_i, x_j) :=  e^{\frac{-\|x_i - x_j\|^2}{2\sigma^2}}$, and
$\sigma$ is fixed; $D$ is a diagonal matrix such that $D_{i,i} =
\sum_{j=1}^n W_{i,j}$.

This generalized eigenvalue problem can be viewed as the discretization of
the following continuous eigenvalue problem:
\begin{equation}
    \int k_{\sigma}(x,y) \psi_k(y) p(y)\diff{y} =
    \lambda_k u_{\sigma}(x)\psi_k(x),
\end{equation}
where $p(x)$ is the density function, and $u_{\sigma}(x) := \int
k_{\sigma}(x,y)p(y)\diff{y}\approx m_0 p(x)\sigma^d +
O_{d,p,k_\sigma}(\sigma^{d+2})$, where $m_0$ depends on the dimension $d$. And
$\psi_i$ satisfies the normalization condition:
\begin{equation}\label{normal}
    \int \psi_i(x)\psi_j(x)u_{\sigma}(x)p(x)\diff{x} = \delta_{ij}.
\end{equation}
Note that $\lim_{n\to\infty} \frac{D_{i,i}}{n} = u_{\sigma}(x_i)$
by the law of large numbers.

Consider the loss function
\begin{equation}
    g(Y) = \trace{-2Y^{\top}AY + Y^{\top}BYY^{\top}BY},
\end{equation}
where $Y \approx [y_1(x),\dots,y_k(x)] = [\sqrt{\lambda_1}\psi_1(x),
\dots, \sqrt{\lambda_k}\psi_k(x)]$ and entries of $Y$ is $O(1)$. We need
to know a proper scaling of $A$ and $B$ in terms of $W$ and $D$ so that
$g$ is $O(1)$ and does not scale with $n$. Recall that $\trace{Y^\top W
Y}$ and $\trace{Y^{\top}DYY^{\top}DY}$ are discretization of integrals,
\begin{align*}
    (Y^{\top}WY)_{k,k} 
    &= \sum_{i=1}^n\sum_{j=1}^n y_{k}(x_i)y_k(x_j)k_{\sigma}(x_i,x_j)\\
    & \approx n^2 \int \int k_{\sigma}(x,y) y_k(x) y_k(y) p(x) p(y)
    \diff{x}\diff{y},\\
    (Y^{\top}DY)_{k,l} 
    & = \sum_{i=1}^n y_k(x_i)y_l(x_i)D_{i,i}\\
    & \approx \sum_{i=1}^n y_k(x_i)y_l(x_i) n u_{\sigma}(x_i)\\
    & \approx n^2\int y_k(x)y_l(x)u_{\sigma}(x)p(x)\diff{x}.
\end{align*}
Hence, the proper scaling for $A$ and $B$ are $A = \frac{W}{n^2}$ and $B =
\frac{D}{n^2}$ respectively.

Now let us look at the functional (variational) derivative of $g$ with
respect to $y_s$. We first split $f$ into $g_1(Y) = \trace{-2Y^{\top}AY}$
and $g_2(Y) = \trace{Y^{\top}BYY^{\top}BY}$.
\begin{equation*}
    g_1(Y) \approx -2\sum_k \int\int
    k_{\sigma}(x,y)y_k(x)y_{k}(y)p(x)p(y)\diff{x}\diff{y}.
\end{equation*}
Replacing $y_s(x)$ by $y_s(x)+\epsilon \eta(x)$, and taking derivative with
repsect to $\epsilon$ at $0$, we have
\begin{equation*}
    \frac{\diff{g_1}}{\diff{\epsilon}}\Big|_{\epsilon=0} = 
    -4\int\int k_{\sigma}(x,y)y_s(x)\eta (y)p(x)p(y)\diff{x}\diff{y},
\end{equation*}
and the variational derivative of $g_1$ with respect to $y_s(x)$ is
\begin{align*}
    \frac{\partial g_1}{\partial y_s}(y) 
    = -4\int k_{\sigma}(x,y)y_s(x)p(x)\diff{x}
    \approx \frac{1}{n}(-4 \sum_{i=1}^n k_{\sigma}(x_i,y)y_s(x_i)).
\end{align*}
Therefore, the $O(1)$ scaling of the gradient of $g_1$ is $\nabla_Y g_1 =
-4\frac{W}{n}Y$.

On the other hand, a similar procedure can be applied to analyze the
scaling of the gradient of $g_2$. Recall,
\begin{align*}
    g_2(Y) =\sum_{k}(\sum_{i=1}^n (Y^\top BY)_{k,i}(Y^\top BY)_{i,k})
    \approx \sum_{k}\sum_{i}(\int y_k(x)y_i (x)u_{\sigma}(x)p(x)\diff{x})^2.
\end{align*}
Replacing $y_s(x)$ by $y_s(x)+\epsilon \eta(x)$, taking derivative with
respect to $\epsilon$ at $0$, and using the orthogonality condition
\eqref{normal}, we have,
\begin{equation*}
    \frac{\diff{g_2}}{\diff{\epsilon}}\Big|_{\epsilon=0}
    = 4[(\int y_s^2(x)u_{\sigma}(x)p(x)\diff{x})(\int
    y_s(x)\eta(x)u_{\sigma}(x)p(x)\diff{x})],
\end{equation*}
and 
\begin{align*}
    \frac{\partial g_2}{\partial y_s}(x_j) 
    & = 4 (\int y_s^2(y)u_{\sigma}(y)p(y)\diff{y})y_s(x_j)u_{\sigma}(x_j)\\
    & \approx 4\frac{1}{n}(\sum_{i=1}^n y_s^2(x_i) \frac{D_{i,i}}{n})
    y_s(x_j) \frac{D_{j,j}}{n}
    = \frac{4}{n^3}(\sum_{i=1}^n y_s^2(x_i)D_{i,i}) y_s(x_j)D_{j,j}.
\end{align*}
Hence the $O(1)$ scaling of the gradient of $g_2$ is $\nabla_Y g_2 =
\frac{4}{n^3}DYY^\top DY$.

\end{document}